\documentclass[accepted]{template}
\pdfoutput=1
\newif\iflong
\longfalse
\newif\ifcomment
\commentfalse
\usepackage{natbib}

\usepackage[utf8]{inputenc} %
\usepackage[T1]{fontenc}    %
\usepackage{url}            %
\usepackage{booktabs}       %
\usepackage{amsfonts}       %
\usepackage{nicefrac}       %
\usepackage{microtype}      %
\usepackage{xcolor}         %
\usepackage{graphicx}
\usepackage{enumitem}
\usepackage{multirow}

\usepackage{amsmath}
\usepackage{amssymb}
\usepackage{mathtools}
\usepackage{amsthm}
\usepackage{mycommands}
\usepackage[mathscr]{euscript}
\usepackage{thm-restate}
\usepackage{algorithm}
\usepackage{algorithmic}
\usepackage{wrapfig}
\usepackage{caption}
\usepackage{subcaption}
\usepackage{hyperref}       %
\usepackage[capitalize,noabbrev]{cleveref}
\theoremstyle{plain}
\newtheorem{theorem}{Theorem}[section]

\theoremstyle{definition}

\theoremstyle{remark}

\usepackage{pascal-def}

\usepackage{physics}

\newcommand{\ouralgo}{{\rm CoP2L}}

\title{Sample Compression for Self-Certified Continual Learning}

\author[1,2]{Jacob Comeau}
\author[1,2]{Mathieu Bazinet}

\author[1,2]{Pascal Germain}
\author[1,2,3]{Cem Subakan}
\affil[1]{%
Département d'informatique et génie logiciel\\
Universit\'e Laval\\
Québec, Qc, Canada
}
\affil[2]{%
    Mila - Quebec Artificial Intelligence Institute \\ Montreal, Qc, Canada
}
\affil[3]{%
Computer Science and Software Engineering Department \\ Concordia University \\ Montreal, Qc, Canada
}

\begin{document}
\maketitle

\begin{abstract}
Continual learning algorithms aim to learn from a sequence of tasks. In order to avoid catastrophic forgetting, most existing approaches rely on heuristics and do not provide computable learning guarantees. In this paper, we introduce Continual Pick-to-Learn (CoP2L), a method grounded in sample compression theory that retains representative samples for each task in a principled and efficient way. 
This allows us to derive non-vacuous, numerically computable upper bounds on the generalization loss of the learned predictors after each task. 
We evaluate CoP2L on standard continual learning benchmarks under Class-Incremental and  Task-Incremental settings, showing that it effectively mitigates catastrophic forgetting. It turns out that CoP2L is empirically competitive with baseline methods while certifying predictor reliability in continual learning with a non-vacuous bound.
\end{abstract}

\section{INTRODUCTION}

A common assumption in traditional machine learning is that the underlying data distribution does not evolve with time. In Continual Learning \citep{de2021continual, wang2023comprehensive}, the goal is to develop machine learning algorithms that are able to learn under settings where this assumption is replaced by a setup where the model is trained on an evolving training data distribution, in such a way that samples are revealed one task at a time. 
However, neural networks trained on evolving data distributions tend to forget earlier tasks, a phenomenon known as catastrophic forgetting \citep{McCloskey1989CatastrophicII, forgetting_french, goodfellow2013empirical}. In order to cope with forgetting, various methodologies have been developed, such as regularization-based approaches \citep{kirkpatrick2017overcoming,zenke2017continual,li2017learning,mas}, architectural-based approaches \citep{pnn_rusu,mallya2018piggyback,expert_aljundi,packnet_mallya,l2p_wang, he2025cl,gao2023unified,zhou2025revisiting, zhou2024expandable} or rehearsal-based approaches \citep{rolnick2019experience,chaudhry2019efficient,replay_shin,icarl_rebuffi,der}.

\begin{figure}[!t]
\centering
    \begin{subfigure}[b]{\columnwidth}
    \centering
            \includegraphics[width=\textwidth]{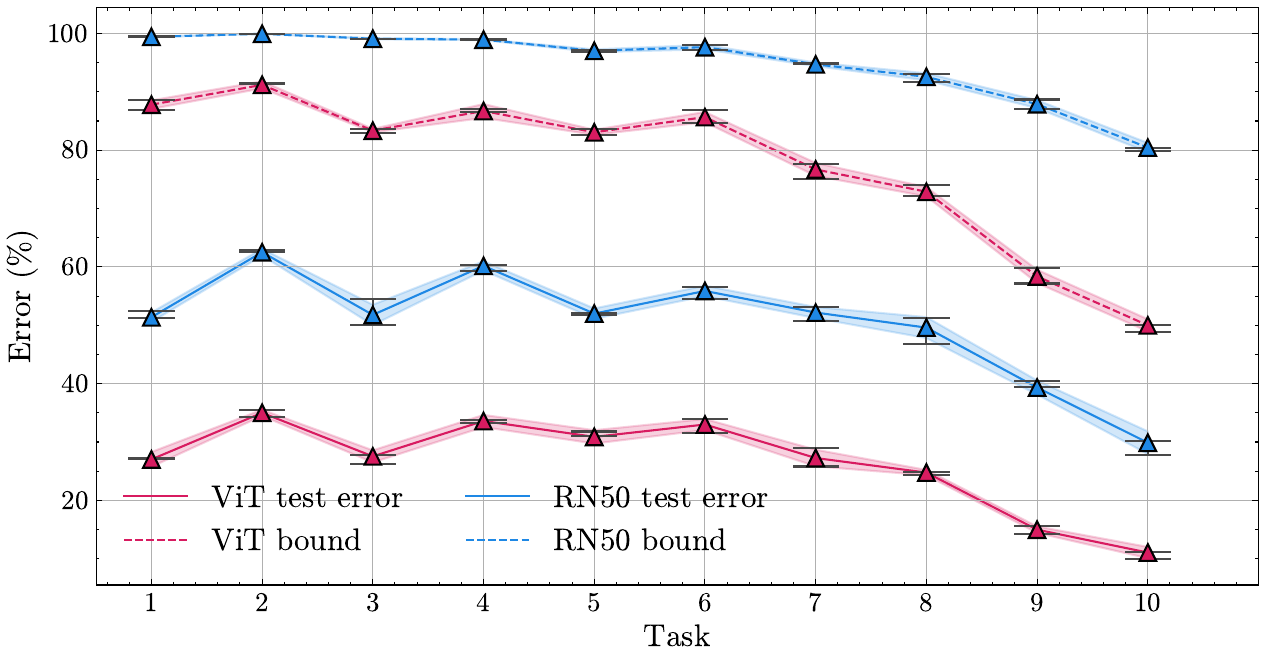}
            \vspace{-5mm}
            \caption{Class-incremental setting}
    \end{subfigure}
    \smallskip
    
\begin{subfigure}[b]{\columnwidth}
\centering
        \includegraphics[width=\textwidth]{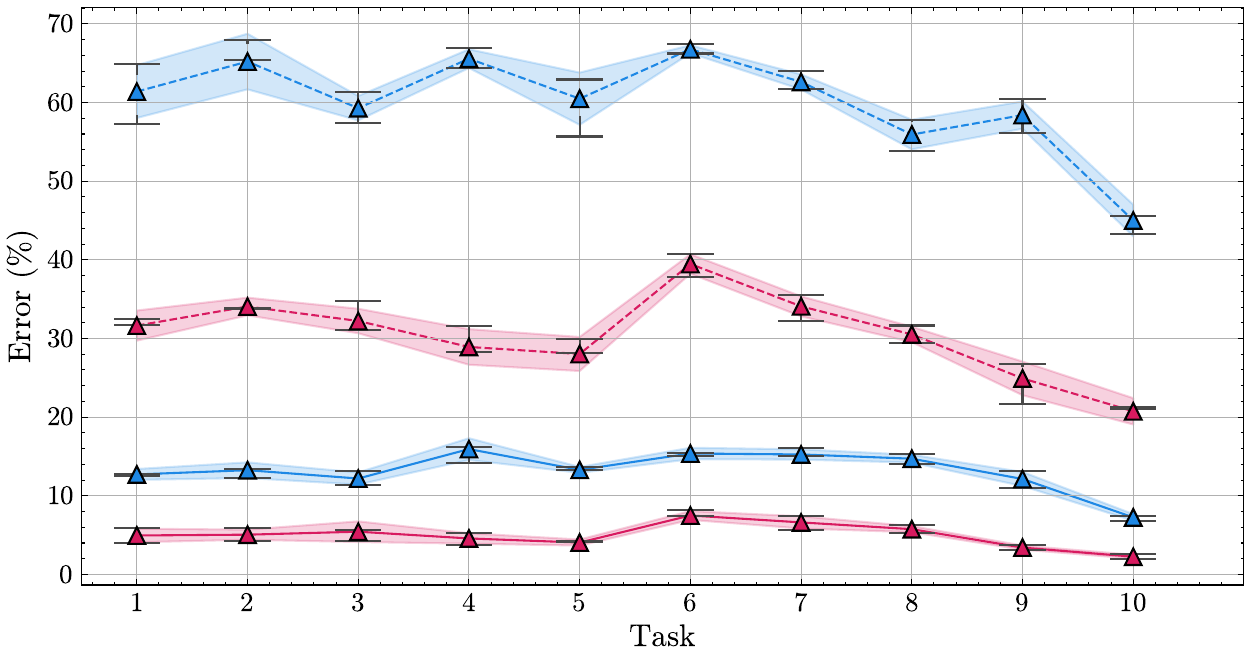}
        \vspace{-5mm}
        \caption{Task-incremental setting}
\end{subfigure}
    \caption{Numerical values of the proposed generalization bounds for continual learning over 10 tasks on CIFAR100, using ViT and ResNet50 backbones. The bounds hold simultaneously for all tasks. }
\end{figure}

We develop a novel sample compression theory for continual learning. Being agnostic to the architecture, our theoretical results are well-suited for rehearsal-based approaches, which are commonly used in the continual learning literature for practical applications \citep{DellaLibera2023CLMASRAC, Mai2020BatchlevelER, chaudhry2019efficient, aljundi2019onlinecontinuallearningmaximally}. 
We then propose a theoretically motivated algorithm that uses the replay buffer only when necessary to mitigate the forgetting. These new results pave the way for continual learning approaches based on theoretical guarantees.

More precisely, our proposed continual learning scheme builds upon the Pick-to-Learn meta-algorithm (P2L) of \citet{paccagnan_pick_learn_2023}.
This meta-algorithm was developed specifically with sample compression theory in mind \citep{floyd1995sample,campi2023compression}. By finding a small subset of the data such that a predictor learned on this subset achieves low error on the whole training set, P2L enables computing tight upper bounds for the generalization error of the learned predictor. 
This ability of a learning algorithm to simultaneously output a predictor and its risk certificate has been coined as \emph{self-certified learning}  \citep{freund98selfbounding,perezortiz2021tighter,marks2025pick2learn-gp}. This certification is arguably an asset to increase the \emph{trustworthiness} of machine learning systems.

The \emph{Continual Pick-to-Learn} (\ouralgo{})\footnote{Our code is available at \url{https://anonymous.4open.science/r/CoP2L_paper_code-0058/}} algorithm that we propose leverages sample compression theory to intelligently select the training data from earlier tasks to mitigate forgetting. We derive high-confidence upper bounds on the generalization error for each task simultaneously, estimated directly from the training set. We also empirically show that \ouralgo{} substantially mitigates forgetting and is comparable with several standard baselines. Our contributions can be summarized as follows:  

    (1) We propose the algorithm \emph{Continual Pick-to-Learn} (\ouralgo{}), that integrates the sample compression theory within the continual learning setup. To the best of our knowledge, we are the first to integrate the theoretical results from sample compression to rehearsal-based continual learning. 
    
    (2) \ouralgo{} is able to provide high confidence non-trivial upper bounds for the generalization error by means of the sample compression theory. We experimentally showcase that the bounds are numerically computable, have non-trivial values and follow the general error trends observed on the test set. Therefore, they can actually be used as risk certificates for the model's behaviour on the learned tasks, improving the trustworthiness of the continually learnt model. 
    
    (3) We experimentally show with a wide range of experiments that \ouralgo{} significantly mitigates forgetting, while also being able to obtain comparable performance compared with several strong continual learning baselines.

\subsection{Related Work}
\label{sec:relatedwork}

\paragraph{Continual learning approaches.} %
In the continual learning literature, various practical approaches have been developed to mitigate forgetting. Broadly, the approaches could be divided into three categories: 1) regularization-based approaches \citep{kirkpatrick2017overcoming,zenke2017continual,li2017learning,mas}, 2) architecture-based approaches \citep{pnn_rusu,mallya2018piggyback,expert_aljundi,packnet_mallya,l2p_wang,he2025cl,gao2023unified,zhou2025revisiting, zhou2024expandable} and 3) rehearsal-based approaches  \citep{rolnick2019experience,chaudhry2019efficient,replay_shin,icarl_rebuffi,der}. 

Most rehearsal-based approaches use a \emph{replay buffer} that contains a small subset of the dataset over which the model continues to train, even when learning over new tasks. The buffer helps to mitigate the forgetting of a model on previous tasks. However, the effect of the buffer depends on the quality of the data chosen in the buffer. A naive approach is to sample randomly the datapoints \citep{Mai2020BatchlevelER, DellaLibera2023CLMASRAC}. Non-naive methods include methods such as iCarl \citep{icarl_rebuffi}, Dark Experience Replay \citep{der}, Gradient Episodic Memory \citep{lopez2017gradient}, Example Forgetting \citep{benkHo2024example}, Global Pseudo-Task simulation \citep{liu2022navigating} and Coreset-based methods \citep[e.g.,][]{hao2023bilevel, tong2025coreset, borsos2020coresets, yoon2021online}.

\paragraph{Continual learning generalization bounds.}
In recent years, 
generalization bounds were proposed to study the asymptotic behavior of various specific models, e.g., for linear models \citep{Evron2022HowCC,Goldfarb2022AnalysisOC,pmlr-v202-lin23f,banayeeanzade2025theoretical}; for Neural Tangent Kernel models trained with Orthogonal Gradient Descent \citep{AbbanaBennani2020GeneralisationGF, Doan2020ATA,farajtabar2019orthogonalgradientdescentcontinual}; for regularization-based methods \citep{yin2021optimizationgeneralizationregularizationbasedcontinual}; or for a PackNet-inspired \citep{packnet_mallya} model for continual representation learning \citep{Li2022ProvableAE}.
Learning algorithms with sample complexity guarantees were also proposed under data assumptions, such that all target tasks lie in a same space obtainable from a (linear or non-linear) combination of some input features \citep{cao2022lifelong}. Finally, \citet{friedman2024data} presented new PAC-Bayes bounds for the backward transfer loss of stochastic models trained without a replay buffer.

In contrast with the listed works above, in this paper, we simultaneously provide risk certificates on the true risk of each task, our bound is non-asymptotic and can be computed from the training set. Moreover, our bound is applicable to any neural network architecture, under minimal data assumptions and does not have limitations on the continual learning setup that is employed.

\paragraph{Sample compression.}
The sample compression theory has been shown to be effective for providing tight generalization bounds 
\citep{marchand2002set,marchand2005learning, shah_margin-sparsity_2005,campi2023compression}.
However, most sample compression approaches are limited to low-complexity models. A notable exception is the method Pick-to-Learn (P2L), presented by \citet{paccagnan_pick_learn_2023}, which successfully provides guarantees for deep neural networks \citep{bazinet2024} and Gaussian Processes \citep{marks2025pick2learn-gp}.  
In this paper, we adapt the key ideas introduced in the sample compression theory to the continual learning case, which enables us to provide an upper bound on the true risk from training samples. 

\section{BACKGROUND}

\paragraph{Continual learning.}
Let us consider a meta-dis\-tri\-bu\-tion~$\frakD$ and a sequence of \emph{i.i.d.{}}
(independent and identically distributed) task distributions $\Dcal_1, \Dcal_2, \ldots, \Dcal_T\sim \frakD$ on an input-output space $\Xcal\times\Ycal$. We are given a predictive \emph{model} $f_{\theta}:\Xcal\to\Ycal$ (\emph{e.g.}, a neural network architecture), with \emph{learnable} parameters $\theta\in\Theta$. From randomly initialized parameters~$\theta_0$, the aim of the continual learning process at step $t\in\{1,\ldots,T\}$ is to update the parameters $\theta_{t-1}$ into~$\theta_t$ by learning from a new task sample $S_t\sim\Dcal_t$. 
We denote {the training set at task $t$} as $S_t = \{(\xbf_{t,i}, y_{t,i})\}_{i=1}^{n_t}$, with $\xbf_{t,i}\in\Xcal$ and $y_{t,i} \in \Ycal_t \subseteq \Ycal$; that is, the same input space $\Xcal$ is shared among all tasks, and the output space $\Ycal$ may be the union of several task specific output spaces $\Ycal_t$.
Given a loss function $\ell:\Theta\times\Xcal\times\Ycal\to[0,1]$, we want the last updated predictor~$f_{\theta_T}$ to maintain a low generalization loss on all observed tasks $t\in\{1,\ldots,T\}$:
\begin{equation}
    \Lcal_{\calD_t}(\theta_T) \ \eqdef \Esp_{(\xbf, y)\sim\Dcal_{t}} \ell(\theta_T, \xbf, y)\,.
\end{equation}
The empirical loss counterpart on observed training samples is given by
\begin{equation}
    \widehat\Lcal_{S_t}(\theta_T)  \ \eqdef \ \frac{1}{n_t}\sum_{i=1}^{n_{t}} \ell(\theta_T, \xbf_{t,i}, y_{t,i})\,.
\end{equation} 
The challenge of continual learning lies in the fact that the learner observes the task datasets $S_1, S_2, \ldots, S_t, \ldots, S_T$ sequentially and that we assume that the system cannot keep all observed data in memory. Nevertheless, at task~$t$, we would like the system to perform well on all tasks, including the previous ones $1,\dots,t-1$. However, simply updating the parameters~$\theta_{t-1}$ learned from a previous tasks to optimize the loss $\widehat\Lcal_{S_t}(\theta_t)$ on a current task~$t$ would lead to \emph{catastrophic forgetting}~\citep{forgetting_french}. Therefore, as indicated in Section~\ref{sec:relatedwork}, numerous empirical methods have been developed to mitigate forgetting. A simple yet very effective strategy is to simply store a small subset of data from the earlier tasks, coined as the \emph{replay buffer}, and to include them in the objective function of the \emph{experience replay} strategy for learning task~$t$: 
\begin{equation}
    \Fcal_{t}^{\ \text{replay}}(\theta_t)  \ \eqdef \ \widehat\Lcal_{S_t}(\theta_t)  +\frac{1}{|\mathcal B |} \sum_{j\in \mathcal B} \ell(\theta_T, \xbf_{j}, y_{j})   \,,
\end{equation} 
where the second term is the loss on the replay buffer $\Bcal$. 

\paragraph{Sample compression theory.}
Let us consider a training dataset $S= \{(\bx_i, y_i)\}_{i=1}^n \in (\calX \times\calY)^n$ sampled \emph{i.i.d.}\ from an unknown distribution $\calD$, a family of learnable parameters~$\Theta$ and a learning algorithm $A$ such that $A(S) \in \Theta$. In this setting, sample compression theory provides generalization bounds for any predictor $f_\theta$, with $\theta = A(S)$, on the condition that~$f_\theta$ can be provably represented as a function of a small subset of the training dataset $S$, called the compression set, and an additional source of information, called the message. If this is the case, we call $f_\theta$ a sample-compressed predictor.

The compression set is denoted $S^{\bfi}$ and is parameterized by a strictly increasing sequence of indices $\bfi {\,\in\,} \scriptP(n)$, with $\scriptP(n)$ the powerset of $\{1,\ldots, n\}$.
The compression set is defined such that 
$S^{\bfi} = S^{(i_1, \ldots, i_{\m})} \!= \{(\bx_{i_1}, y_{i_1}), \ldots, (\bx_{i_{\m}}, y_{i_{\m}})\} \!\subseteq S.$ We define the complement sequence $\notbfi$ such that $\bfi \cap \notbfi = \emptyset$ and $\bfi \cup \notbfi = \{1,\ldots, n\}$. Thus, we have the complement set $S^{\notbfi} = S \setminus S^{\bfi}$ and $\notm = n-\m$.

In addition to the compression set, a message $\mu$ could be needed to compress the predictor. Although generally defined as a binary sequence, the message set can be any set of countable sequences of symbols. Let $\Sigma=\{\sigma_1, \sigma_2, \ldots, \sigma_N\}$ be the alphabet used to construct the messages and $\Sigma^*$ be the set of all possible sequences, of length $0$ to $\infty$, constructed using the alphabet~$\Sigma$. For all $\bfi\,{\in}\, \scriptP(n)$, we choose $\mathscr{M}(\bfi)$ a countable subset of $\Sigma^*$, which represents all the possible messages that can be chosen for~$\bfi$.

Given learned parameters $\theta=A(S)$, to prove that the predictor $f_\theta$ is a sample-compressed one, we need to define two functions: a compression function and a reconstruction function. The compression function $\mathscr{C} : \pmb\cup_{1\leq n \leq \infty} (\calX \times \calY)^n \to \pmb\cup_{m \leq n} (\calX \times \calY)^m \times \pmb\cup_{\bfi \in \scriptP(n)}\mathscr{M}(\bfi)$ outputs a compression set~$S^{\bfi}$ and a message~$\mu$. 
Then, the reconstruction function $\scriptR : \pmb\cup_{m \leq n} (\calX \times \calY)^m \times \pmb\cup_{\bfi \in \scriptP(n)}\mathscr{M}(\bfi) \to \Theta$ is defined such that $A(S) = \scriptR\qty(S^{\bfi}, \mu)$. 
The forthcoming results rely on a probability distribution over the set of sample-compressed predictors $\overline{\Theta} \subseteq \Theta$. For any sample-compressed predictor $\scriptR(S^{\bfi},\mu)$, this data-independent distribution is expressed as $P_{\overline{\Theta}}\qty(\scriptR(S^{\bfi}, \mu)) {=} P_{\scriptP(n)}(\bfi) P_{\mathscr{M}(\bfi)}(\mu)$, with $P_{\scriptP(n)}$ a probability distribution over $\scriptP(n)$ and $P_{\mathscr{M}(\bfi)}$ a probability distribution over $\mathscr{M}(\bfi)$. Following common practices~\citep{marchand2005learning,bazinet2024}, we consider hereunder
$P_{\scriptP(n)}(\bfi) {=} {n \choose \m}^{_{-1}}\zeta(\m)$, 
with $\zeta(k) {=} \tfrac{6}{\pi^{_2}}(k{+}1)^{-2}$. Also, given a sequence of indices~$\bfi$, we set $P_{\mathscr{M}(\bfi)}$ to a uniform distribution over the messages.

We now present a recent generalization guarantee for sample compression \citep{bazinet2024}, on which we build our theoretical framework. We use the shorthand notation $\halLSbfic(\cdot)$ to denote the empirical loss on the training sample that do not belong to the compression set: $S^{\notbfi} = S\setminus S^\ibf$; it is mandatory to use $\halLSbfic(\cdot)$ to obtain an unbiased estimate of the loss, as the predictor $\scriptR(S^{\bfi},\mu)$ relies on $S^\ibf$ and the result wouldn't hold otherwise.

\begin{theorem}[\citealp{bazinet2024}]\label{thm:sample_compress}
For any distribution $\calD$ over $\calX \times \calY$, for any family of set of messages $\{\mathscr{M}(\bfi) | \bfi {\in} \scriptP(n)\}$, for any deterministic reconstruction function $\scriptR$ and for any $\delta \in (0,1]$, with probability at least $1-\delta$ over the draw of $S \sim \calD^n$, we have 
\begin{align*}
&\forall \mathbf{i} \in \scriptP(n), \; \mu \in \mathscr{M}(\bfi): \\  
& \mathcal{L}_{\calD}\qty(
\theta_{\bfi, \mu}
) %
\leq \kl^{-1}\left(\halLSbfic\qty(
\theta_{\bfi, \mu}
),
\frac{1}{\notm}\ln\frac{2\smqty(n \\ \m)\sqrt{\notm}}{\zeta(\m)P_{\mathscr{M}(\bfi)}(\mu)\delta}\right),
\end{align*}
with 
$\theta_{\bfi, \mu} \,{=}\, \scriptR(S^{\bfi},\mu)$ the reconstructed parameters, 
$\kl(q,p) = q\ln\frac{q}{p} + (1-q)\ln\frac{1-q}{1-p}$ the binary Kullback-Leibler divergence and its inverse 
\begin{equation}\label{eq:klinv}
    \kl^{-1}(q,\varepsilon) = \argsup_{p\in[0,1]} \left\{\kl(q,p) \leq \varepsilon\right\}.
\end{equation}
\end{theorem}
On the one hand, given a fixed compression set size~$\m$, the bound decreases when the training set size $n$ increases. On the other hand, given a fixed dataset size~$n$, the bound increases with $\m$. 

\paragraph{Pick-to-Learn.}
To obtain sample compression guarantees for a class of predictors, it is necessary to prove that the learned predictor only depends on a small subset of the data and (optionally) a message. Some predictors, such as the SVM \citep{boser_svm} and the perceptron \citep{rosenblatt1958perceptron, moran2020perceptron}, have a straightforward compression scheme, as they only depend on a subset of the data after training. Some learning algorithms, such as the set covering machine (SCM) \citep{marchand2002set, marchand2003set, marchand2005learning} and decision trees \citep{shah_sample_2007}, necessitate hand-crafted compression schemes involving a message. Up until recently, there was no sample compression scheme for deep neural networks.

\begin{algorithm}[t] \small
\caption{\small Pick-To-Learn (P2L) \quad \citep{paccagnan_pick_learn_2023}
}\label{alg:p2l}
\begin{algorithmic}[1]
\INPUT{$\theta_{\rm init}$}
\hfill\COMMENT{Initial parameters of the model $f_\theta$}
\INPUT{$S = \{(\xbf_i, y_i)\}_{i=1}^n$}
\hfill\COMMENT{Training set}
\INPUT{$\gamma$} \hfill\COMMENT{Stopping criteria}
\STATE{$k \gets 0$ ; $C_0 \gets\emptyset$ ; $\theta_0 \gets \theta_{\rm init}$}
\STATE{$(\overline{\xbf}, \overline{y}) \gets \argmax_{(\xbf,y) \in S} \ell(\theta_0, \xbf, y)$}
\WHILE{$\ell(\theta_k, \overline{\xbf}, \overline{y}) \geq \gamma $}
\STATE{$k \gets k + 1$}
\STATE
$C_k \gets C_{k-1} \cup \{(\overline{\xbf}, \overline{y})\}$\;
\STATE $\theta_k \gets \mbox{\bfseries update}(\theta_{k-1} , C_k)$\;
\STATE $(\overline{\xbf}, \overline{y}) \leftarrow \argmax_{(\xbf,y) \in S \setminus C_k} \ell(\theta_k, \xbf, y)$
\ENDWHILE
\STATE {\bfseries return } $\theta_{k}$ \hfill
\COMMENT{Final parameters}
\end{algorithmic}
\vspace{-.1cm}
\end{algorithm}

The meta-algorithm Pick-To-Learn (P2L) was proposed by \citet{paccagnan_pick_learn_2023} as a compression scheme for any class of predictors, with a specific focus on deep neural networks. The meta-algorithm acts as an additional training loop, by iteratively selecting datapoints over which the model is updated. Starting with initial parameters $\theta_{\rm init}$, P2L evaluates the predictor on the whole training dataset, adds the datapoints with the highest losses to the compression set and then updates the predictor using the compression set. The meta-algorithm stops once the losses of training points that do not belong to the compression set are lower than a given stopping criterion~$\gamma$. We provide the pseudo-code of P2L in Algorithm~\ref{alg:p2l}. Recent works introduce improvements to the original scheme, notably early stopping strategies \citep{bazinet2024, paccagnan2025pick} and extension to the Bayesian setting \citep{marks2025pick2learn-gp}.

\section{SAMPLE COMPRESSION FOR CONTINUAL LEARNING}
A striking realization coming from the P2L algorithm is that only a small fraction of the training set -- the compression set -- needs to be provided to the learner in order to achieve good generalization. The winning strategy is to select this compression set to ensure a low risk on the training samples not being retained in the compression set (i.e., the complement of the compression set). This observation motivates the strategy employed in our new algorithm Continual Pick-To-Learn (\ouralgo{}), which manages the \emph{replay buffer} by subsampling only from the complement set instead of the whole training set.
Then, while learning subsequent tasks, \ouralgo{} maintains a low risk on previous tasks by adding well-chosen datapoints from the buffer to its compression set. 
This strategy inherently both mitigates the forgetting and enables the computation of sample compression bounds.

\begin{algorithm}[t] \small
\caption{\small Modified Pick-To-Learn (mP2L)}\label{alg:modified-p2l}
\begin{algorithmic}[1]
\INPUT{$\theta_{\rm init}$}
\hfill\COMMENT{Initial parameters of the model $f_\theta$}
\INPUT{$S = \{(\xbf_i, y_i, w_i)\}_{i=1}^n$}
\hfill\COMMENT{Training set (with weights)}
\INPUT{$B^{\star} = \{(\xbf_i, y_i, w_i)\}_{i=1}^m$}
\hfill\COMMENT{Buffer set (with weights)}
\INPUT{$\gamma$} \hfill\COMMENT{Stopping criteria}
\INPUT{$K^{\star}$} \hfill\COMMENT{Maximum number of iterations}
\STATE{$k \gets 0$ ; $C_0 \gets\emptyset$ ; $\theta_0 \gets \theta_{\rm init}$}
\STATE{$S^{\star} \gets S \cup B^{\star}$}
\STATE{$(\overline{\xbf}, \overline{y}, \overline{w}) \gets \argmax_{(\xbf,y, w) \in S^{\star}} \ell(\theta_0, \xbf, y)\cdot w$}
\WHILE{$ \ell(\theta_k, \overline{\xbf}, \overline{y})\cdot\overline{w} \geq \gamma $ and $k \leq K^{\star}$}
\STATE{$k \gets k + 1$}
\STATE
$C_k \gets C_{k-1} \cup \{(\overline{\xbf}, \overline{y}, \overline{w})\}$\;
\STATE $\theta_k \gets \mbox{\bfseries update}(\theta_{k-1} , C_k)$\;
\STATE $(\overline{\xbf}, \overline{y}, \overline{w}) \leftarrow \argmax_{(\xbf,y, w) \in S^{\star} \setminus C_k} \ell(\theta_k, \xbf, y)\cdot w$
\ENDWHILE
\IF{$k< K^{\star}$}
\STATE $k \gets \argmin_{0\leq k' \leq k} \boundfct(S, \theta_{k'}, C_{k'} \cap S)$ 
\hfill\COMMENT{Eq.~\eqref{eq:boundfct}}
\ENDIF
\STATE {\bfseries return } $\theta_{k}, C_{k}$ \COMMENT{Final parameters and compression~set}
\end{algorithmic}
\end{algorithm}

In \ouralgo{}, each new task is learned by a modified version of the P2L meta-algorithm (Algorithm~\ref{alg:modified-p2l}, entitled mP2L), as explained further down. That is, for each new task, the proposed continual learning algorithm (Algorithm~\ref{alg:p2lcl}) calls mP2L and then updates the replay buffer by sampling datapoints that were not chosen in the compression set. Similar to the replay buffer method, at the end of each task~$t$, the buffer contains~$\lfloor \tfrac{m}{t}\rfloor$ datapoints of each previous task, with~$m$ the maximum buffer size. 

We modified Pick-To-Learn in two significant ways. \textbf{(i)}~First, we introduced weights to the loss functions to tackle the imbalance problem between the current task and the previous tasks. Traditionally, when working with a replay buffer, the class imbalance problem is taken care of by training on a number of datapoints from previous tasks. As our model is only trained on a very limited subset of the data, another strategy must be implemented to mitigate the class imbalance effect. Before starting the training on a new task, mP2L sets the weight of the datapoints from the previous task to~$\omega\,{>}\,1$ and the weight of the current task to~$1$. For the stopping criteria of mP2L to be satisfied, the worst loss on the current task must be less than $\gamma$ and the worst loss on the previous tasks must be less than~$\tfrac{\gamma}{\omega}$. 
\textbf{(ii)} 
The original P2L algorithm trains the model until it first achieves zero errors on the complement set. 
Instead, mP2L performs early stopping based on the bound value \citep[as proposed by][]{bazinet2024}. That is, mP2L relies on the trade-off encoded in \cref{thm:sample_compress} between the accuracy on the complement set $C$ and the complexity of the model $f_{\theta}$. More precisely, it returns the model's parameters $\theta$ that minimizes
\begin{equation}\label{eq:boundfct}
\boundfct(S, \theta, C) \!
=  \!\kl^{-1}  \!\left(\!\hatL_{S \setminus C}(\theta),  \tfrac{1}{|S\setminus C|}\ln\tfrac{2\sqrt{|S\setminus C|}{|S| \choose |C|}}{\zeta(|C|)\delta}\right)\!,\!
\end{equation}
with $\kl^{-1}$ given by Eq.~\eqref{eq:klinv}. Thanks to \cref{thm:sample_compress}, we have that $\calL_\Dcal(\theta)\leq \boundfct(S, \theta, C)$ with high probability.

Note that for the first observed task, the mP2L procedure starts from randomly initialized parameters~$\theta_0$ and is executed on the training sample $S_1$ (as it is done by the original P2L). Then, for every subsequent task $t\in \{2,\ldots, T\}$, the mP2L procedure is initialized to the previously learned parameters~$\theta_{t-1}$. It then learns from the current task sample $S_t$ and a subset of randomly selected instances from previous tasks. The latter is obtained from the procedure \textbf{sample}$(S, m)$ (see Algorithm~\ref{alg:p2lcl}), which represents the random sampling of $m$ instances of $S$ without replacement.\footnote{Although we only discuss random sampling, any buffer management technique applied to the complement set would be valid in \ouralgo{}, such as coreset methods.}

\begin{algorithm}[t] \small
\caption{\small Continual Pick-To-Learn (\ouralgo{})}\label{alg:p2lcl}
\begin{algorithmic}[1]
\INPUT{$\theta_0$}\hfill
\COMMENT{Initial parameters of the model}
\INPUT{$S_1, S_2, \ldots, S_T$}
\hfill\COMMENT{Training sets}
\INPUT{$\gamma$} \hfill\COMMENT{mP2L's stopping criteria}
\INPUT{$m$} \hfill\COMMENT{Buffer's max sampling size}
\INPUT{$\omega$} \hfill\COMMENT{Weight for buffer tasks}

\STATE{$B_i \gets \emptyset \quad  \forall i=1,\ldots, T$}
\STATE{$B^{\star} \gets \emptyset$}
\FOR{$t\in \{1,\ldots, T\}$}
\STATE $\hat{S_t} \gets \{(\xbf, y, 1)\}_{(\xbf,y)\in S_t}$
\STATE
$\theta_t, C^\star \gets \mbox{\bfseries mP2L}(\theta_{t-1}, \hat{S_t}, B^{\star}, \gamma, \infty)$
\STATE $B_i \gets \mbox{\bfseries sample}(B_i, \lfloor\frac{m}{t}\rfloor)
\  \forall i {=} 1,\ldots, t{-}1$
\STATE{ $B_t \gets \mbox{\bfseries sample}(\hat S_t \setminus C^\star, \lfloor\frac{m}{t}\rfloor)$ \label{alg:linesampling}}
\STATE $B^\star \gets \bigcup_{i=1}^t \{(\xbf, y, \omega)\}_{(\xbf, y, \cdot)\in B_i}$
\ENDFOR
\STATE {\bfseries return } $\theta_T$. \hfill\COMMENT{Final parameters}
\end{algorithmic}
\end{algorithm}

\paragraph{Compression and reconstruction scheme.}
From a theoretical standpoint, in contrast to the original P2L, both mP2L and \ouralgo{} cannot be reconstructed straightforwardly without a message for two reasons:
(1) when only the compression set is given as input, the bound function $\boundfct(S, \theta, C)$ used as stopping criteria in mP2L cannot be computed on the whole dataset; 
(2) the sampling procedure {\bfseries sample}$(S, m)$ in \ouralgo{} cannot be reproduced in the reconstruction step. 
In Appendix~\ref{appendix:compression_reconstruction}, we prove that \ouralgo{} is a sample-compression algorithm by providing its compression and reconstruction functions, which subsequently gives rise to sample-compression bounds as presented in \cref{thm:main_results}. Noteworthy, 
after learning on $T$ tasks, the compression function provides two compression sets~$S^{\bfi}$ and~$S^{\bfj}$, along with a message pair $\mu = (\mu_1, \mu_2)$.
As explained further in the appendix, the use of two compression sets, along with the first part of the message $\mu_1 = [\mu_1^i]_{i=1}^T$, serves the proper treatment of the weighted buffer $B^\star$, while the  second part of the message $\mu_2 = [\mu_2^i]_{i=1}^T$ contains the number of iterations $K^\star$ to perform for each call to mP2L. 
The set of all possible message pairs is denoted as
$
\mathscr{M}_{1:T}(\bfj) = \{2, \ldots, T\}^{|\bfj|}\times \left[\pmb\times_{t=1}^T \{1, \ldots, n_t + |\Bcal|\}\right].
$

\paragraph{Generalization bound.}
Given the compression and reconstruction function detailed in Appendix~\ref{appendix:compression_reconstruction}, 
\cref{thm:sample_compress} can be used to obtain generalization bounds on the last learned task, which is done by mP2L to compute the bound~$\boundfct$. However, this bound holds for only one distribution of data. Therefore, it cannot be applied to bound the risks of the learned predictor on previous tasks seen by \ouralgo{}.

The next Theorem~\ref{thm:main_results} holds for any previously learned tasks. 
For a dataset $S_t \sim \calD_t$, we denote 
the loss on the complement set of the joint set
$S_t^{\bfi}\cup S_t^{\bfj}
$
as $\halLStbfcap\!(\theta)$. 
Moreover, we denote the reconstruction function of \ouralgo{} as 
$\scriptR_{1:T}\big(S_t^{\bfi}, S_t^{\bfj}, \mu\mid S_1, \ldots,S_{t-1}, S_{t+1}, \ldots, S_T\big).$
This formulation is important for the following theorem, as the reconstruction function for task $t$ is \emph{conditioned} on all datasets $S_1$ to $S_T$, with the exception of $S_t$.
\begin{restatable}{theorem}{mainresults}\label{thm:main_results}
For any set of distributions $\left\{\Dcal_t\right\}_{t=1}^T$ over $\Xcal \times \Ycal$ sampled from $\frakD$, 
and for any $\delta \in (0,1]$, with probability at least $1-\delta$ over the draw of $S_{t'} \sim \Dcal_{t'}, {t'}=1,\ldots, T$, we have 
\iflong
\begin{align*}
\forall t \in \{1,\ldots,T\},\ 
\ibf,
\bfj \in \scriptP(n_t),\ 
\mu \!=\! (\mu_1, \mu_2)\in\mathscr{M}_{1:T}(\bfj):
\Lcal_{\calD_t}\qty(
\theta_{\bfi,\bfj,\mu}^{(t)}
) 
\leq \kl^{-1}\left(\halLStbfcap\!\qty(
\theta_{\bfi,\bfj,\mu}^{(t)}
), \frac{\epsilon(\bfi, \bfj, \mu)}{n_t - |\ibf|-|\mathbf{j}|}\right),
\end{align*} 
\else
\begin{align*}
&\forall t \in \{1,\ldots,T\},\ 
\ibf,
\bfj \in \scriptP(n_t),\ 
\mu \!=\! (\mu_1, \mu_2)\in\mathscr{M}_{1:T}(\bfj):\\[1mm]
&\Lcal_{\calD_t}\qty(
\theta_{\bfi,\bfj,\mu}^{(t)}
) 
\leq \kl^{-1}\left(\halLStbfcap\!\qty(
\theta_{\bfi,\bfj,\mu}^{(t)}
), \frac{\epsilon(\bfi, \bfj, \mu)}{n_t - |\ibf|-|\mathbf{j}|}\right),
\end{align*} 
\fi
with 
$\theta_{\bfi,\bfj,\mu}^{(t)}= \scriptR_{1:T}\big(S_t^{\bfi}, S_t^{\bfj}, \mu\,{\mid}\, S_1, \ldots,S_{t-1}, S_{t+1}, \ldots, S_T\big)$ the reconstructed parameters given by \cref{alg:recon_p2lcl} and
\begin{align*}
    \epsilon(\bfi, \bfj, \mu) = \ln\left[ \tfrac{T}{\delta} \textstyle{n_t \choose \m}\textstyle{n_t - \m \choose |\mathbf{j}|} \tfrac{(T-1)^{|\mathbf{j}|}}{\zeta(|\ibf|)\zeta(|\bfj|)}\prod_{i=1}^T \tfrac{1}{\zeta(\mu_2^i)}\right]\!.
\end{align*}
\end{restatable}
The proof is given in \cref{app:proof_main}. It first relies on a tighter version of \cref{thm:sample_compress}, where the $2\sqrt{n-\m}$ term is removed and a second compression set is considered. This new result is tailored to the continual learning setting. For all tasks $1$ through $T$, \cref{thm:main_results} encodes a tradeoff similar to \cref{thm:sample_compress} between the predictor accuracy and its complexity.
Here, the predictor complexity is evaluated using the size of the compression sets $\bfi$ and $\bfj$ and the probability of choosing a message. The probability of the messages $\mu_1$ and $\mu_2$ are functions of the size of $\bfj$ and the number of tasks.  Indeed, the probability of $\mu_1$ 
decays when $|\mathbf{j}|$ and $T$ grow larger. Moreover, the probability of $\mu_2$ depends on the size of the compression set outputted by mP2L at each task $T$. Thus, the larger the compression sets, the smaller the probability of both messages.
 In conclusion, the bound encodes the ability of a model to minimize the loss on the dataset whilst still keeping $\bfi$ and $\bfj$ small.
\section{EXPERIMENTS}\label{sec:experiments}
In this section, we present experiments with \ouralgo{} to demonstrate the tightness of our new generalization bounds for continual learning in two different settings. Afterward, we compare \ouralgo{} to established baseline approaches to show that although our algorithm was first designed to obtain tight generalization bounds, its performances are still very competitive with established continual learning schemes.

\paragraph{Datasets and continual learning settings.}
We carry out experiments on
CIFAR10 \citep{cifar10cifar100}, CIFAR100 \citep{cifar10cifar100} and TinyImageNet \citep{tiny-imagenet}. We report results in Class-Incremental (CI) and Task-Incremental (TI) continual learning settings, as defined in the Avalanche framework \citep{avalanchelib}.

\paragraph{Baselines.} To assess the effectiveness of our proposed method \ouralgo{}, we benchmark against several strong and widely adopted baselines
across Class-Incremental and Task-Incremental learning settings. 

For the Class-Incremental setting, we use the baselines Replay \citep{rolnick2019experience}, Dark Experience Replay (DER) \citep{der}, iCaRL \citep{rebuffi2017icarl}, GDumb \citep{prabhu2020gdumb}, Contrastive
Continual Learning via Importance Sampling (CCLIS) \citep{Li_Azizov_LI_Liang_2024} and Coreset Selection via Reducible Loss (CSReL) \citep{tong2025coreset}

In the Task-Incremental setting, we report Replay, DER, LaMAML \citep{gupta2020look}, Learning without Forgetting (LwF) \citep{li2017learning}, CCLIS and CSReL.

\begin{figure*}[!t]
    \centering
    \begin{subfigure}[b]{.49\textwidth}
        \centering
        \includegraphics[width=\textwidth, trim=0cm 0cm 0cm 0cm,clip]{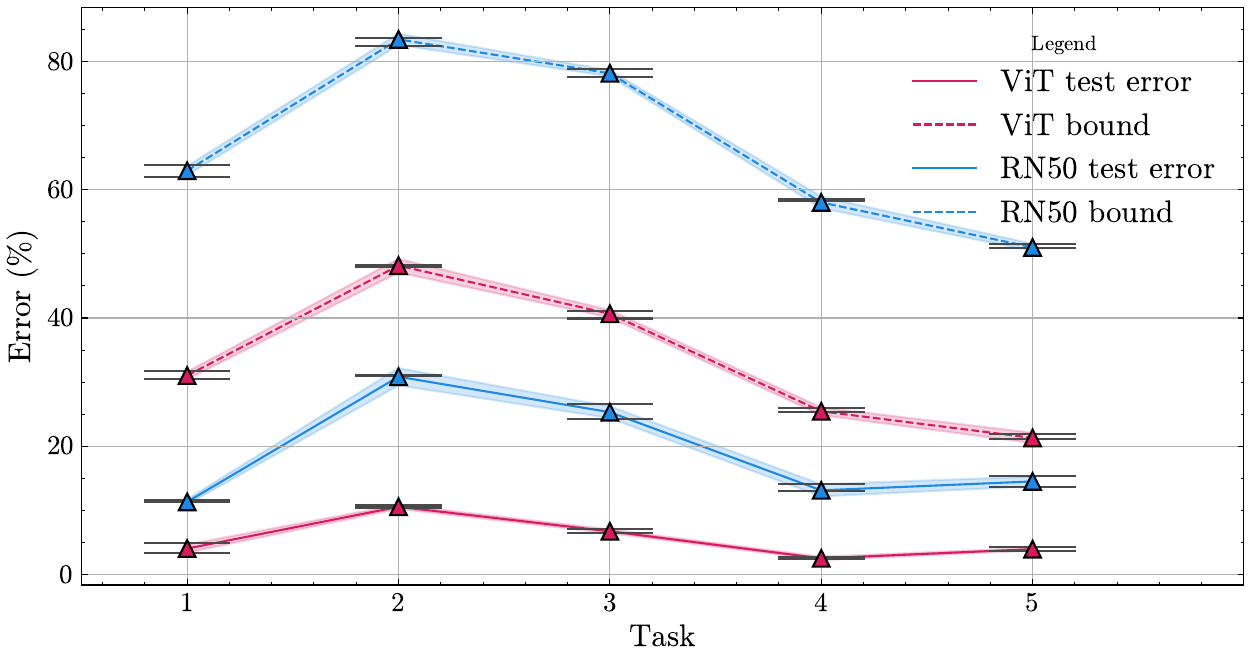} 
        \caption{Class-incremental setting}
    \end{subfigure}
    \begin{subfigure}[b]{.49\textwidth}
        \centering
        \includegraphics[width=\textwidth, trim=0cm 0cm 0cm 0cm,clip]{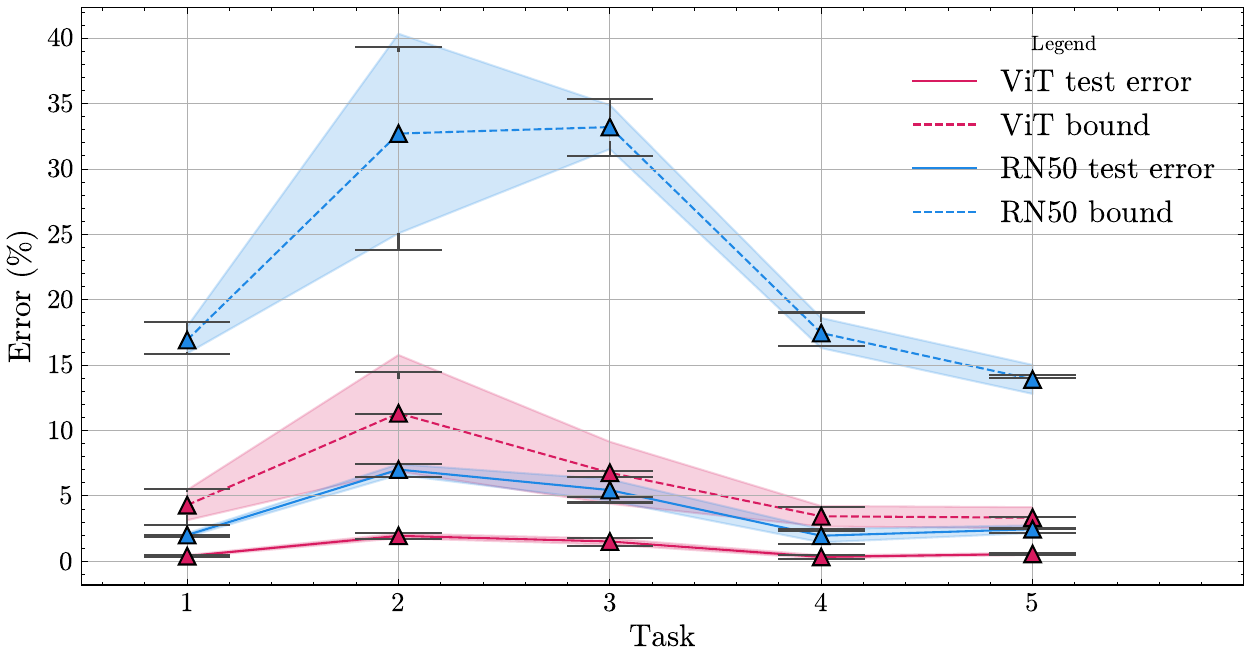}
        \caption{Task-incremental setting}
    \end{subfigure}
    \caption{Illustration of the behavior of the bound on CIFAR10 with 5 tasks.}
    \label{fig:main_paper_cifar10_5}
\end{figure*}

\begin{figure*}
        \centering
    \begin{subfigure}[b]{.49\textwidth}
        \centering
        \includegraphics[width=\textwidth, trim=0cm 0cm 0cm 0cm,clip]{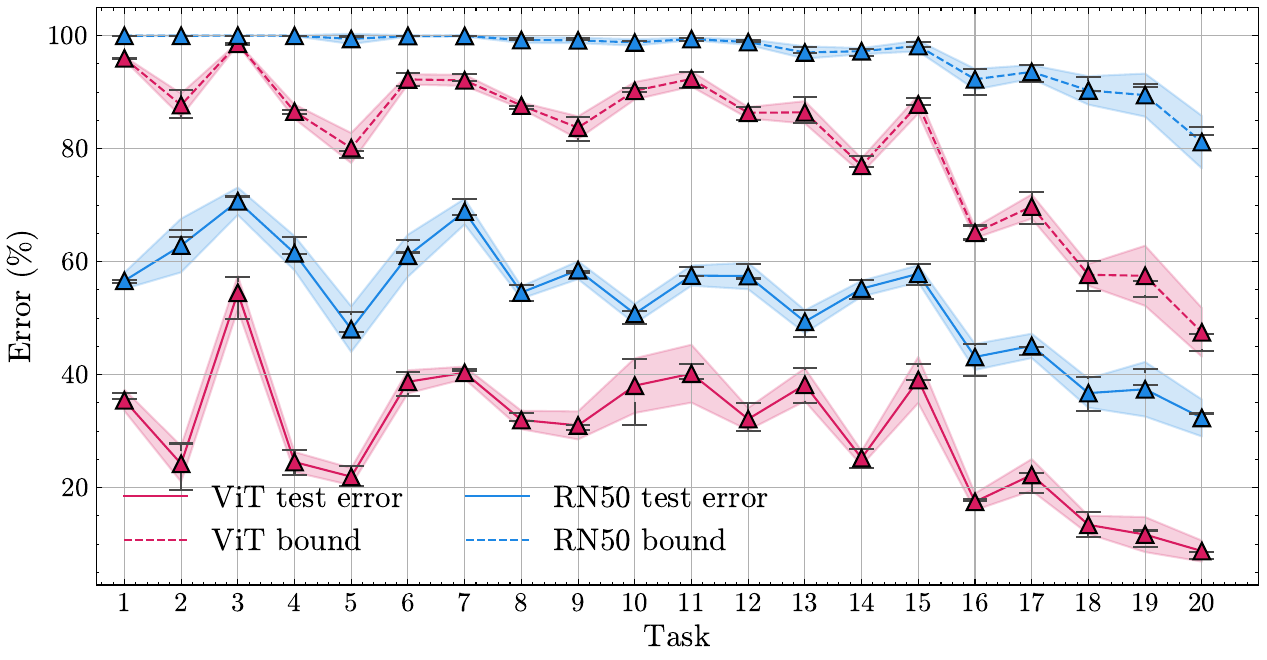} 
        \caption{Class-incremental setting}
    \end{subfigure}
    \begin{subfigure}[b]{.49\textwidth}
        \centering
        \includegraphics[width=\textwidth, trim=0cm 0cm 0cm 0cm,clip]{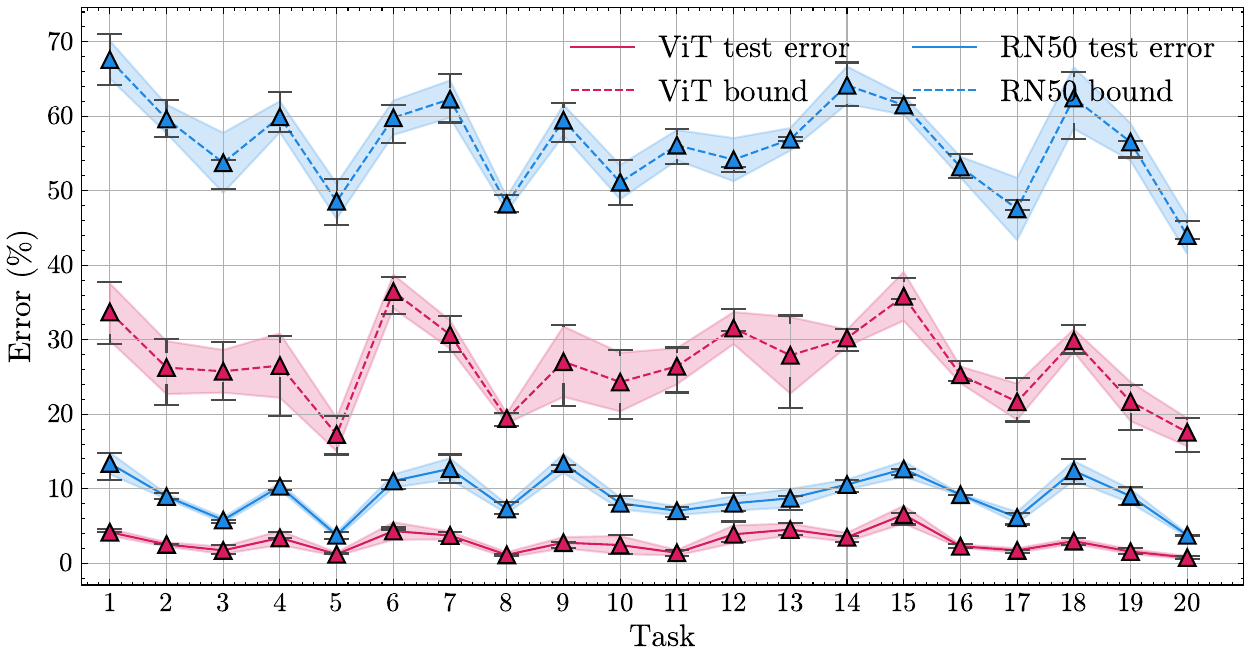}
        \caption{Task-incremental setting}
    \end{subfigure}
    \caption{Illustration of the behavior of the bound on CIFAR100 with 20 tasks.}
    \label{fig:main_paper_cifar100_20}
\end{figure*}

\begin{table*}[!t]
    \centering
    \caption{Summary of average accuracies and forgetting values obtained on various continual learning tasks on CIFAR10, CIFAR100 and TinyImageNet datasets.
    \vspace{-0.2cm}} 
    \label{tab:colorimages}
    \begin{subtable}[t]{\textwidth}
    \label{tab:cifar10_ci}
    \caption{Class-Incremental experiments}
    \centering
    \resizebox{.95\textwidth}{!}{%
        \begin{tabular}{lccccccccc}
     \toprule
     & \multicolumn{2}{c}{\small{\textit{CIFAR10 (5 tasks)}}} & \multicolumn{2}{c}{\small{\textit{CIFAR100 (10 tasks)}}} & \multicolumn{2}{c}{\small{\textit{CIFAR100 (20 tasks)}}} & \multicolumn{2}{c}{\small{\textit{TinyImNet (40 tasks)}}}\\
     \textbf{Method} & \textbf{Acc ($\uparrow$)} & \textbf{Forg ($\downarrow$)} & \textbf{Acc ($\uparrow$)} & \textbf{Forg ($\downarrow$)} & \textbf{Acc ($\uparrow$)} & \textbf{Forg ($\downarrow$)} & \textbf{Acc ($\uparrow$)} & \textbf{Forg ($\downarrow$)}  \\
     \midrule
      \textbf{CoP2L} (ViT) & 94.45 & \textbf{2.10 } & \textbf{73.49} & 17.06 & 70.56 & 21.15 & 51.59 & 36.32  \\
     Finetuning (ViT) & 27.55 & 89.74 & 12.87 & 93.27 & 8.49 & 94.64 & 2.83 & 95.36 \\
     Replay (ViT) & 94.00 & 6.11 & 68.71 & 30.02 & 69.21 & 29.52  & 47.49 & 48.39  \\
     DER (ViT) & \textbf{95.03} & 3.79 & 72.29 & 25.11 &\textbf{77.03} & 19.27 & 53.92 & 41.77\\
     iCaRL (ViT) & 88.60 & 3.13  &  65.54 & 13.64 & 62.36 & 15.88  & 21.06 & 34.49 \\
     GDumb (ViT) & 94.16 & 3.16 & 72.51 & \textbf{11.53}& 72.35 & \textbf{11.77} & \textbf{53.93} & \textbf{15.11}  \\
      CCLIS (ViT) & 93.15 & 4.44 & 67.95 & 23.68 & 66.81 & 24.42 & 48.55 & 32.26 \\
     \midrule
      \textbf{CoP2L} (RN50) & 80.98 & 5.84 &  \textbf{49.53} & 22.31& 46.74 & 22.54 & 33.30 & 35.93  \\
     Finetuning (RN50) & 19.63 & 97.22 & 9.36 & 89.93 & 4.88 & 92.81 & 1.80 & 78.66 \\
     Replay (RN50) & 81.38 & 17.72 & 49.41 & 42.12 & 49.44 & 44.96 & 28.38 & 63.36 \\
     DER (RN50) & \textbf{82.35} & 15.31 & 49.37 & 44.01 & \textbf{57.53} & 33.37 & 28.77 & 64.17\\
     iCaRL (RN50) & 74.49 & \textbf{1.79} & 46.06 & \textbf{8.50}& 46.47 & \textbf{6.74} & \textbf{33.89} & \textbf{10.42}   \\
     GDumb (RN50) & 80.84 & 9.35 & 43.37 & 20.23  & 43.60 & 19.93 & 24.29 & 21.29\\
     CCLIS (RN50) & 76.89 & 12.32 & 44.22 & 23.23 & 39.10 & 25.78 & 18.32 & 40.40 \\
     CSReL (RN50) & 34.50 & 26.44 & 38.75 & 42.12 & 32.57 & 40.23 & 12.84 & 73.69 \\
     \bottomrule
    \end{tabular}
    }
    \end{subtable}
    \newline
    \vspace*{0.001cm}
    \newline
    \begin{subtable}[t]{\textwidth}
    \label{tab:cifar10_ti}
    \caption{Task-Incremental experiments}
    \centering
    \resizebox{.95\textwidth}{!}{%
    \begin{tabular}{lccccccccc}
     \toprule
     & \multicolumn{2}{c}{\small{\textit{CIFAR10 (5 tasks)}}} & \multicolumn{2}{c}{\small{\textit{CIFAR100 (10 tasks)}}} & \multicolumn{2}{c}{\small{\textit{CIFAR100 (20 tasks)}}} & \multicolumn{2}{c}{\small{\textit{TinyImNet (40 tasks)}}} \\
     \textbf{Method} & \textbf{Acc ($\uparrow$)} & \textbf{Forg ($\downarrow$)} & \textbf{Acc ($\uparrow$)} & \textbf{Forg ($\downarrow$)} & \textbf{Acc ($\uparrow$)} & \textbf{Forg ($\downarrow$)} & \textbf{Acc ($\uparrow$)} & \textbf{Forg ($\downarrow$)}  \\
     \midrule
      \textbf{CoP2L} (ViT) & 99.04 & -0.00 & 95.03 & 0.90 & 97.18 & 0.56 & 93.81 & 1.50 \\
     Finetuning (ViT) & 99.15 & 0.23 & 95.17 & 1.84 & 97.57 & 0.86 & 93.62 & 2.61 \\
     Replay (ViT) & \textbf{99.29} & \textbf{-0.01} & \textbf{96.43} & \textbf{0.34} & \textbf{98.09} & \textbf{0.25} & \textbf{95.66} & \textbf{0.50} \\
    DER (ViT) & 99.15 & 0.21 & 96.22 & 0.71 & 98.08 & 0.31 & 95.08 & 1.13 \\
     LaMAML (ViT) &  99.25 & 0.05 & 95.58 & 0.58&  97.84 & 0.26 & 94.82 & 0.79 \\
     LwF (ViT) & 97.89 & 1.69 & 95.09 & 0.98&  95.74 & 2.15 & 85.18 & 9.72 \\
     CCLIS (ViT) & 98.81 & 0.48 & 90.99 & 4.49 & 94.86 & 2.88 & 86.82 & 6.14\\
     \midrule
      \textbf{CoP2L} (RN50) &  96.22 & 0.56 & 86.78 & 2.01 & 90.89 & 2.27 & 88.57 & 2.46  \\
     Finetuning (RN50) & 96.60 & 1.00 & 88.23 & 1.78 & 93.25 & 1.27 & 90.62 & 1.38 \\
     Replay (RN50) & \textbf{96.88} & 0.42  & 87.61 & 0.86 & 92.89 & 0.47  & 90.03 & 1.26\\
     DER (RN50) & 96.57 & 1.09 & \textbf{88.29 }& 1.72 & \textbf{93.43} & 1.11 & \textbf{90.66} & 1.12 \\
     LaMAML (RN50) & 96.65 & 0.73  & 83.57 & 6.18 & 89.89 & 4.12 & 87.20 & 4.53\\
     LwF (RN50) & 95.80 & \textbf{0.33} & 79.39 & \textbf{0.47} & 87.73 & \textbf{0.33} & 85.03 & \textbf{0.89} \\
     CCLIS (RN50) & 95.45 & 0.79 & 78.55 & 6.76 & 83.51 & 6.60 & 75.65 & 14.12\\
     CSReL (RN50) & 65.46 & 13.66 & 69.14 & 15.17 & 63.52 & 14.49 & 50.22 & 22.34 \\ 
     \bottomrule
    \end{tabular}
    }
    \end{subtable}
\end{table*}

 All the details about the experiments can be found in \cref{app:experiments}, alongside supplementary Class Incremental experiments on MNIST \citep{lecun2010mnist}, FashionMNIST (FMNIST) \citep{fashionmnist} and EMNIST \citep{cohen2017emnistextensionmnisthandwritten}. Moreover, we present Domain-Incremental (DI) experiments on PermutedMNIST \citep{goodfellow2013empirical} and RotatedMNIST \citep{ben2010theory}. We report all the results in Appendix~\ref{app:allresults}, including the standard deviations and the training times for each method. 

\paragraph{Accuracy and Forgetting.}
We denote $A(t, \theta_T)$ the accuracy obtained on a task $t$ of a model~$f_{\theta_T}$ trained on $T$ tasks. 
We report the average accuracy over $T$ tasks, $$\overline{A}(\theta_T) \,{=}\, \frac{1}{T} \sum_{t=1}^T A(t, \theta_T)\,,$$ 
and the average forgetting at task~$T$, 
$${\overline{F}(T) = \frac{1}{T-1}\sum_{t=1}^{T-1} [A(t, \theta_t) - A(t, \theta_T)]}\,.$$ 

\subsection{Study of our generalization bound}

In Figures~\ref{fig:main_paper_cifar10_5} and~\ref{fig:main_paper_cifar100_20}, we present our generalization bound on the true risk for each task for CIFAR10 with 5 tasks and CIFAR100 with 20 tasks.
These bounds are presented for Vision Transformer (ViT) \citep{dosovitskiy2020image} networks and ResNet50 (RN50) \citep{he2016deep} networks trained with a buffer size of 2000 under both Class-Incremental and Task-Incremental settings. We present the details for the computation of the bound in \cref{appendix:compression_reconstruction}.

We observe that although the bounds are calculated exclusively on the training set, their values are non-vacuous and follow the test set error trends.
\ouralgo{} achieves particularly tight generalization bounds in the Task-Incremental setting, where forgetting is inherently reduced due to access to task identities and the use of a separate classifier head for each task. These conditions allow the model to better compress task-specific knowledge, leading to both stronger theoretical guarantees and improved empirical performance. We also observe that with the ViT backbone, the bounds are significantly tighter than those obtained with the ResNet50 backbone. This is likely because the ViT backbone provides more structured and information-rich representations, which are easier to compress effectively, thereby enhancing the performance of \ouralgo{}. Furthermore, the tightness of the bounds improves when the size of the dataset increases, which explains that the bounds on CIFAR10 (5000 samples per class) are noticeably tighter than the bounds on CIFAR100 (about 500 samples per class).

In addition, the complexity of the incremental scenario plays a role: CIFAR10 with 5 tasks involve learning two classes at a time, whereas CIFAR100 with 10 tasks require learning ten classes simultaneously, further contributing to looser bounds.
Generalization bounds for multiple combinations of datasets and architectures are reported in \cref{app:allboundplots}.

\subsection{Study of the algorithm's performance}

\cref{tab:colorimages} presents results on CIFAR10, CIFAR100 and TinyImageNet datasets, for both Class-Incremental and Task-Incremental settings. We observe that in Class-Incremental settings, \ouralgo{} is competitive with most baselines in terms of accuracy while maintaining low forgetting. Alongside these baselines, we present results for finetuning a model without a replay buffer, which evidently leads to a drastic drop in performance. In Task-Incremental settings, our method remains competitive. We note that even simple finetuning performs well in this setting. This behavior is expected, since adding a new classification head for each task naturally reduces forgetting. All the while being a competitive continual learning approach, \ouralgo{} offers the additional advantage of certifying the generalization loss of the learned predictor. %
Further experiments on Imagenette and Imagewoof \citep{Howard_Imagewoof_2019} are presented in Appendix~\ref{app:experiments}. We do not present results for the ViT with CSReL, as the method is designed for small datasets ($\approx$1000 datapoints), and becomes prohibitively more expensive when the dataset and the model grow in size. With the ResNet50, running CSReL was 5 to 32 times longer than executing \ouralgo{}, while achieving much worse accuracy. 

Finally, we refer the reader to Appendix~\ref{app:ablation} for an extended analysis of \ouralgo{} behavior,
including an ablation study on the loss terms, a sensitivity analysis with respect to task ordering, an analysis of the impact of buffer size, an ablation study on the use of early stopping with the bound and an evaluation of memory cost. 
Furthermore, we report experiments that CoP2L is able to mitigate forgetting more than the other methods while still retaining plasticity, and is able to achieve a balanced performance over past tasks.  

\section{CONCLUSION}

We proposed \ouralgo{}, an algorithm rooted in the sample compression theory, {for self-certified continual learning}. To the best of our knowledge, this is the first attempt to employ sample compression theory within the continual learning context. We provided sample compression bounds for \ouralgo{} and verified empirically that they are non-vacuous and informative. We furthermore showed that on {several challenging datasets, our approach yields comparable results when compared to several strong baselines, while also providing learning guarantees for trustworthy continual learning.}
As the combination of Pick-to-Learn with experience replay leads to a successful learning scheme, we foresee that combining Pick-to-Learn with other continual learning approaches could lead to new compelling ways of obtaining self-certified predictors. 

\newpage
\bibliographystyle{plainnat}
\bibliography{references}

\clearpage
\appendix
\thispagestyle{empty}

\onecolumn
\title{Sample Compression for Self-Certified Continual Learning: \\
(Supplementary Materials)}
\maketitle
\longtrue

\section{Table of symbols}

We summarize the notation and their definitions in the following table.
\begin{table}[!h]
    \caption{Summary table of the symbols used in this paper.}
    \label{tab:of_symbols}
    \centering
    \begin{tabular}{c|l}
    \toprule
    Symbol & Definition \\ \midrule
         $\calX$ & A feature space\\
         $\calY$ & A label space\\
         $S$  & A dataset composed of datapoints in $\calX \times \calY$ \\
         $\calD$ & A distribution over $\calX \times \calY$\\
         $\frakD$ & A meta-distribution over the space of distributions $\calD$\\
         $\Theta$ & A set of learnable parameters $\theta$\\
         $f_{\theta}$ & A function $f_{\theta} : \calX \to \calY$ parametrized by $\theta \in \Theta$\\
          $A$  & An learning algorithm that returns parameters $\theta \in \Theta$ \\
          \midrule 
         $\bfi$  & The indices of the compression set $S^{\bfi}$\\
         $\bfi^{\setcomp}$ & The indices of the compression set's  complement $S^{\bfi^{\setcomp}}$\\
         $\m$  & The size of the compression set\\
         $\bfj$  & The indices of the second compression set $S^{\bfj}$ used by \ouralgo{}\\
         $\scriptP(n)$  & The powerset of $\{1,\ldots, n\}$\\
         $\mu$ & A message\\
         $\Sigma$  & The alphabet of symbols $\sigma$ used to create the message $\mu$\\
         $\Sigma^*$ & The set of all possible sequences of symbols $\sigma \in \Sigma$ \\
         $\mathscr{M}(\bfi)$ & The set of possible messages in $\Sigma^*$ given a sequence $\bfi$\\
         $\scriptC$  & The compression function\\
         $\scriptR$  & The reconstruction function\\
         $\boundfct$   & The sample-compression bound of \cref{eq:boundfct}\\
         \midrule
          $\ell$  & A loss function $\ell : \Theta \times \calX \times \calY \to [0,1]$\\
         $\hatL_S$  & The empirical loss over a dataset $S$ given a loss function $\ell$\\
          $\hatL_{S}^{\ \bfi}$ & The empirical loss over a dataset $S^{\bfi}$ given a loss function $\ell$\\
          $\calL_{\calD}$  & The true loss over a distribution $\calD$ given a loss function $\ell$\\\bottomrule
    \end{tabular}
\end{table}

\section{Compression and Reconstruction Functions}
\label{appendix:compression_reconstruction}

In line with the prevalent literature, we introduced the sample compression framework based on a compression function providing a compression set and optionally a message.
 In the case of \ouralgo{}, the reconstruction scheme relies on two compression sets~$(S^{\bfi},S^{\bfj})$ and a message pair $(\mu_1,\mu_2)$, which we explain later in this section.\footnote{The idea of using multiple compression sets appeared in \citet{marchand2002set} and \citet{marchand2003set}, but in a setting without a message.}

\begin{algorithm}[h]\small
\caption{Compression function of Continual Pick-To-Learn (\ouralgo{})}\label{alg:comp_p2lcl}
\begin{algorithmic}[1] 
\INPUT{$\theta_0$}
\hfill\COMMENT{Initialization parameters of the model}
\INPUT{$S_1, S_2, \ldots, S_T$}
\hfill\COMMENT{Training sets}
\INPUT{$\gamma$} \hfill\COMMENT{P2L's stopping criteria}
\INPUT{$m$} \hfill\COMMENT{Buffer's max sampling size}
\INPUT{$\omega$} \hfill\COMMENT{Weight for buffer tasks}
\STATE{$B_t \gets \emptyset, S^{\bfi}_t \gets \emptyset, S^{\bfj}_t \gets \emptyset, \mu_{1,t} \gets \emptyset \quad  \forall t=1,\ldots, T$}
\STATE{$B^{\star} \gets \emptyset, \mu_2 \gets \emptyset$}
\FOR{$t\in \{1,\ldots, T\}$}
\STATE $\hat{S_t} \gets \{(\xbf, y, 1)\}_{(\xbf,y)\in S_t}$
\STATE $\theta_t, C^\star\gets \mbox{\bfseries mP2L}(\theta_{t-1}, \hat{S_t}, B^{\star}, \gamma, \infty)$
\STATE $\mu_2^t \gets |C^{\star}|$
\STATE $S^{\bfi}_i \gets S^{\bfi}_i \cup (C^\star\cap S_i) \quad \forall i = 1,\ldots, t$
\STATE $B_i \gets \mbox{\bfseries sample}(B_i, \lfloor\frac{m}{t}\rfloor)
\quad \forall i = 1,\ldots, t-1$
\STATE $B_t \gets \mbox{\bfseries sample}(S_t \setminus C^\star, \lfloor\frac{m}{t}\rfloor)$
\FOR{$i \in \{1, \ldots, t-1\}$}
\FOR{$(\xbf_{i,j}, y_{i,j}, \cdot) \in B^{\star}$ and $(\xbf_{i,j}, y_{i,j}, \cdot) \notin B_i$}
\STATE $S^{\bfj}_i \gets S^{\bfj}_i \cup \{(\xbf_{i,j}, y_{i,j})\}$
\STATE $S^{\bfi}_i \gets S^{\bfi}_i \setminus \{(\xbf_{i,j}, y_{i,j})\}$
\STATE $\mu_{1,i}^{j} \gets \{t\}$
\ENDFOR
\ENDFOR
\STATE $B^\star \gets \bigcup_{i=1}^t \{(\xbf, y, \omega)\}_{(\xbf, y, \cdot)\in B_i}$
\ENDFOR
\STATE {\bfseries return } $\theta_T, \{S^{\bfi}_t\}_{t=1}^T, \{S^{\bfj}_t\}_{t=1}^T, \{\mu_{1,t}\}_{t=1}^T, \mu_2$. 
\hfill \COMMENT{Learned parameters, compression sets and message sets}
\end{algorithmic}
\end{algorithm}

\subsection{Compression Function}

Two challenges prevent \ouralgo{} (as presented by Algorithm~\ref{alg:p2lcl}) from being used as its own compression and reconstruction function, as it is done when working with the original P2L algorithm in a standard (non continual learning) setting. In this subsection, we expose each of these challenges while explaining how the proposed compression algorithm (Algorithm~\ref{alg:comp_p2lcl}) allows to circumvent them. 

The first challenge comes from the sampling of the buffer~$B^\star$ (Lines 6 to 8 of Algorithm~\ref{alg:p2lcl}). When training using \ouralgo{}, datapoints from the buffer can be chosen to be part of the compression set. At the next task, these datapoints may still be available to train the model (if they are not excluded from the buffer by the sampling step). Therefore, the reconstruction function needs to know when to remove the datapoint. The message $\mu_1$ returned by Algorithm~\ref{alg:comp_p2lcl} thus provides the task index after which the datapoint was removed during the execution of \ouralgo{}. As the reconstruction function only requires a message for datapoints who were removed from the buffer after being added to the compression set~$S^{\bfi}$, a second compression set $S^{\bfj}$ is dedicated to datapoints that necessitate a message. The message $\mu_1$ given is chosen among $\{2,\ldots, T\}$ for each datapoint belonging to $S^{\bfj}$. 

The second challenge is the use of the bound as stopping criterion in mP2L. Recall that the reconstruction function does not have access to the whole dataset. Thus, it cannot recover the bound value computed during the initial training phase (see Line~11 of \cref{alg:modified-p2l}) and use it as a stopping criterion. To address this inconvenience, the number of iterations to perform is provided to mP2L as a message $\mu_2 = (\mu_2^1, \ldots, \mu_2^T)$. That is, for each task~$t$, with a buffer $\mathcal{B}$, the message component $\mu_2^t$ is chosen in $\{1, \ldots, n_t + |\mathcal{B}|\}$, giving the number of mP2L's iterations to perform.

After learning on $T$ tasks, the compression function provides the sample compression set~$S^{\bfi}$ and ~$S^{\bfj}$, along with a message pair $(\mu_1, \mu_2)$ chosen among the set of all possible messages, denoted as
\begin{equation} \label{eq:CoP2Lmessages}
\mathscr{M}_{1:T}(\bfj) \,=\, \{2, \ldots, T\}^{|\bfj|}\times \left[\textstyle\bigtimes_{t=1}^T \{1, \ldots, n_t + |\Bcal|\}\right].
\end{equation}

\subsection{Reconstruction Function}
Under the assumption that the input parameters $\theta_0, \gamma, m,\omega$ are the same for the compression function (Algorithm~\ref{alg:comp_p2lcl}) and for \ouralgo{} (Algorithm~\ref{alg:p2lcl}), the composition of the compression and the reconstruction function must output the exact same predictor as \ouralgo{}. However, it only has access to the compression sets and the messages. The first challenge of the compression function is addressed using $\mu_1$ from line~6 to 11 of \cref{alg:recon_p2lcl}. For each datasets, we verify if a datapoint was previously excluded from the buffer by the sampling function. If so, we remove it from the buffer. The second challenge is addressed in line~4 of \cref{alg:recon_p2lcl} by giving the message $\mu_2$ to mP2L as stopping criterion.

\begin{algorithm}[t] \small
\caption{Reconstruction function of Continual Pick-To-Learn (\ouralgo{})}\label{alg:recon_p2lcl}
\begin{algorithmic}[1]
\INPUT{$\theta_0$}
\hfill\COMMENT{Initialization parameters of the model}
\INPUT{$S^{\bfi}_1, S^{\bfi}_2, \ldots, S^{\bfi}_T$}
\hfill\COMMENT{Compression sets}
\INPUT{$S^{\bfj}_1, S^{\bfj}_2, \ldots, S^{\bfj}_T$}
\hfill\COMMENT{Compression sets}
\INPUT{$\mu_{1,1}, \mu_{1,2}, \ldots, \mu_{1,T}$}\hfill\COMMENT{Message sets}
\INPUT{$\mu_2$}\hfill\COMMENT{Message set}
\INPUT{$\gamma$} \hfill\COMMENT{P2L's stopping criteria}
\INPUT{$m$} \hfill\COMMENT{Buffer's max sampling size}
\INPUT{$\omega$} \hfill\COMMENT{Weight for buffer tasks}
\STATE{$B^{\star} \gets \emptyset$}
\FOR{$t\in \{1,\ldots, T\}$}
\STATE $\hat{S_t} \gets \{(\xbf, y, 1)\}_{(\xbf,y)\in S^{\bfi}_t}$
\STATE
$\theta_t, C^{\star}\gets \mbox{\bfseries mP2L}(\theta_{t-1}, \hat{S_t}, B^\star, \gamma, \mu_2^t)$
\STATE $B_t \gets (S^{\bfi}_t \setminus C^{\star}) \cup S^{\bfj}_t$
\FOR{$i \in \{1, \ldots, t-1\}$}
\FOR{$(\bx_{i,j}, y_{i,j}, \omega_i) \in S^{\bfj}_i$}
\IF{$\mu_{1,i}^j = t$}
\STATE $B_i \gets B_i \setminus \{(\bx_{i,j}, y_{i,j}, \omega_i)\}$
\ENDIF
\ENDFOR
\STATE $B^\star \gets \bigcup_{i=1}^t \{(\xbf, y, \omega)\}_{(\xbf, y, \cdot)\in B_i}$
\ENDFOR
\ENDFOR
\STATE {\bfseries return } $\theta_T$. \hfill\COMMENT{Learned parameters}
\end{algorithmic}
\end{algorithm}

\subsection{Computation of the Bound}
To compute the bound of \cref{thm:main_results}, we need three things : the size of the first compression set $\bfi$, the size of the second compression set $\bfj$ and the message $\mu$. By simply running our \ouralgo{}, we are able to retrieve $|\bfi|+|\bfj|$, but not the size of each compression set individually. Moreover, we are completely unable to retrieve the messages. Thus, instead of implementing \cref{alg:p2lcl}, we actually implemented the compression function as described in \cref{alg:comp_p2lcl}. 

Moreover, we use a block version of P2L, as defined by Algorithm~2 of \citet{paccagnan_pick_learn_2023}. We thus add $k$ datapoints at once to the compression set at Line~6 of Algorithm~\ref{alg:modified-p2l} (the used values of $k$ for each experiments is given in \cref{app:trainingdetails}. Thus, the value of $\mu_2^t$ (Line 6 of \cref{alg:comp_p2lcl}) becomes the number of iterations of the mP2L algorithm at for task $t$. Knowing that mP2L outputs a compression set $C^{\star}_t$, and adds $k$ datapoints to the compression set at each iteration, the number of iterations is $\mu_2^t = \frac{|C^{\star}_t|}{k}$. Finally, when computing the bound, we use the code provided by \citet{viallard2021self} to invert the $\kl$ divergence.

\newpage
\section{Tighter sample compression $\kl$ bound}\label{sec:tighter_kl}

In this section, we prove a tighter $\kl$ bound that is not derived from \cref{thm:sample_compress} but proved in a very similar way. To do so, we start by providing the Chernoff test-set bound for losses in $[0,1]$ \citep{langford2005tutorial,foong2022note}.
\begin{theorem}[\citealp{foong2022note}]\label{thm:chernoff_test_set}
    Let $X_1, \ldots, X_n$ be \emph{i.i.d.} random variables with $X_i \in [0,1]$ and $\E[X_i] = p$. Then, for any $\delta \in (0,1]$, with probability at least $1-\delta$
    \begin{equation*}
        p \ \leq\ \kl^{-1}\qty(\tfrac{1}{n} \textstyle\sum_{i=1}^n X_i, \tfrac{1}{n} \ln\tfrac{1}{\delta}).
    \end{equation*}
\end{theorem}
We now prove a tighter $\kl$ sample compression bound. As our algorithm \ouralgo{} needs two compression sets, we consider a second compression set $\bfj$. To use these two compression sets, we need to first redefine the reconstruction function
\begin{equation*}
    \scriptR : \pmb\cup_{m \leq n} (\calX \times \calY)^m \times \pmb\cup_{k \leq n} (\calX \times \calY)^k \times \pmb\cup_{\bfj \in \scriptP(n)}\mathscr{M}(\bfj) \to \Theta.
\end{equation*}
We define a conditional probability distribution $P_{\scriptP(n)}(\bfj|\bfi)$ that incorporates the knowledge that a vector $\bfi$ was already drawn from $\scriptP(n)$ and that $\bfi \cap \bfj = \emptyset$. Thus, we have $\sum_{\bfi \in \scriptP(n)}\sum_{\bfj \in \scriptP(n)} P_{\scriptP(n)}(\bfi) P_{\scriptP(n)}(\bfj|\bfi) \leq 1.$ If the choice of $\bfj$ isn't conditional to the choice of $\bfi$, we simply have $P_{\scriptP(n)}(\bfj|\bfi) = P_{\scriptP(n)}(\bfj)$.
\begin{theorem}\label{corr:tighter_kl}
For any distribution $\calD$ over $\calX \times \calY$, for any family of set of messages $\{\mathscr{M}(\bfj)\, | \bfj \in \scriptP(n)\}$, for any deterministic reconstruction function $\scriptR$ , for any loss $\ell: \Theta \times \calX \times \calY \to [0,1]$ and for any $\delta \in (0,1]$, with probability at least $1-\delta$ over the draw of $S \sim \calD^n$, we have 
\begin{align*}
&\forall \bfi \in \scriptP(n), \bfj \in \scriptP(n),   \mu \in \mathscr{M}(\bfj) : \\
&\calL_{\calD}\qty(\scriptR\qty(S^{\bfi}, S^{\bfj}, \mu)) \leq \kl^{-1}\qty(\hatL_{S}^{\, \notbfi \cap \notbfj}\qty(\scriptR\qty(S^{\bfi}, S^{\bfj}, \mu)), \frac{1}{n-\m-|\bfj|}\ln\frac{1}{P_{\scriptP(n)}(\bfi)P_{\scriptP(n)}(\bfj|\bfi)P_{\mathscr{M}(\bfj)}(\mu)\delta})
    \end{align*}
\end{theorem}

\begin{proof}
Let us prove the complement of the expression in \cref{corr:tighter_kl}.
Denote the upper bound 
\begin{equation*}
    U_{\bfi, \bfj, \mu}(\delta) \ = \ \kl^{-1}\qty(\hatL_{S}^{\, \notbfi \cap \notbfj}\qty(\scriptR\qty(S^{\bfi}, S^{\bfj}, \mu)), \frac{1}{n-\m-|\bfj|}\ln\frac{1}{P_{\scriptP(n)}(\bfi)P_{\scriptP(n)}(\bfj|\bfi)P_{\mathscr{M}(\bfj)}(\mu)\delta})\,.
\end{equation*}
We have
\begin{align*}
   &\Prob_{S \sim \calD^n}\qty(\exists \bfi,\bfj \in \scriptP(n), \mu \in \mathscr{M}(\bfj) : 
    \calL_{\calD}\qty(\scriptR\qty(S^{\bfi}, S^{\bfj}, \mu)) > U_{\bfi, \bfj, \mu}(\delta)) \\
   &\leq \sum_{\bfi \in \scriptP(n)} \sum_{\bfj \in \scriptP(n)} \Prob_{S \sim \calD^n}\qty(\exists \mu \in \mathscr{M}(\bfj) : 
     \calL_{\calD}\qty(\scriptR\qty(S^{\bfi}, S^{\bfj}, \mu)) > U_{\bfi, \bfj, \mu}(\delta)) \numberthis\label{eq:union_bound_1}\\
    &\leq  \sum_{\bfi \in \scriptP(n)}\sum_{\bfj \in \scriptP(n)} \sum_{\mu \in \mathscr{M}(\bfj)} \Prob_{S \sim \calD^n}\qty(
    \calL_{\calD}\qty(\scriptR\qty(S^{\bfi}, S^{\bfj}, \mu)) > U_{\bfi, \bfj, \mu}(\delta)) \numberthis\label{eq:union_bound_2}\\
    &\leq  \sum_{\bfi \in \scriptP(n)}\sum_{\bfj \in \scriptP(n)} \sum_{\mu \in \mathscr{M}(\bfj)} \E_{S^{\bfi} \sim \calD^{\m}} \E_{S^{\bfj} \sim \calD^{|\bfj|}}\Prob_{S^{\notbfi \cap \notbfj} \sim \calD^{n-\m-|\bfj|}}\qty(
     \calL_{\calD}\qty(\scriptR\qty(S^{\bfi}, S^{\bfj}, \mu)) > U_{\bfi, \bfj, \mu}(\delta)) \numberthis\label{eq:idd_assump_tight}\\
    &\leq  \sum_{\bfi \in \scriptP(n)} \sum_{\bfj \in \scriptP(n)}\sum_{\mu \in \mathscr{M}(\bfj)} \E_{S^{\bfi} \sim \calD^{\m}}\E_{S^{\bfj} \sim \calD^{|\bfj|}}P_{\scriptP(n)}(\bfi)P_{\scriptP(n)}(\bfj|\bfi)P_{\mathscr{M}(\bfj)}(\mu)\delta \numberthis \label{eq:use_chernoff}\\
    &\leq  \sum_{\bfi \in \scriptP(n)} \sum_{\bfj \in \scriptP(n)}\sum_{\mu \in \mathscr{M}(\bfj)} P_{\scriptP(n)}(\bfi)P_{\scriptP(n)}(\bfj|\bfi)P_{\mathscr{M}(\bfj)}(\mu)\delta\\
    &\leq \delta.
\end{align*}

\cref{eq:union_bound_1,eq:union_bound_2} use the union bound, \cref{eq:idd_assump_tight} uses the \emph{i.i.d.} assumption, and \cref{eq:use_chernoff} uses the Chernoff Test-set bound of \cref{thm:chernoff_test_set} with $p= \calL_{\calD}(\scriptR(S^{\bfi},S^{\bfj}, \mu)) = \E_{(\bx,y) \sim \calD} \ell(\scriptR(S^{\bfi}, S^{\bfj},\mu), \bx, y)$ and $X_i = \ell(\scriptR(S^{\bfi}, S^{\bfj}, \mu), \bx_i, y_i) \forall i \in \notbfi$. Finally, the last inequality is obtained from the definition of $P_{\scriptP(n)}$ and $P_{\mathscr{M}(\bfj)}$.
\end{proof}

\section{Proof of Theorem~\ref{thm:main_results}}\label{app:proof_main}

\mainresults*

\begin{proof}
Let us choose $t \in [1,T]$. Let us sample $S_1, \ldots, S_{t-1}, S_{t+1}, \ldots,  S_T$.

Let 
\begin{equation*}
    U_{\bfi, \bfj, \mu}(\delta) = \kl^{-1}\qty(\hatL_{S}^{\, \notbfi \cap \notbfj}\qty(\theta_{\bfi,\bfj,\mu}^{(t)}), \frac{1}{n_t-\m-|\bfj|}\ln\frac{1}{P_{\scriptP(n_t)}(\bfi)P_{\scriptP(n_t)}(\bfj|\bfi)P_{\mathscr{M}(\bfj)}(\mu)\delta}).
\end{equation*}

Then, we have
\begin{align*}
    \Prob_{S^t}\qty(\forall \bfi \in \scriptP(n_t), \bfj \in \scriptP(n_t),  \mu \in \mathscr{M}(\bfj) : \calL_{\calD}\qty(\theta_{\bfi,\bfj,\mu}^{(t)}) \leq U_{\bfi, \bfj, \mu}(\delta)).
\end{align*}

As all the datasets (except $S^t$) are sampled beforehand, we can define $\scriptR_{1:T}(\cdot) = \scriptR\qty(\cdot;S^1, \ldots, S_{t-1}, S_{t+1}, \ldots, S_T)$ before drawing $S^t$. Thus, the reconstruction function is only a function of the dataset $S^t$, which is the setting of the result of \cref{corr:tighter_kl}. We can then lower bound this probability by $1-\delta$.

We finish the proof by applying the bound to all datasets. We have
\begin{align*}
&\Prob_{S_1,\ldots, S_T}\Biggl(
\forall \bfi , \bfj ,  \mu : \calL_{\calD}\qty(\theta_{\bfi,\bfj,\mu}^{(t)}) \leq U_{\bfi, \bfj, \mu}(\delta)\Biggr) \\
=&\E_{S_1,\ldots, S_T}\indicator\Bigg(
\forall \bfi , \bfj ,  \mu : \calL_{\calD}\qty(\theta_{\bfi,\bfj,\mu}^{(t)}) \leq U_{\bfi, \bfj, \mu}(\delta)\Biggr) \\
=&\E_{S_1,\ldots,S_{t-1}, S_{t+1},\ldots, S_T}\E_{S_t}\indicator\Bigg(
\forall \bfi , \bfj ,  \mu :  \calL_{\calD}\qty(\theta_{\bfi,\bfj,\mu}^{(t)}) \leq U_{\bfi, \bfj, \mu}(\delta)\Biggr) \\
=&\E_{S_1,\ldots,S_{t-1}, S_{t+1},\ldots, S_T}\Prob_{S_t}\Bigg(
\forall \bfi , \bfj ,  \mu :  \calL_{\calD}\qty(\theta_{\bfi,\bfj,\mu}^{(t)}) \leq U_{\bfi, \bfj, \mu}(\delta)\Biggr) \\
\geq& \E_{S_1,\ldots,S_{t-1}, S_{t+1},\ldots, S_T} 1-\delta \\
\geq& 1-\delta.
\end{align*}
We use a union bound argument to add the $\forall t \in [1,T]$, leading to the term $\frac{\delta}{T}$\,. 

As this theorem holds specifically for \ouralgo{}, we specify the distributions $P_{\scriptP(n_t)}$ and $P_{\mathscr{M}(\bfj)}$. Following the work of \citet{marchand2003set}, which used sample compression bounds with three compression sets, with $\zeta(k) = \frac{6}{\pi}(k+1)^{-1}$, we choose
$$P_{\scriptP(n_t)}(\bfi) = \mqty(n_t \\ \m )^{-1}\zeta(\m) \quad \mbox{ and }\quad P_{\scriptP(n_t)}(\bfj|\bfi) = \mqty(n_t-\m \\ |\bfj| )^{-1}\zeta(|\bfj|)\,.$$

Finally, we split the message $\mu$ into two messages $\mu_1$ and $\mu_2$. The first message is defined using the alphabet $\Sigma_1 = \{2,\ldots, T\}$. This message indicates the task number where a datapoint is sampled out of the buffer. Thus, the size of $\Sigma_1$ is $T-1$ and the probability of a symbol is $\frac{1}{T-1}$. For any vector $\bfj$, we have a sequence of length $|\bfj|$, and thus the probability of choosing each sequence is $(\frac{1}{T-1})^{|\bfj|}$. 

The second message $\mu_2$ is defined using the alphabet $\Sigma_2 = \{1,\ldots, n_t + |\Bcal|\}$. We choose to use $\zeta(k) = \frac{6}{\pi^2}(k+1)^{-2}$ as the probability distribution over each character. We know that for any $N \geq 0$, $\sum_{k=1}^{N} \zeta(k) \leq 1$. Thus, we have the probability of a sequence $\mu_2 = \mu_2^{1}\ldots \mu_2^{T}$ is $\prod_{i=1}^{T} \zeta(\mu_2^{i})$.

We define $\mathscr{M}(\bfj, T)$ the set of messages such that the sequence of symbols from $\Sigma_1$ is of length $|\mathbf{j}|$ and the sequence of symbols from $\Sigma_2$ is of length $T$ (Equation~\ref{eq:CoP2Lmessages}).
\end{proof}
\newpage

\section{Experiments}
\label{app:experiments}

We denote the licensing information of the principal assets used in this project. Notably, we use Avalanche \citep{avalanchelib} (MIT License), PyTorch \citep{Ansel_PyTorch_2_Faster_2024} (BSD 3-Clause License) and NumPy \citep{harris2020numpy} (NumPy license). For the datasets, we use the MNIST dataset \citep{lecun1998_mnist} (MIT License), the CIFAR10 and CIFAR100 datasets \citep{cifar10cifar100} and subsets of the ImageNet dataset \citep{deng2009imagenet}. There is no license information for CIFAR10 and CIFAR100, but it is  freely available at \url{https://www.cs.toronto.edu/~kriz/cifar.html} and the authors only ask for their technical report to be cited \citep{cifar10cifar100}. The ImageNet dataset \citep{deng2009imagenet} and its subsets (TinyImageNet \citep{tiny-imagenet}, ImageNette \citep{imagenette} and ImageWoof \citep{Howard_Imagewoof_2019}) are restricted to non-commercial uses and educational purposes \footnote{The licensing information of ImageNet can be found here : \url{https://image-net.org/download.php}}.

\subsection{Dataset summary}
\label{app:datasetdetails}
Datasets used for class incremental learning experiments:

\begin{itemize}
    \item CIFAR10, 5 tasks of 2 classes each (total of 10 classes)
    \item CIFAR100 (10 tasks), 10 tasks of 10 classes each (total of 100 classes)
    \item CIFAR100 (20 tasks), 20 tasks of 5 classes each (total of 100 classes)
    \item TinyImageNet (40 tasks), 40 tasks of 5 classes each  (total of 200 classes)
    \item ImageNette (5 tasks), 5 tasks of 2 classes each (total of 10 classes)
    \item ImageWoof (5 tasks), 5 tasks of 2 classes each (total of 10 classes)
\end{itemize}

\subsection{Training, architecture details and hyperparameter settings for the experiments}
\label{app:trainingdetails}

Most of our experiments are conducted using the \textbf{Avalanche} continual learning framework \citep{avalanchelib}, which provides a unified interface for datasets, benchmarks, and baseline implementations. \cem{specify which code came from where.} \mathieu{Mention which experiments do not come from avalanche} \jacob{CCLIS and CSReL only were not from Avalanche}
We use the \texttt{SimpleMLP} architecture and customized \texttt{SimpleCNN} architectures from the Avalanche toolkit for the MNIST, FMNIST, and EMNIST experiments. For CIFAR10, CIFAR100 and TinyImageNet, we use a vision transformer (ViT) \citep{dosovitskiy2020image}, a ResNet50 and a ResNet18 \citep{he2016deep}. The first two models are pretrained with Dino \citep{caron2021emerging}, frozen and followed by a trainable MLP head. ResNet18 was pretrained on ImageNet \citep{deng2009imagenet}, frozen and followed by a trainable MLP head. In the Task-Incremental setting, we follow the pretrained models with a linear layer and a MLP head for each task. All models are trained with a NVIDIA GeForce RTX 4090 GPU using the SGD optimizer.

For all experiments, please refer to \cref{tab:hyperparameter_mnist_mlp,tab:hyperparameter_mnist_cnn,tab:hyperparameter_ci,tab:hyperparameter_ti,tab:hyperparameter_di} for the hyperparameter configuration. We denote the number of samples added to the compression set at each iteration with $k$. We denote the weight value used to handle the class-imbalance problem with $\omega$ (Both $k$ and $\omega$ apply only for the \ouralgo{} (our) algorithm). We also provide the number of epochs, the training batch size and the learning rate for each experiment. For $\ouralgo{}$, we have chosen the aforementioned hyperparameters by optimizing the performance on the validation set, and for the baseline methods, we have used the hyperparameter values reported in their official implementations. We use $\gamma =-\ln(0.5)$ as stopping criterion for mP2L, as for the cross-entropy loss, this is equivalent to achieving zero errors on the complement set~$S^{\notbfi}$.

In experiments using a Dino backbone (both ResNet50 and ViT backbones), we upsampled the images to 64x64, since the backbones were pretrained on 224x224 images. We use feature extractors pre-trained on ImageNet using Dino \citep{caron2021emerging}, which are available on \url{https://github.com/facebookresearch/dino}. For experiments with a ResNet18, we use a backbone pretrained on ImageNet. For all experiments with pretrained backbones, we freeze the backbone and follow it by a linear head. In Task-Incremental learning, we follow the backbone with a linear layer and a multi-head classification layer, that is a layer that uses a different linear head for each task. We observed better performance from the pretrained feature extractor when the input image size is closer to the original training image size of Dino.

\begin{table}[t]
\caption{Hyperparameter setup for MNIST, FMNIST and EMNIST using a MLP architecture.}
\label{tab:hyperparameter_mnist_mlp}
\centering\small
\begin{tabular}{|l|c|c|c|c|c|}
\hline
\textbf{Method} & ${k}$ & ${\omega}$ & \textbf{Epochs} & \textbf{Batch size} & \textbf{Learning rate} \\
\hline
CoP2L & $8$ & $15.0$ & $10$ & $256$ & $0.001$ \\
\hline
Replay & $-$ & $-$ & $20$ & $128$ & $0.01$ \\
\hline
GDumb & $-$ & $-$ & $20$ & $128$ & $0.01$ \\
\hline
\end{tabular}
\end{table}

\vspace{-2mm}

\begin{table}[t]
\caption{Hyperparameter setup for MNIST, FMNIST and EMNIST using a CNN architecture.}
\label{tab:hyperparameter_mnist_cnn}
\centering\small
\begin{tabular}{|l|c|c|c|c|c|}
\hline
\textbf{Method} & ${k}$ & ${\omega}$ & \textbf{Epochs} & \textbf{Batch size} & \textbf{Learning rate} \\
\hline
CoP2L (MNIST \& FMNIST) & $48$ & $25.0$ & $25$ & $256$ & $0.0009$ \\
\hline
CoP2L (EMNIST) & $105$ & $20.0$ & $15$ & $130$ & $0.0006$ \\
\hline
Replay & $-$ & $-$ & $20$ & $128$ & $0.01$ \\
\hline
GDumb & $-$ & $-$ & $20$ & $128$ & $0.01$ \\
\hline
\end{tabular}
\end{table}

\begin{table}[t]
\caption{Hyperparameter setup for Class-Incremental setup.}
\label{tab:hyperparameter_ci}
\centering
\small{
\begin{tabular}{|l|c|c|c|c|c|}
\hline
\textbf{Method} & ${k}$ & ${\omega}$ & \textbf{Epochs} & \textbf{Batch size} & \textbf{Learning rate} \\
\hline
CoP2L (ViT) & $4$ & $25.0$ & $2$ & $256$ & $0.001$ \\
\hline
Finetuning (ViT) &  $-$ & $-$ & $20$ & $128$ & $0.01$ \\
\hline
DER (ViT) & $-$ & $-$ & $20$ & $128$ & $0.01$ \\
\hline
Replay (ViT) & $-$ & $-$ & $20$ & $128$ & $0.01$ \\
\hline
iCaRL (ViT) & $-$ & $-$ & $10$ & $128$ & $0.01$ \\
\hline
GDumb (ViT) & $-$ & $-$ & $20$ & $128$ & $0.01$ \\
\hline
CCLIS (ViT) & $-$ & $-$ & $50$ & $512$ & $1.0$ \\
\hline
CoP2L (RN50) & $16$ & $25.0$ & $2$ & $256$ & $0.001$ \\
\hline
Finetuning (RN50) &  $-$ & $-$ & $20$ & $128$ & $0.01$ \\
\hline
DER (RN50) &  $-$ & $-$ & $20$ & $128$ & $0.01$ \\
\hline
Replay (RN50) & $-$ & $-$ & $20$ & $128$ & $0.01$ \\
\hline
iCaRL (RN50) & $-$ & $-$ & $30$ & $128$ & $0.01$ \\
\hline
GDumb (RN50) &  $-$ & $-$ & $20$ & $128$ & $0.01$ \\
\hline
CCLIS (RN50) &  $-$ & $-$ & $50$ & $512$ & $1.0$ \\
\hline
CSReL (RN50) &  $-$ & $-$ & $100$ & $256$ & $0.001$ \\
\hline
\end{tabular}
}
\end{table}

\begin{table}[t]
\caption{Hyperparameter setup for Task-Incremental setup.}
\label{tab:hyperparameter_ti}
\centering
{\small
\begin{tabular}{|l|c|c|c|c|c|}
\hline
\textbf{Method} & ${k}$ & ${\omega}$ & \textbf{Epochs} & \textbf{Batch size} & \textbf{Learning rate} \\
\hline
CoP2L (ViT) & $4$ & $1.0$ & $2$ & $256$ & $0.001$ \\
\hline
Finetuning (ViT) & $-$ & $-$ & $10$ & $128$ & $0.01$ \\
\hline
DER (ViT) & $-$ & $-$ & $10$ & $128$ & $0.01$ \\
\hline
Replay (ViT) & $-$ & $-$ & $10$ & $128$ & $0.01$ \\
\hline
LaMAML (ViT) & $-$ & $-$ & $10$ & $10$ & $0.01$ \\
\hline
LwF (ViT) & $-$ & $-$ & $10$ & $200$ & $0.001$ \\
\hline
CCLIS (ViT) & $-$ & $-$ & $50$ & $512$ & $1.0$ \\
\hline
CoP2L (RN50) & $16$ & $1.0$ & $2$ & $256$ & $0.001$ \\
\hline
Finetuning (RN50) & $-$ & $-$ & $10$ & $128$ & $0.01$ \\
\hline
DER (RN50) & $-$ & $-$ & $10$ & $128$ & $0.01$ \\
\hline
Replay (RN50) & $-$ & $-$ & $10$ & $128$ & $0.01$ \\
\hline
LaMAML (RN50) & $-$ & $-$ & $10$ & $10$ & $0.01$ \\
\hline
LwF (RN50) & $-$ & $-$ & $10$ & $200$ & $0.001$ \\
\hline
CCLIS (RN50) & $-$ & $-$ & $50$ & $512$ & $1.0$ \\
\hline
CSReL (RN50) & $-$ & $-$ & $100$ & $256$ & $0.001$ \\
\hline
\end{tabular}
}
\end{table}

\clearpage

\begin{table}[t]
\caption{Hyperparameter setup for Domain-Incremental setup (MLP and CNN).}
\label{tab:hyperparameter_di}
\centering\small
\begin{tabular}{|l|c|c|c|c|c|}
\hline
\textbf{Method} & ${k}$ & ${\omega}$ & \textbf{Epochs} & \textbf{Batch size} & \textbf{Learning rate} \\
\hline
CoP2L & $16$ & $15.0$ & $1$ & $256$ & $0.001$ \\
\hline
Replay & $-$ & $-$ & $20$ & $128$ & $0.01$ \\
\hline
LFL & $-$ & $-$ & $3$ & $128$ & $0.01$ \\
\hline
EWC & $-$ & $-$ & $10$ & $256$ & $0.001$ \\
\hline
SI & $-$ & $-$ & $10$ & $256$ & $0.001$ \\
\hline
\end{tabular}
\end{table}

\subsection{All results with standard deviation and training time}
\label{app:allresults}

In this section, we provide the average accuracy values obtained on CIFAR10, CIFAR100, TinyImageNet, ImageNette and ImageWoof datasets with the standard deviation estimates. Note that we use different random number generator seeds to account for variability in random number generations (e.g., model parameter initializations) -- but train / test splits are kept fixed with respect to seeds across different experiments.

We present new CI results on MNIST, FMNIST and EMNIST, with additional Domain-Incremental (DI) experiments on PermutedMNIST and RotatedMNIST. We evaluate against Replay, Less Forgetful Learning (LFL) \citep{jung2018less}, Elastic Weight Consolidation (EWC) \citep{kirkpatrick2017overcoming} and Synaptic Intelligence (SI) \citep{zenke2017continual}.

Note that, for CIFAR10, CIFAR100, TinyImageNet, ImageWoof, ImageNette, PermutedMNIST and RotatedMNIST experiments  we also report the total time it takes to complete each experiment (for all tasks).

\subsubsection{Class-Incremental Learning problems}

\begin{table}[h]
\caption{Class-incremental learning on CIFAR10 with 5 tasks.}
\centering\small
\begin{tabular}{p{4cm}ccc}
 \toprule
 \textbf{Methods} & \textbf{Average Accuracy (\%)} & \textbf{Average Forgetting (\%)} & \textbf{Total time (s)} \\ \midrule
 CoP2L (ViT) & $94.45 \pm 0.14$ & $2.10 \pm 0.28$ & $118.07 \pm 2.59$ \\
 Finetuning (ViT) & $27.55 \pm 1.04$ & $89.74 \pm 1.26$ & $880.59 \pm 8.33$ \\
 Replay (ViT) & $94.00 \pm 0.19$ & $6.11 \pm 0.25$ & $486.25 \pm 43.07$ \\
 DER (ViT) & $95.03\pm0.06$ & $3.79\pm0.15$ & $436.07\pm1.21$ \\
 iCaRL (ViT) & $88.60 \pm 10.04$ & $3.13 \pm 4.20$ & $884.59 \pm 1.39$ \\
 GDumb (ViT) & $94.16 \pm 0.24$ & $3.16 \pm 0.39$ & $47.72 \pm 0.12$ \\
 CCLIS (ViT) & $93.15 \pm 0.34$ & $4.44 \pm 0.67$ & $2527.15 \pm 44.56$ \\
 \hline
 CoP2L (ResNet50) & $80.98 \pm 0.17$ & $5.84 \pm 0.37$ & $527.95 \pm 8.43$ \\
 Finetuning (ResNet50) & $19.63 \pm 0.02$ & $97.22 \pm 0.15$ & $819.80 \pm 2.81$ \\
 Replay (ResNet50) & $81.38 \pm 0.39$ & $17.72 \pm 0.51$ & $287.38 \pm 25.06$ \\
 DER (ResNet50) & $82.35\pm0.22$ & $15.31\pm0.33$ & $285.65\pm0.69$ \\
 iCaRL (ResNet50) & $74.49 \pm 0.22$ & $1.79 \pm 0.43$ & $579.17 \pm 2.39$ \\
 GDumb (ResNet50) & $80.84 \pm 0.14$ & $9.35 \pm 0.30$ & $31.98 \pm 0.12$  \\
 CCLIS (ResNet50) & $76.89\pm 0.30$ & $12.32\pm0.33$ & $2654.32 \pm 52.34$ \\
 CSReL (ResNet50) & $34.50\pm0.29$ & $26.44\pm0.11$ & $17284.41\pm157.43$ \\
 \hline
\end{tabular}
\end{table}

\begin{table}[h]
\caption{Class-incremental learning on CIFAR100 with 10 tasks.}
\centering\small
\begin{tabular}{p{4cm}ccc}
 \toprule
 \textbf{Methods} & \textbf{Average Accuracy (\%)} & \textbf{Average Forgetting (\%)} & \textbf{Total time (s)} \\
 \midrule
 CoP2L (ViT) & $73.49 \pm 0.27$ & $17.06 \pm 0.29$ & $522.18 \pm 3.90$ \\
 Finetuning (ViT) & $12.87 \pm 0.15$ & $93.27 \pm 0.20$ & $901.37 \pm 1.73$ \\
 Replay (ViT) & $68.71 \pm 0.25$ & $30.02 \pm 0.31$ & $453.04 \pm 11.93$ \\
 DER (ViT) & $72.29\pm0.53$ & $25.11\pm0.53$ & $489.57\pm7.50$ \\
 iCaRL (ViT) & $65.54 \pm 2.23$ & $13.64 \pm 0.64$ & $2491.81 \pm 2889.82$ \\
 GDumb (ViT) & $72.51 \pm 0.26$ & $11.53 \pm 0.65$ & $95.43 \pm 0.34$ \\
 CCLIS (ViT) & $67.95 \pm 0.65$ & $23.68 \pm 0.19$ & $3023.03 \pm 87.54$ \\
 \midrule
 CoP2L (ResNet50) & $49.53 \pm 0.24$ & $22.31 \pm 0.20$ & $1724.15 \pm 23.50$  \\
 Finetuning (ResNet50) & $9.36 \pm 0.05$ & $89.93 \pm 0.02$ & $823.61 \pm 2.81$ \\
 Replay (ResNet50) & $49.41 \pm 0.38$ & $42.12 \pm 0.54$ & $295.80 \pm 15.88$ \\
 DER (ResNet50) & $49.37\pm0.16$ & $44.01\pm0.18$ & $307.86\pm13.37$ \\
 iCaRL (ResNet50) & $46.06 \pm 0.40$ & $8.50 \pm 0.28$ & $676.76 \pm 56.25$ \\
 GDumb (ResNet50) & $43.37 \pm 0.24$ & $20.23 \pm 0.42$ & $63.58 \pm 0.36$ \\
 CCLIS (ResNet50) & $44.22\pm0.28$ & $23.23\pm0.10$ & $3122.45 \pm 78.87$ \\
 CSReL (ResNet50) & $38.75\pm0.25 $ & $42.12\pm0.48$ & $ 9154.71\pm278.30$\\
 \bottomrule
\end{tabular}
\end{table}

\begin{table}[t]
\caption{Class-incremental learning on CIFAR100 with 20 tasks.}
\centering \small
\begin{tabular}{p{4cm}ccc}
 \toprule
 \textbf{Methods} & \textbf{Average Accuracy (\%)} & \textbf{Average Forgetting (\%)} & \textbf{Total time (s)} \\
 \midrule
 CoP2L (ViT) & $70.56 \pm 0.99$ & $21.15 \pm 2.10$ & $820.66 \pm 176.27$ \\
 Finetuning (ViT) & $8.49 \pm 0.59$ & $94.64 \pm 0.54$ & $902.33 \pm 8.58$ \\
 Replay (ViT) & $69.21 \pm 0.26$ & $29.52 \pm 0.25$ & $457.46 \pm 2.07$ \\
 DER (ViT) & $77.03\pm0.34$ & $19.27\pm0.48$ & $1096.62\pm19.46$ \\
 iCaRL (ViT) & $62.36 \pm 1.46$ & $15.88 \pm 0.84$ & $1359.13 \pm 4.58$ \\
 GDumb (ViT) & $72.35 \pm 0.30$ & $11.77 \pm 0.40$ & $190.80 \pm 0.42$ \\
 CCLIS (ViT) & $66.81 \pm 0.34$ & $24.42 \pm 0.29$ & $3947.38 \pm 89.23$ \\
 \midrule
 CoP2L (ResNet50) & $46.74 \pm 1.63$ & $22.54 \pm 0.92$ & $1771.12 \pm 757.50$ \\
 Finetuning (ResNet50) & $4.88 \pm 0.03$ & $92.81 \pm 0.31$ & $1247.98 \pm 506.91$ \\
 Replay (ResNet50) & $49.44 \pm 0.30$ & $44.96 \pm 0.36$ & $293.25 \pm 1.77$ \\
 DER (ResNet50) & $57.53\pm0.05$ & $33.37\pm0.05$ & $510.52\pm99.48$ \\
 iCaRL (ResNet50) & $46.47 \pm 0.33$ & $6.74 \pm 0.31$ & $858.11 \pm 120.41$ \\
 GDumb (ResNet50) & $43.60 \pm 0.38$ & $19.93 \pm 0.56$ & $126.74 \pm 0.46$ \\
CCLIS (ResNet50) & $39.10\pm0.04$ & $25.78\pm0.03$ & $4021.21 \pm 59.34$ \\
 CSReL (ResNet50) & $32.57\pm0.02$ & $40.23\pm0.27$ & $13404.29\pm320.31$\\
 \bottomrule
\end{tabular}
\end{table}

\begin{table}[t]
\caption{Class-incremental learning on TinyImageNet with 40 tasks.}
\centering\small
\begin{tabular}{p{4cm}ccc}
 \toprule
 \textbf{Methods} & \textbf{Average Accuracy (\%)} & \textbf{Average Forgetting (\%)} & \textbf{Total time (s)} \\
 \midrule
 CoP2L (ViT) & $51.59 \pm 0.24$ & $36.32 \pm 0.29$ & $1451.93 \pm 4.71$ \\
 Finetuning (ViT) & $2.83 \pm 0.17$ & $95.36 \pm 0.31$ & $2304.54 \pm 37.18$ \\
 Replay (ViT) & $47.49 \pm 0.47$ & $48.39 \pm 0.49$ & $1128.97 \pm 16.95$\\
 DER (ViT) & $53.92\pm0.36$ & $41.77\pm0.45$ & $1280.76\pm131.53$ \\
 iCaRL (ViT) & $21.06 \pm 1.11$ & $34.49 \pm 9.37$ & $3919.23 \pm 2170.08$ \\
 GDumb (ViT) & $53.93 \pm 0.34$ & $15.11 \pm 0.19$ & $471.23 \pm 17.85$ \\
 CCLIS (ViT) & $48.55 \pm 0.46$ & $32.26 \pm 0.43$ & $8000.69 \pm 355.45$ \\
 \midrule
 CoP2L (ResNet50) & $33.30 \pm 0.32$ & $35.93 \pm 0.27$ & $3671.69 \pm 29.19$ \\
 Finetuning (ResNet50) & $1.80 \pm 0.36$ & $78.66 \pm 1.04$ & $2167.76 \pm 161.91$ \\
 Replay (ResNet50) & $28.38 \pm 0.42$ & $63.36 \pm 0.44$ & $882.12 \pm 180.18$\\
 DER (ResNet50) & $28.77\pm0.34$ & $64.17\pm0.34$ & $818.11\pm10.28$ \\
 iCaRL (ResNet50) & $33.89 \pm 0.47$ & $10.42 \pm 0.44$ & $1905.26 \pm 295.18$ \\
 GDumb (ResNet50) & $24.29 \pm 0.58$ & $21.29 \pm 0.56$ & $332.22 \pm 0.98$ \\
CCLIS (ResNet50) & $18.32\pm0.11$ & $40.40\pm0.02$ & $8349.31 \pm 378.39$ \\
 CSReL (ResNet50) & $12.84\pm0.22$ & $73.69\pm0.40$ & $39644.46\pm237.23$ \\
 \bottomrule
\end{tabular}
\end{table}

\begin{table}[t]
\caption{Class-incremental learning on MNIST with 5 tasks.}
\centering\small
\begin{tabular}{p{4cm}ccc}
 \toprule
 \textbf{Methods} & \textbf{Average Accuracy (\%)} & \textbf{Average Forgetting (\%)} \\
 \midrule
 CoP2L (MLP) & $95.12 \pm 0.18$ & $3.35 \pm 0.24$ \\
 Replay (MLP) & $94.48 \pm 0.28$ & $6.08 \pm 0.32$ \\
 GDumb (MLP) & $91.70 \pm 0.31$ & $3.84 \pm 0.29$ \\
 \midrule
 CoP2L (CNN) & $96.61 \pm 0.07$ & $2.16 \pm 0.28$ \\
 Replay (CNN) & $96.93 \pm 0.29$ & $3.31 \pm 0.38$ \\
 GDumb (CNN) & $94.41 \pm 0.13$ & $2.97 \pm 0.46$ \\
 \bottomrule
\end{tabular}
\end{table}

\begin{table}[t]
\caption{Class-incremental learning on Fashion-MNIST with 5 tasks.}
\centering\small
\begin{tabular}{p{4cm}cc}
 \toprule
 \textbf{Methods} & \textbf{Average Accuracy (\%)} & \textbf{Average Forgetting (\%)} \\
 \midrule
 CoP2L (MLP) & $84.23 \pm 0.27$ & $12.10 \pm 0.66$ \\
 Replay (MLP) & $83.99 \pm 0.52$ & $15.92 \pm 0.42$ \\
 GDumb (MLP) & $81.76 \pm 0.36$ & $10.44 \pm 0.50$ \\
 \midrule
 CoP2L (CNN) & $87.35 \pm 0.15$ & $11.12 \pm 0.36$ \\
 Replay (CNN) & $87.59 \pm 0.11$ & $13.32 \pm 0.29$ \\
 GDumb (CNN) & $84.98 \pm 0.21$ & $8.78 \pm 0.41$ \\
 \bottomrule
\end{tabular}
\end{table}

\begin{table}[t]
\caption{Class-incremental learning on EMNIST with 13 tasks.}
\centering\small
\begin{tabular}{p{4cm}cc}
 \toprule
 \textbf{Methods} & \textbf{Average Accuracy (\%)} & \textbf{Average Forgetting (\%)} \\
 \midrule
 CoP2L (MLP) & $78.03 \pm 0.17$ & $16.67 \pm 0.31$ \\
 Replay (MLP) & $77.22 \pm 0.44$ & $22.08 \pm 0.50$ \\
 GDumb (MLP) & $62.64 \pm 0.20$ & $15.96 \pm 0.58$  \\
 \midrule
 CoP2L (CNN) & $87.90 \pm 0.41$ & $8.38 \pm 0.54$ \\
 Replay (CNN) & $87.70 \pm 0.50$ & $12.29 \pm 0.49$ \\
 GDumb (CNN) & $85.05 \pm 0.27$ & $7.50 \pm 0.61$ \\
 \bottomrule
\end{tabular}
\end{table}

\clearpage
\pagebreak

\subsubsection{Task-Incremental Learning problems}

\begin{table}[h]
\caption{Task-incremental learning on CIFAR10 with 5 tasks.}
\centering\small
\begin{tabular}{p{4cm}ccc}
 \toprule
 \textbf{Methods} & \textbf{Average Accuracy (\%)} & \textbf{Average Forgetting (\%)} & \textbf{Total time (s)} \\
 \midrule
 CoP2L (ViT)  & $99.04 \pm 0.05$ & $-0.00 \pm 0.10$ & $64.51 \pm 0.68$\\
 Finetuning (ViT) & $99.15 \pm 0.06$ & $0.23 \pm 0.04$ & $466.97 \pm 11.76$ \\
 Replay (ViT) & $99.29 \pm 0.04$ & $-0.01 \pm 0.03$ & $217.83 \pm 0.76$ \\
 DER (ViT) & $99.15\pm0.08$ & $0.21\pm0.05$ & $133.94\pm1.02$ \\
 LaMAML (ViT) & $99.25 \pm 0.06$ & $0.05 \pm 0.07$ & $344.65 \pm 1.46$ \\
 LwF (ViT) & $97.89 \pm 0.41$ & $1.69 \pm 0.48$ & $165.41 \pm 0.75$ \\
 CCLIS (ViT) & $98.81 \pm 0.89$ & $0.48 \pm 0.12$ & $2534.15 \pm 53.56$ \\
 \midrule
 CoP2L (ResNet50) & $96.22 \pm 0.18$ & $0.56 \pm 0.28$ & $172.51 \pm 4.42$ \\
 Finetuning (ResNet50) & $96.60 \pm 0.08$ & $1.00 \pm 0.11$ & $460.78 \pm 3.38$ \\
 Replay (ResNet50) & $96.88 \pm 0.06$ & $0.42 \pm 0.13$ & $144.57 \pm 0.51$ \\
 DER (ResNet50) & $96.57\pm0.10$ & $1.09\pm0.14$ & $90.56\pm0.36$ \\
 LaMAML (ResNet50) & $96.65 \pm 0.21$ & $0.73 \pm 0.21$ & $310.51 \pm 1.64$\\
 LwF (ResNet50) & $95.80 \pm 0.15$ & $0.33 \pm 0.20$ & $97.09 \pm 0.21$\\
CCLIS (ResNet50) & $95.45\pm0.02$ & $0.79\pm0.13$ & $2637.16 \pm 61.23$ \\
 CSReL (ResNet50) & $65.46\pm0.56$ & $13.66\pm0.19$ & $12354.00\pm160.45$\\
 \bottomrule
\end{tabular}
\end{table}

\begin{table}[h]
\caption{Task-incremental learning on CIFAR100 with 10 tasks.}
\centering\small
\begin{tabular}{p{4cm}ccc}
 \toprule
 \textbf{Methods} & \textbf{Average Accuracy (\%)} & \textbf{Average Forgetting (\%)} & \textbf{Total time (s)} \\
 \midrule
 CoP2L (ViT) & $95.03 \pm 0.08$ & $0.90 \pm 0.11$ & $290.95 \pm 30.98$ \\
 Finetuning (ViT) & $95.17 \pm 0.11$ & $1.84 \pm 0.13$ & $478.64 \pm 6.32$ \\
 Replay (ViT) & $96.43 \pm 0.14$ & $0.34 \pm 0.17$ & $235.81 \pm 1.10$ \\
 DER (ViT) & $96.22\pm0.12$ & $0.71\pm0.13$ & $137.12\pm1.96$ \\
 LaMAML (ViT) & $95.58 \pm 0.16$ & $0.58 \pm 0.12$ & $397.78 \pm 3.40$ \\
 LwF (ViT) & $95.09 \pm 0.10$ & $0.98 \pm 0.10$ & $173.57 \pm 0.63$ \\
 CCLIS (ViT) & $90.99 \pm 0.45$ & $4.49 \pm 0.69$ & $2532.34 \pm 49.57$ \\
 \midrule
 CoP2L (ResNet50) & $86.78 \pm 0.16$ & $2.01 \pm 0.17$ & $435.10 \pm 7.74$ \\
 Finetuning (ResNet50) & $88.23 \pm 0.09$ & $1.78 \pm 0.08$ & $430.99 \pm 2.11$ \\
 Replay (ResNet50) & $87.61 \pm 0.05$ & $0.86 \pm 0.16$ & $157.05 \pm 0.14$ \\
 DER (ResNet50) & $88.29\pm0.20$ & $1.72\pm0.28$ & $91.77\pm1.07$ \\
 LaMAML (ResNet50) & $83.57 \pm 0.43$ & $6.18 \pm 0.53$ & $357.24 \pm 4.58$ \\
 LwF (ResNet50) & $79.39 \pm 0.13$ & $0.47 \pm 0.11$ & $101.27 \pm 0.48$ \\
 CCLIS (ResNet50) & $78.55\pm0.09$ & $6.76\pm0.21$ & $2659.12 \pm 43.02$  \\
 CSReL (ResNet50) & $69.14\pm0.14$ & $15.17\pm0.41$ & $5670.32\pm256.32$\\
 \bottomrule
\end{tabular}
\end{table}

\begin{table}[h]
\caption{Task-incremental learning with CIFAR100 with 20 tasks.}
\centering\small
\begin{tabular}{p{4cm}ccc}
 \toprule
 \textbf{Methods} & \textbf{Average Accuracy (\%)} & \textbf{Average Forgetting (\%)} & \textbf{Total time (s)} \\
 \midrule
 CoP2L (ViT) & $97.18 \pm 0.13$ & $0.56 \pm 0.18$ & $226.65 \pm 21.43$ \\
 Finetuning (ViT) & $97.57 \pm 0.23$ & $0.86 \pm 0.22$ & $473.71 \pm 6.70$ \\
 Replay (ViT) & $98.09 \pm 0.09$ & $0.25 \pm 0.12$ & $250.78 \pm 1.45$ \\
 DER (ViT) & $98.08\pm0.10$ & $0.31\pm0.17$ & $137.98\pm0.24$ \\
 LaMAML (ViT) & $97.84 \pm 0.10$ & $0.26 \pm 0.14$ & $448.43 \pm 2.81$ \\
 LwF (ViT) & $95.74 \pm 0.19$ & $2.15 \pm 0.21$ & $180.36 \pm 0.98$ \\
 CCLIS (ViT) & $94.86 \pm 0.54$ & $2.88 \pm 0.45$ & $3951.13 \pm 68.94$ \\
 \midrule
 CoP2L (ResNet50) & $90.89 \pm 0.22$ & $2.27 \pm 0.28$ & $278.67 \pm 6.20$ \\
 Finetuning (ResNet50) & $93.25 \pm 0.04$ & $1.27 \pm 0.06$ & $436.93 \pm 2.10$ \\
 Replay (ResNet50) & $92.89 \pm 0.18$ & $0.47 \pm 0.24$ & $167.65 \pm 1.26$ \\
 DER (ResNet50) & $93.43\pm0.02$ & $1.11\pm0.05$ & $94.25\pm0.30$ \\
 LaMAML (ResNet50) & $89.89 \pm 0.75$ & $4.12 \pm 0.78$ & $407.84 \pm 7.38$ \\
 LwF (ResNet50) & $87.73 \pm 0.18$ & $0.33 \pm 0.23$ & $105.46 \pm 0.49$ \\
 CCLIS (ResNet50) & $83.51\pm0.14$ & $6.60\pm0.11$ & $4079.41 \pm 61.34$ \\
 CSReL (ResNet50) & $63.52\pm0.32$ & $14.49\pm0.22$ & $8364.45\pm289.56$\\
 \bottomrule
\end{tabular}
\end{table}

\begin{table}[h]
\caption{Task-incremental learning on TinyImageNet with 40 tasks.}
\centering\small
\begin{tabular}{p{4cm}ccc}
 \toprule
 \textbf{Methods} & \textbf{Average Accuracy (\%)} & \textbf{Average Forgetting (\%)} & \textbf{Total time (s)} \\
 \midrule
 CoP2L (ViT) & $93.81 \pm 0.26$ & $1.50 \pm 0.28$ & $1039.14 \pm 146.21$ \\
 Finetuning (ViT) & $93.62 \pm 0.20$ & $2.61 \pm 0.21$ & $983.62 \pm 9.49$ \\
 Replay (ViT) & $95.66 \pm 0.08$ & $0.50 \pm 0.09$ & $632.18 \pm 2.97$ \\
 DER (ViT) & $95.08\pm0.07$ & $1.13\pm0.15$ & $316.37\pm0.81$ \\
 LaMAML (ViT) & $94.82 \pm 0.23$ & $0.79 \pm 0.19$ & $1527.24 \pm 891.43$ \\
 LwF (ViT) & $85.18 \pm 0.76$ & $9.72 \pm 0.71$ & $430.63 \pm 2.14$ \\
 CCLIS (ViT) & $86.82 \pm 0.23$ & $6.14 \pm 0.56$ & $8012.79 \pm 321.46$ \\
 \midrule
 CoP2L (ResNet50) & $88.57 \pm 0.17$ & $2.46 \pm 0.19$ & $890.94 \pm 3.44$ \\
 Finetuning (ResNet50) & $90.62 \pm 0.35$ & $1.38 \pm 0.40$ & $905.9 \pm 10.10$ \\
 Replay (ResNet50) & $90.03 \pm 0.26$ & $1.26 \pm 0.27$ & $514.42 \pm 102.82$ \\
 DER (ResNet50) & $90.66\pm0.06$ & $1.12\pm0.15$ & $237.61\pm1.12$ \\
 LaMAML (ResNet50) & $87.20 \pm 0.34$ & $4.53 \pm 0.28$ & $1268.50 \pm 551.93$ \\
 LwF (ResNet50) & $85.03 \pm 0.17$ & $0.89 \pm 0.21$ & $284.79 \pm 5.67$ \\
 CCLIS (ResNet50) & $75.65\pm0.05$ & $14.12\pm0.31$ & $8379.89 \pm 381.29$ \\
 CSReL (ResNet50) & $50.22\pm0.21$ & $22.34\pm0.51$ & $21455.97\pm231.76$\\
 \bottomrule
\end{tabular}
\end{table}

\begin{table}[h]
\caption{Task-incremental learning on ImageNette with 5 tasks.}
\centering\small
\begin{tabular}{p{4cm}ccc}
 \toprule
 \textbf{Methods} & \textbf{Average Accuracy (\%)} & \textbf{Average Forgetting (\%)} & \textbf{Total time (s)} \\
 \midrule
 CoP2L (ViT)  & $98.22 \pm 0.19$ & $-0.43 \pm 0.12$ & $118.71 \pm 24.35$ \\
 Finetuning (ViT) & $98.24 \pm 0.22$ & $0.53 \pm 0.17$ & $221.15 \pm 2.16$ \\
 Replay (ViT) & $98.61 \pm 0.03$ & $-0.02 \pm 0.10$ & $356.58 \pm 2.89$ \\
 LaMAML (ViT) & $98.42 \pm 0.02$ & $0.03 \pm 0.00$ & $309.91 \pm 5.42$ \\
 LwF (ViT) & $96.51 \pm 0.67$ & $1.97 \pm 0.79$ & $164.80 \pm 1.90$ \\
 \bottomrule
\end{tabular}
\end{table}

\begin{table}[h]
\caption{Task-incremental learning on ImageWoof with 5 tasks.}
\centering\small
\begin{tabular}{p{4cm}ccc}
 \toprule
 \textbf{Methods} & \textbf{Average Accuracy (\%)} & \textbf{Average Forgetting (\%)} & \textbf{Total time (s)} \\
 \midrule
 CoP2L (ViT) & $93.34 \pm 0.30$ & $-0.65 \pm 0.40$ & $246.77 \pm 13.43$ \\
 Finetuning (ViT) & $93.20 \pm 0.49$ & $2.00 \pm 0.84$ & $201.91 \pm 1.89$ \\
 Replay (ViT) & $94.29 \pm 0.30$ & $0.16 \pm 0.26$ & $316.87 \pm 0.93$ \\
 LaMAML (ViT) & $94.00 \pm 0.21$ & $0.12 \pm 0.38$ & $276.19 \pm 3.45$ \\
 LwF (ViT) & $89.75 \pm 1.88$ & $5.76 \pm 2.61$ & $175.82 \pm 39.30$ \\
 \bottomrule
\end{tabular}
\end{table}

\clearpage
\pagebreak

\subsubsection{Domain-Incremental Learning problems}
\begin{table}[h]
\caption{Domain-incremental learning on Permuted-MNIST with 10 tasks.}
\centering\small
\begin{tabular}{p{4cm}ccc}
 \toprule
 \textbf{Methods} & \textbf{Average Accuracy (\%)} & \textbf{Average Forgetting (\%)} & \textbf{Total time (s)} \\
 \midrule
 CoP2L (MLP) & $91.61 \pm 0.39$ & $6.69 \pm 0.40$ & $1381.06 \pm 1232.99$ \\
 Replay (MLP) & $91.28 \pm 0.18$ & $6.75 \pm 0.21$ & $1128.99 \pm 49.82$\\
 LFL (MLP) & $82.76 \pm 1.23$ & $10.15 \pm 1.36$ & $92.86 \pm 0.24$ \\
 EWC (MLP) & $82.53 \pm 0.82$ & $6.50 \pm 0.88$ & $340.16 \pm 3.81$ \\
 SI (MLP) & $85.85 \pm 0.61$ & $11.69 \pm 0.67$ & $319.95 \pm 1.38$ \\
 \bottomrule
\end{tabular}
\end{table}

\begin{table}[h]
\caption{Domain-incremental learning on Rotated-MNIST on 10 tasks.}
\centering\small
\begin{tabular}{p{4cm}ccc}
 \toprule
 \textbf{Methods} & \textbf{Average Accuracy (\%)} & \textbf{Average Forgetting (\%)} & \textbf{Total time (s)} \\
 \midrule
 CoP2L (MLP) & $89.40 \pm 1.00$ & $8.71 \pm 1.04$ & $4970.86 \pm 1452.34$ \\
 Replay (MLP) & $88.75 \pm 0.67$ & $9.53 \pm 0.72$ & $3894.62 \pm 75.11$ \\
 LFL (MLP) & $53.09 \pm 7.06$ & $44.18 \pm 7.93$ & $387.61 \pm 160.54$ \\
 EWC (MLP) & $55.59 \pm 5.57$ & $36.58 \pm 6.27$ & $1133.19 \pm 18.10$ \\
 SI (MLP) & $59.90 \pm 6.48$ & $41.38 \pm 7.18$ & $1043.86 \pm 7.28$\\
 \bottomrule
\end{tabular}
\end{table}

\clearpage
\subsection{Bound Plots}
\label{app:allboundplots}

\subsubsection{Bound plots for different buffer sizes}
This section presents the computed bound values over task risks on MNIST, Fashion-MNIST and EMNIST datasets (letters) in Class-Incremental settings. \cref{fig:bound} shows results for different replay buffer sizes (1000, 2000, 3000, 4000, 5000) obtained using the MLP architecture. \cref{fig:bound-mnist-cnn} shows the results obtained with the CNN architecture. We report the estimated deviation of the bound values over different seeds using the shaded regions which show the $\pm$1 standard deviation, and also the whiskers (plotted using the seaborn toolkit's boxplot function) which show the spread of the values excluding the outliers. This plotting style for showing the deviation across seeds applies to all bound plots presented in this paper. 

\begin{figure}[ht]
    \begin{center}
        \includegraphics[width=.325\textwidth, trim=0cm 0cm 0cm 0cm,clip]{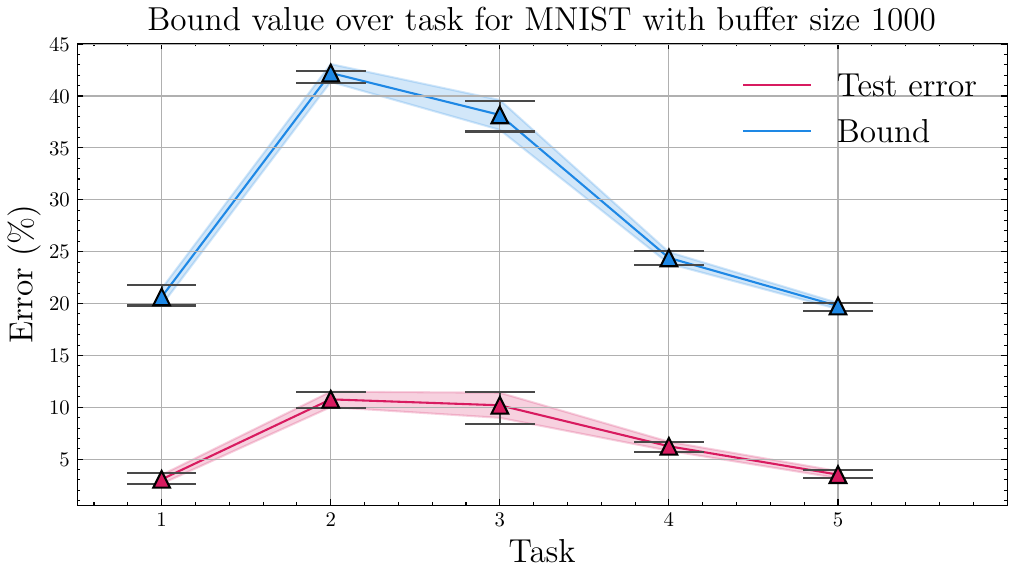}
        \includegraphics[width=.325\textwidth, trim=0cm 0cm 0cm 0cm,clip]{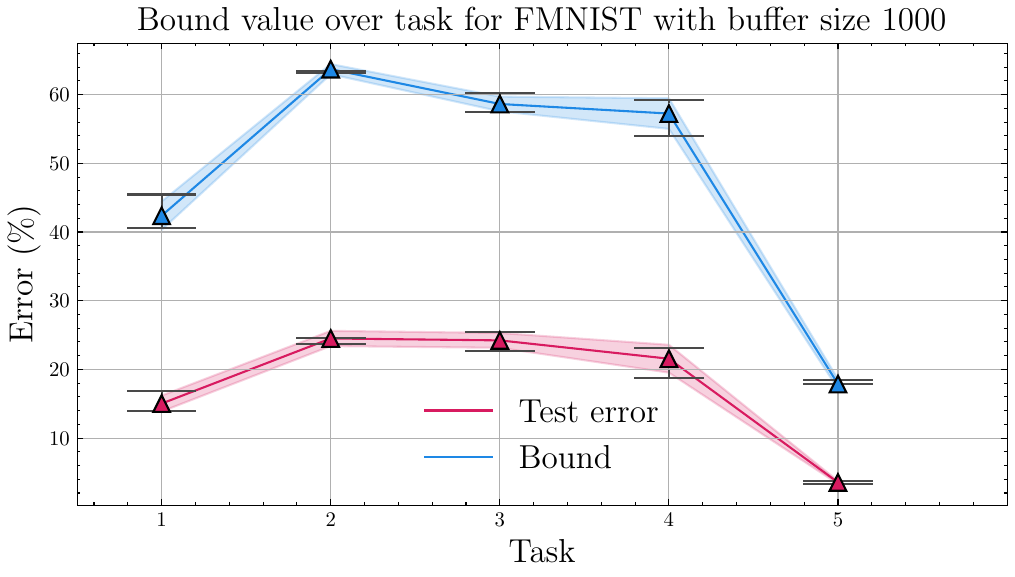}
        \includegraphics[width=.325\textwidth, trim=0cm 0cm 0cm 0cm,clip]{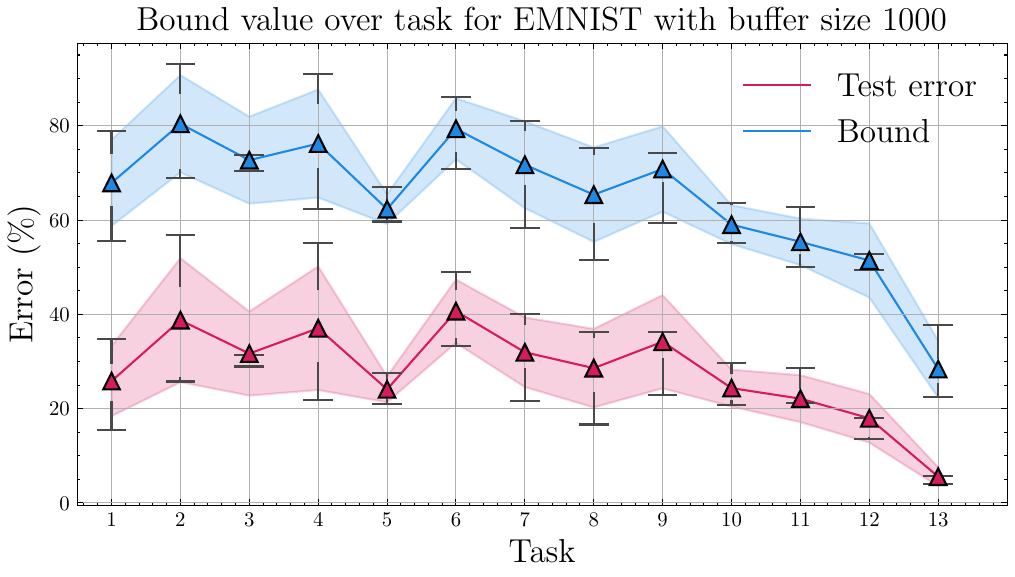}
        \includegraphics[width=.325\textwidth, trim=0cm 0cm 0cm 0cm,clip]{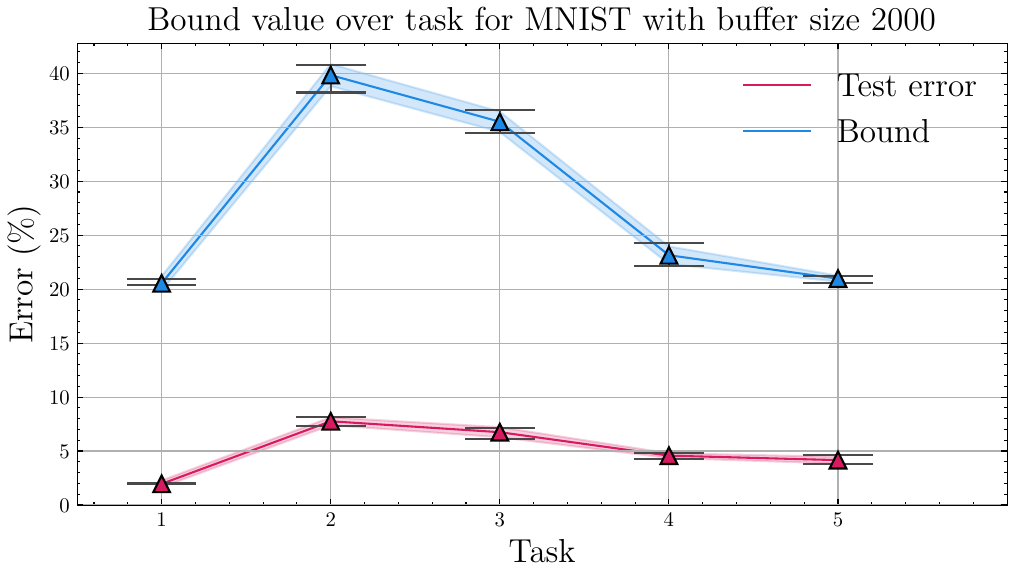}
        \includegraphics[width=.325\textwidth, trim=0cm 0cm 0cm 0cm,clip]{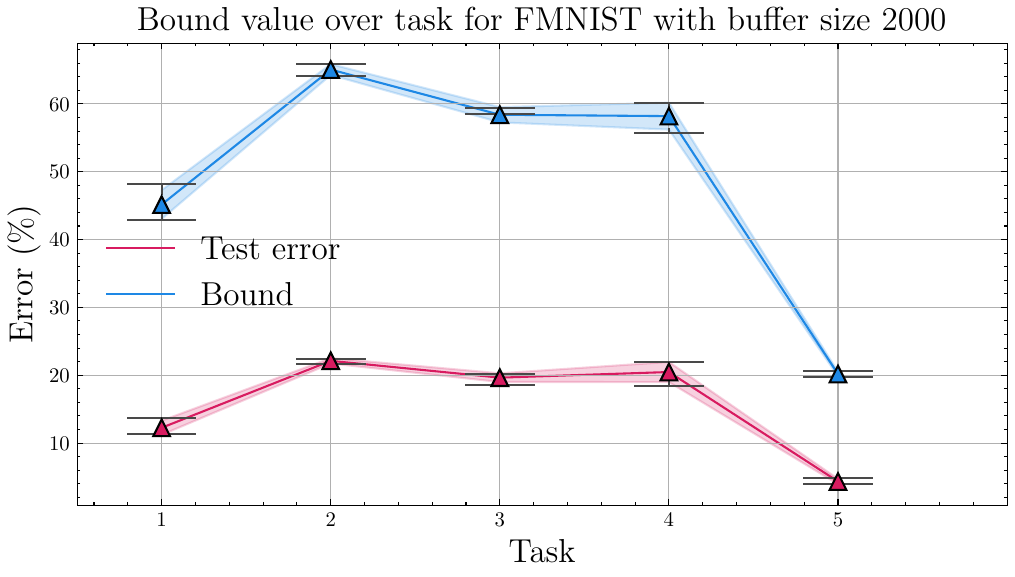}
        \includegraphics[width=.325\textwidth, trim=0cm 0cm 0cm 0cm,clip]{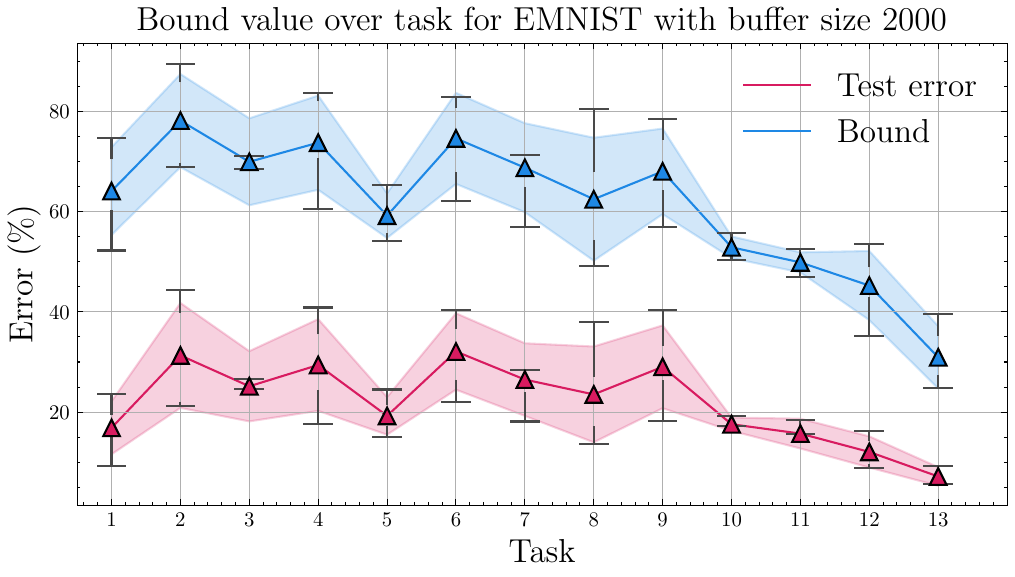}
        \includegraphics[width=.325\textwidth, trim=0cm 0cm 0cm 0cm,clip]{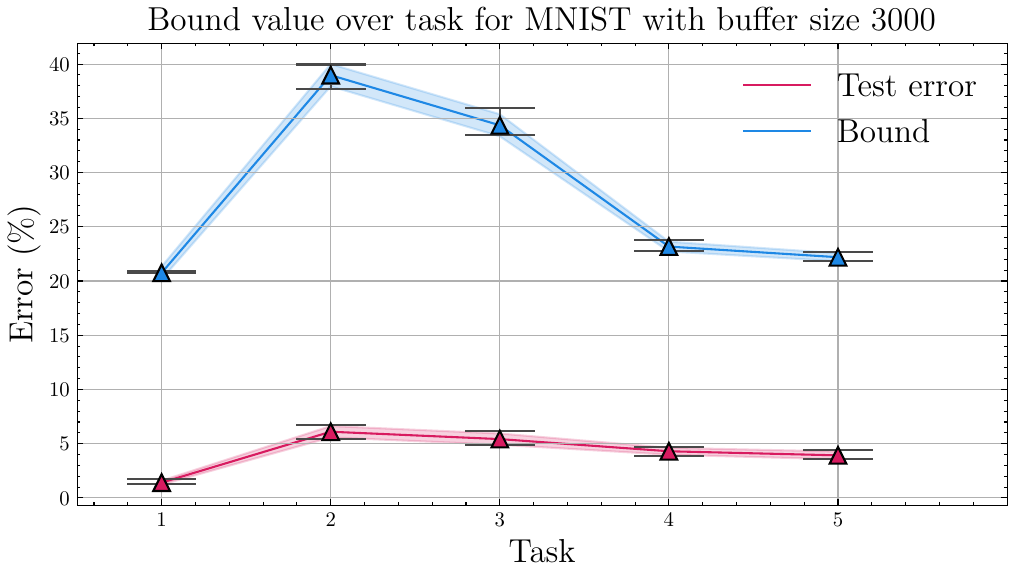}
        \includegraphics[width=.325\textwidth, trim=0cm 0cm 0cm 0cm,clip]{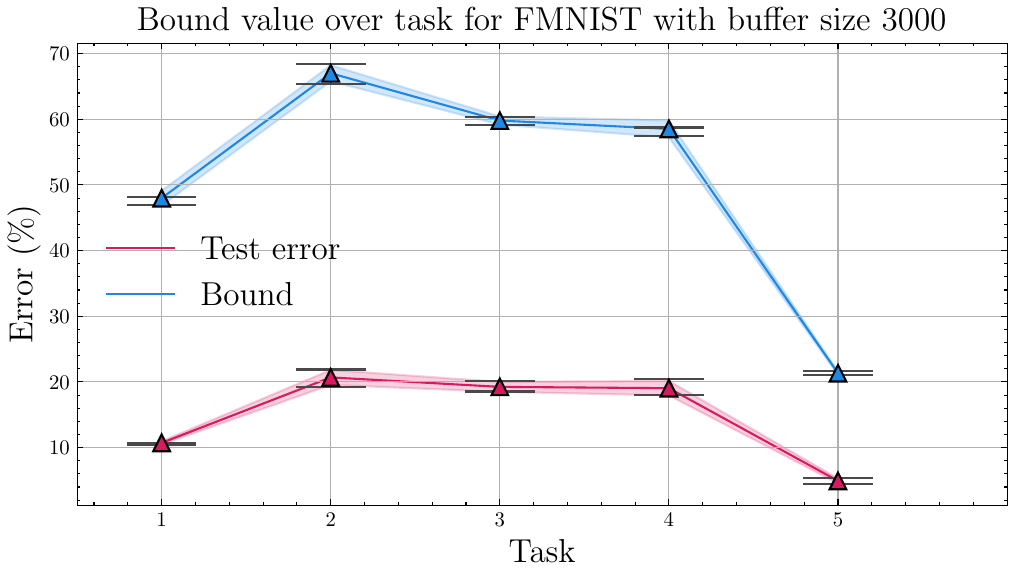}
        \includegraphics[width=.325\textwidth, trim=0cm 0cm 0cm 0cm,clip]{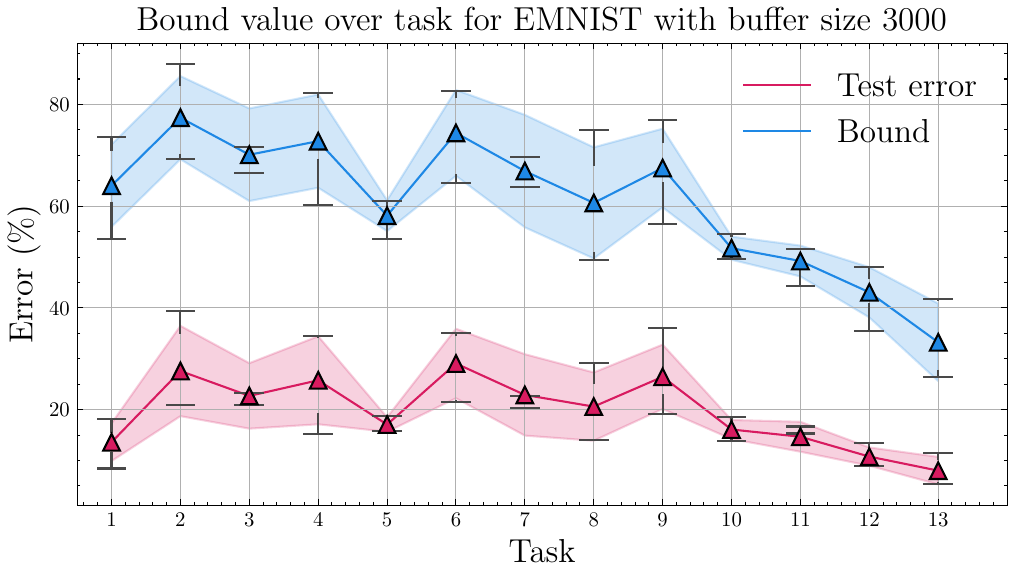}
        \includegraphics[width=.325\textwidth, trim=0cm 0cm 0cm 0cm,clip]{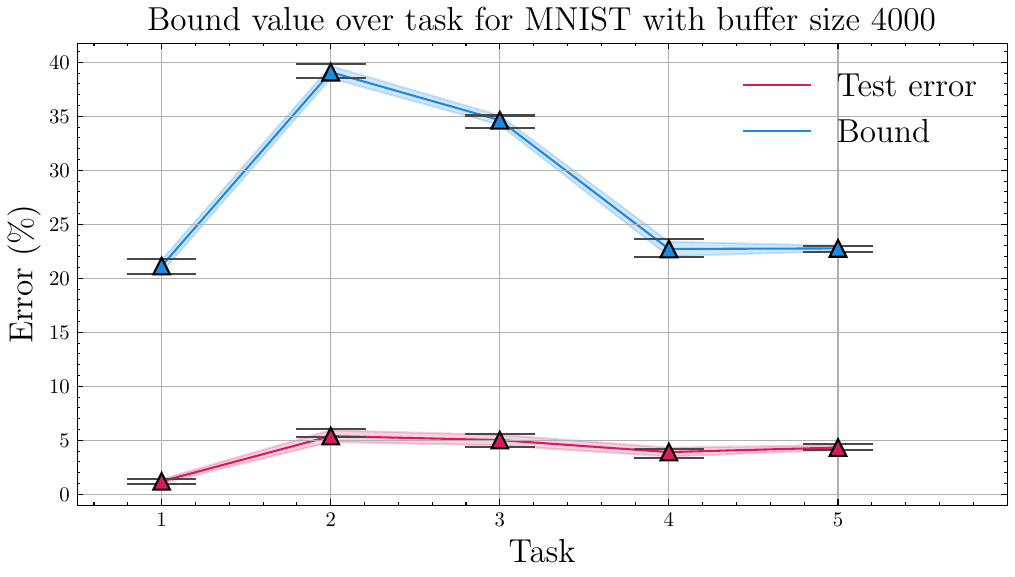}
        \includegraphics[width=.325\textwidth, trim=0cm 0cm 0cm 0cm,clip]{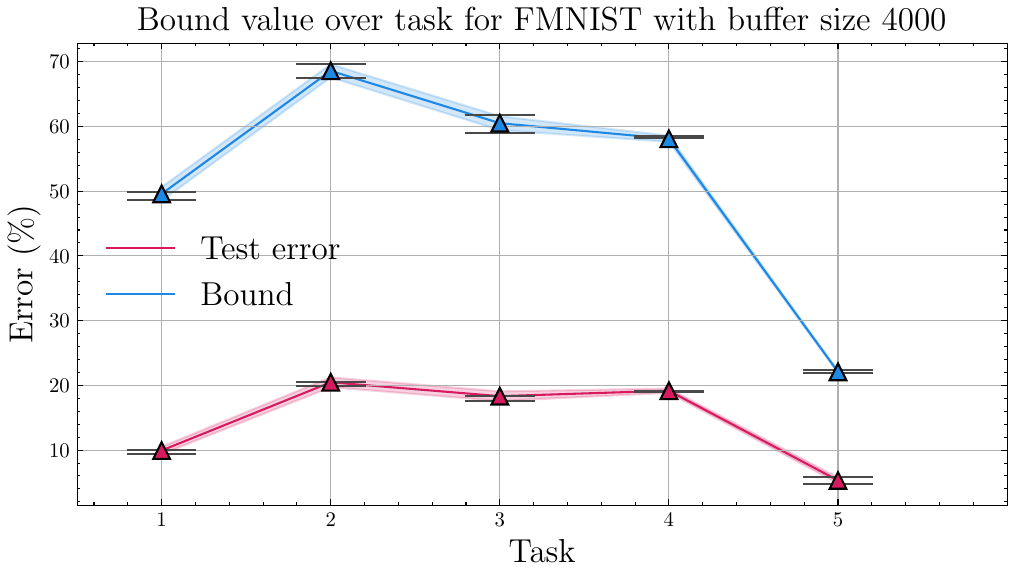}
        \includegraphics[width=.325\textwidth, trim=0cm 0cm 0cm 0cm,clip]{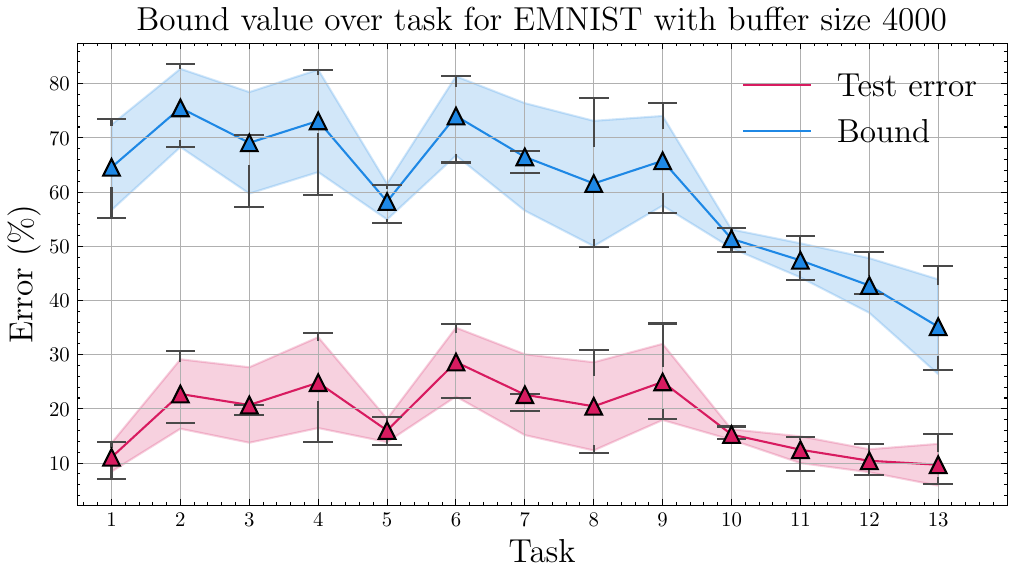}
        \includegraphics[width=.325\textwidth, trim=0cm 0cm 0cm 0cm,clip]{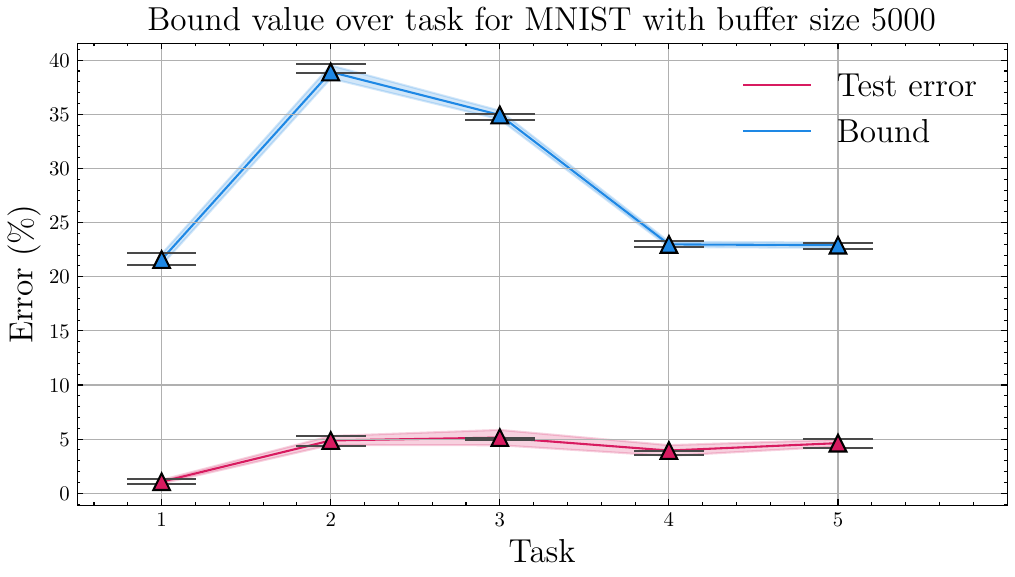}
        \includegraphics[width=.325\textwidth, trim=0cm 0cm 0cm 0cm,clip]{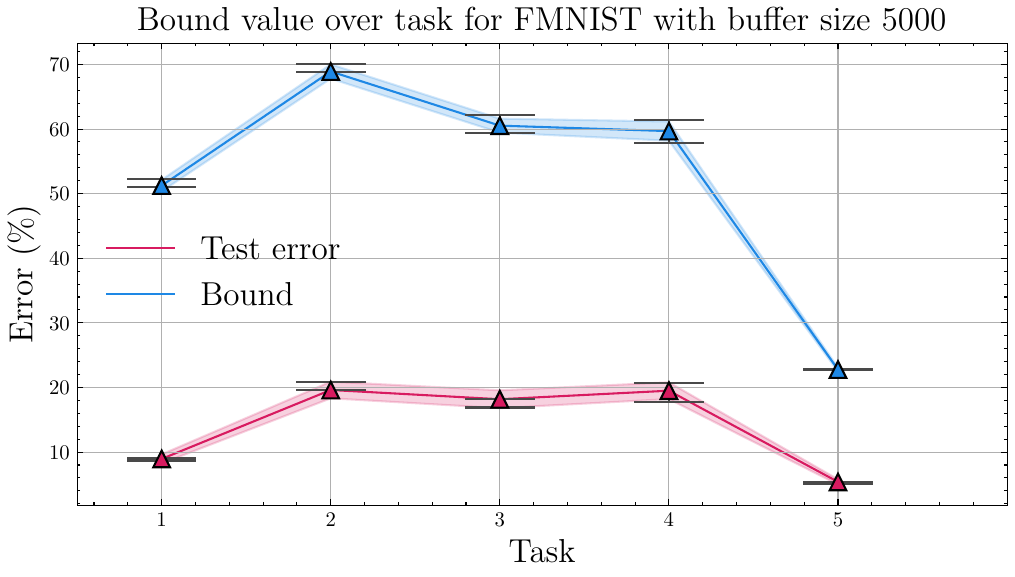}
        \includegraphics[width=.325\textwidth, trim=0cm 0cm 0cm 0cm,clip]{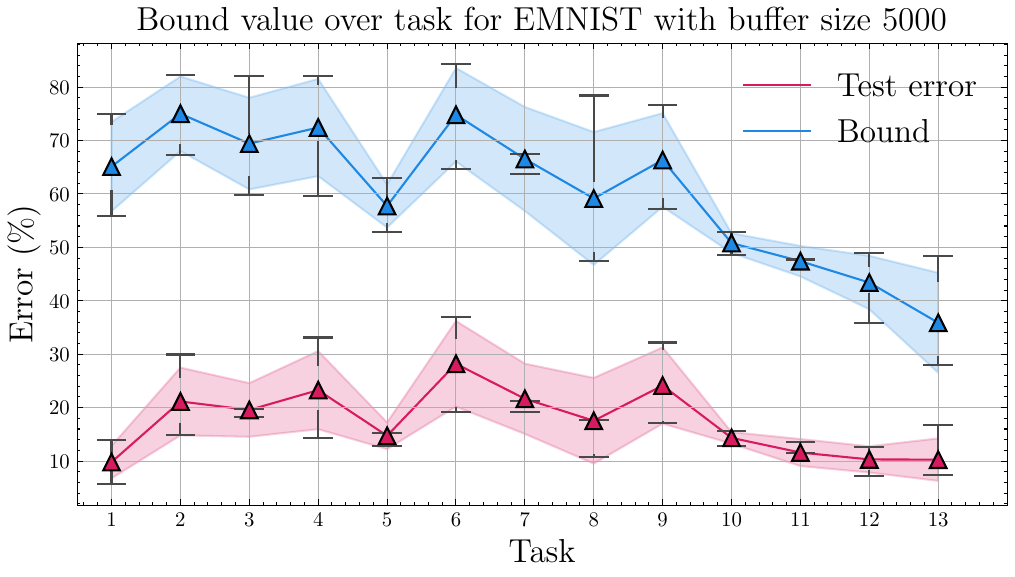}
    \end{center}
    \caption{\ouralgo{} Bounds with respect to tasks on the MNIST, Fashion-MNIST and EMNIST datasets using an MLP architecture}
    \label{fig:bound}
\end{figure}

\begin{figure}[h]
    \begin{center}
        \includegraphics[width=.325\textwidth, trim=0cm 0cm 0cm 0cm,clip]{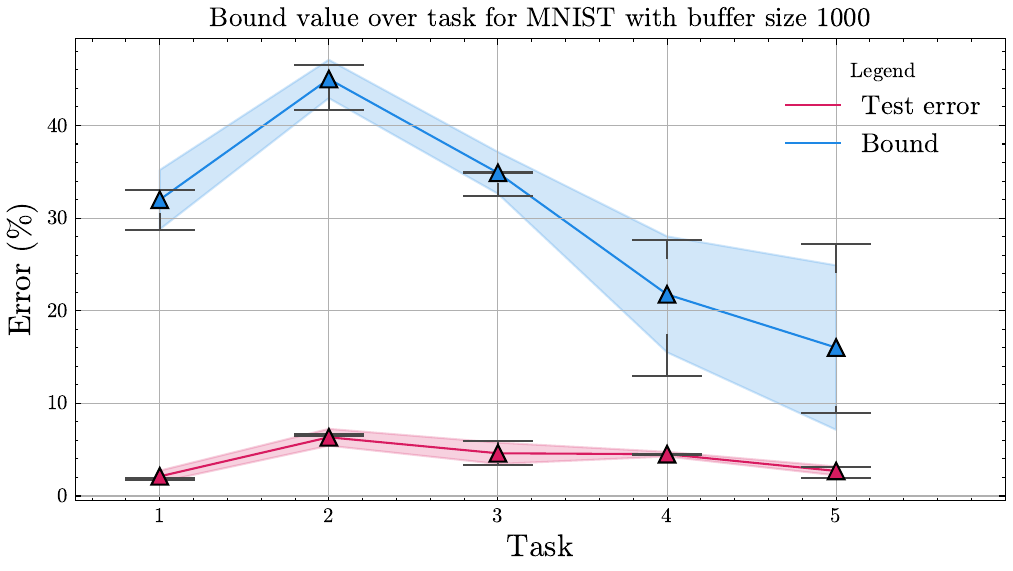}
        \includegraphics[width=.325\textwidth, trim=0cm 0cm 0cm 0cm,clip]{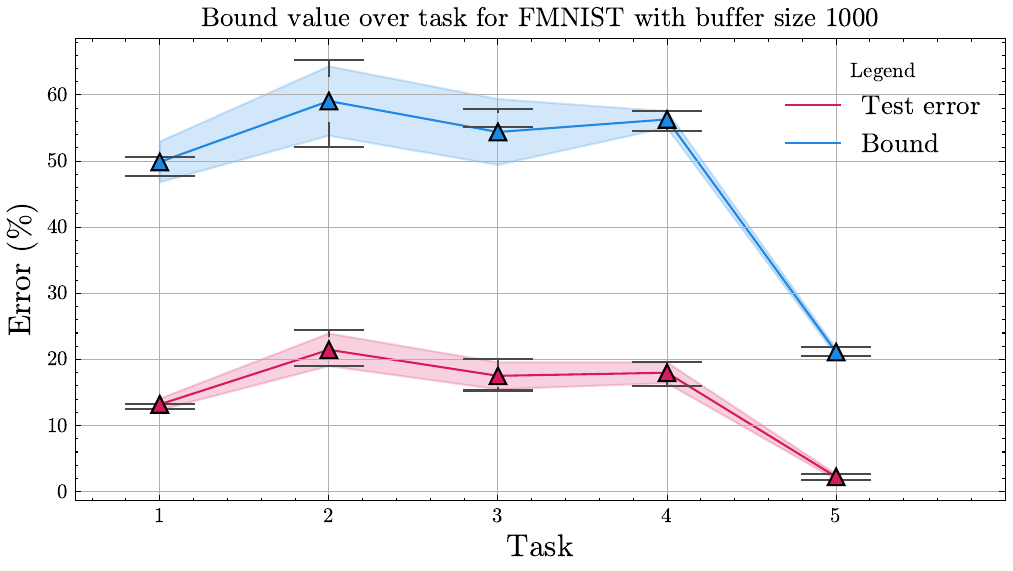}
        \includegraphics[width=.325\textwidth, trim=0cm 0cm 0cm 0cm,clip]{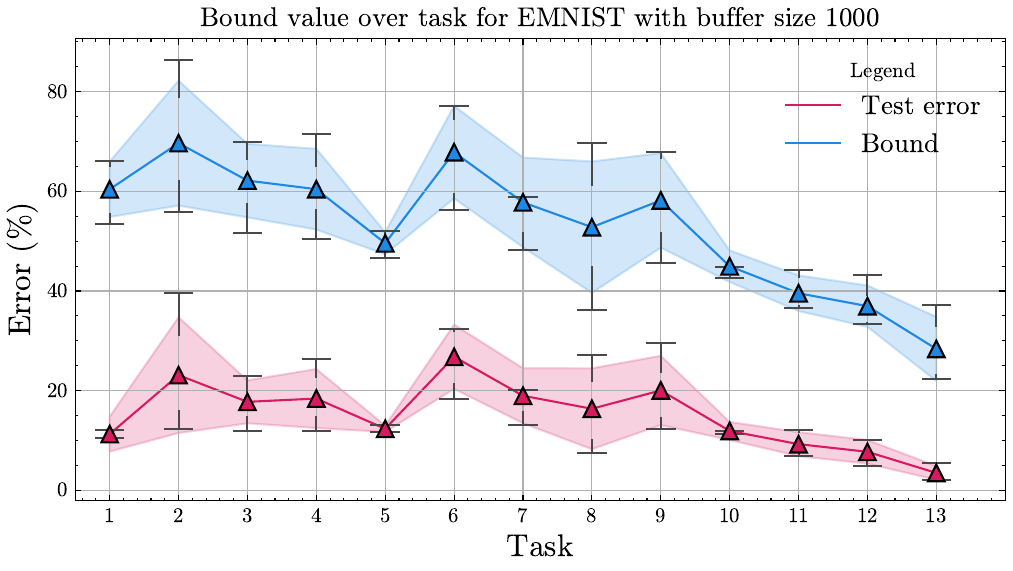}
        \includegraphics[width=.325\textwidth, trim=0cm 0cm 0cm 0cm,clip]{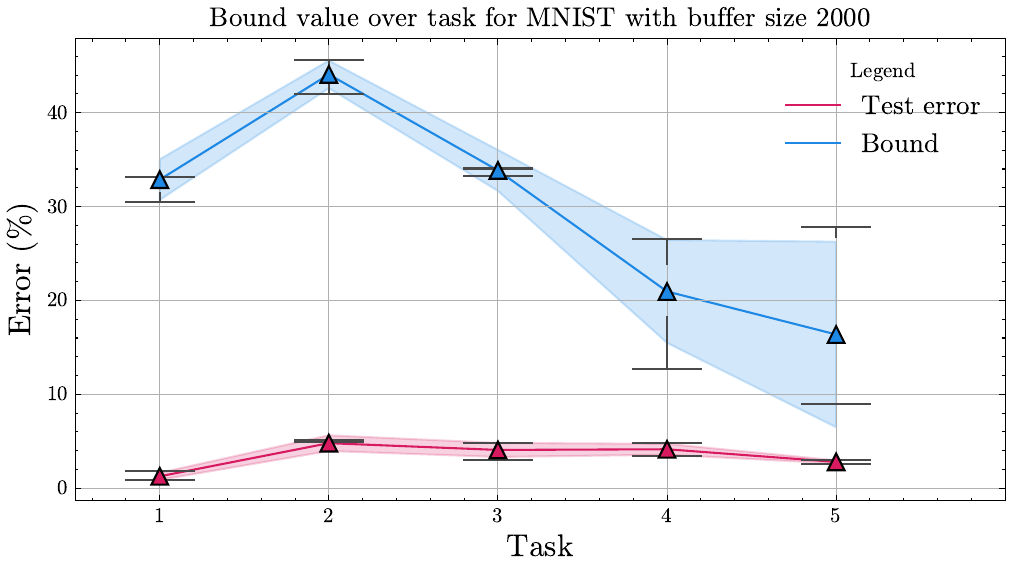}
        \includegraphics[width=.325\textwidth, trim=0cm 0cm 0cm 0cm,clip]{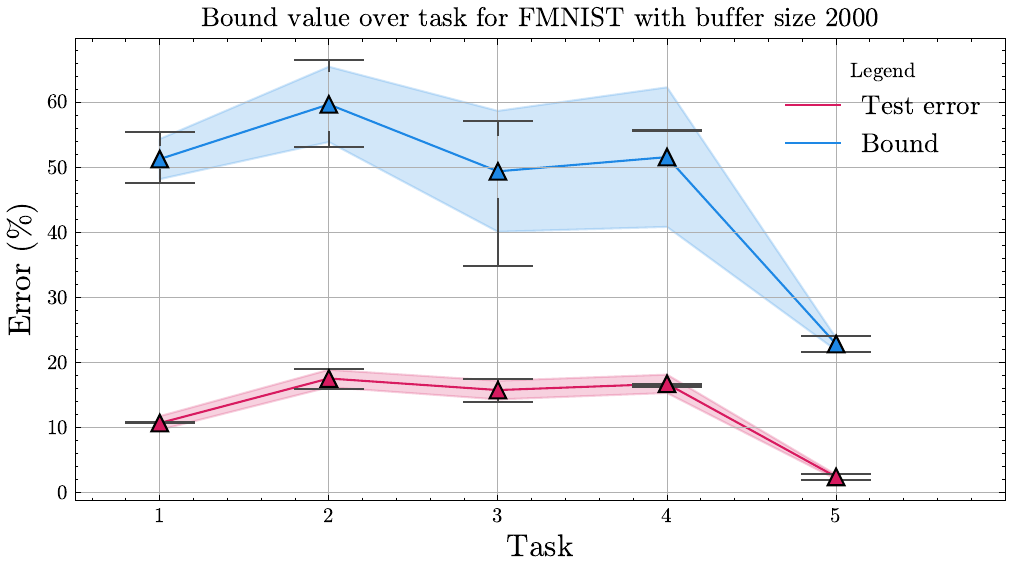}
        \includegraphics[width=.325\textwidth, trim=0cm 0cm 0cm 0cm,clip]{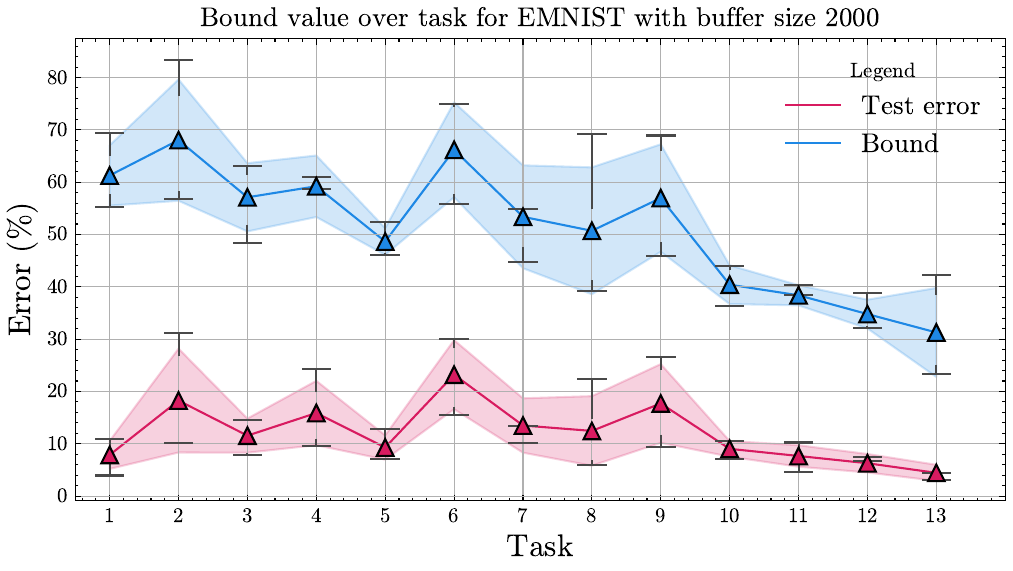}
        \includegraphics[width=.325\textwidth, trim=0cm 0cm 0cm 0cm,clip]{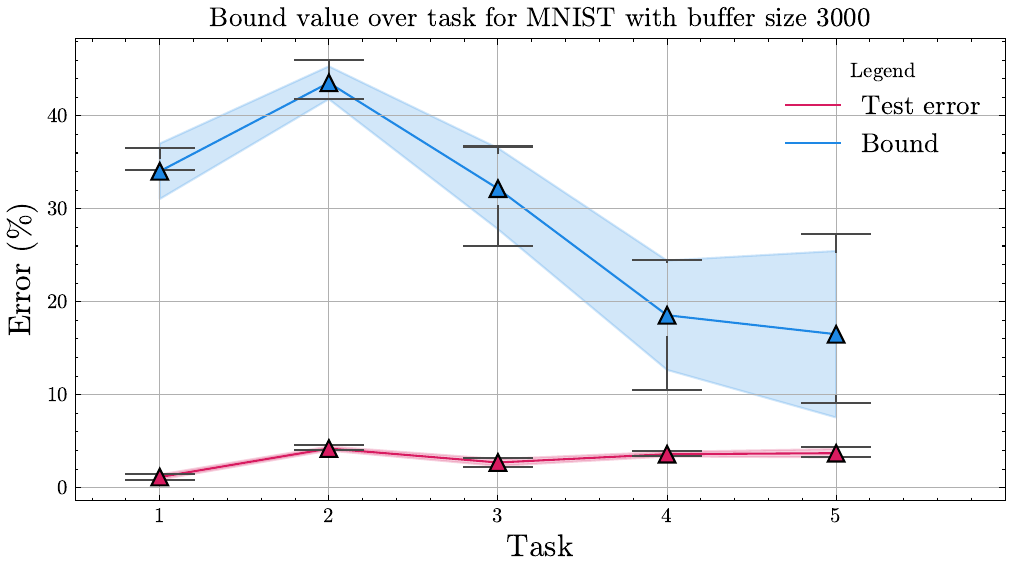}
        \includegraphics[width=.325\textwidth, trim=0cm 0cm 0cm 0cm,clip]{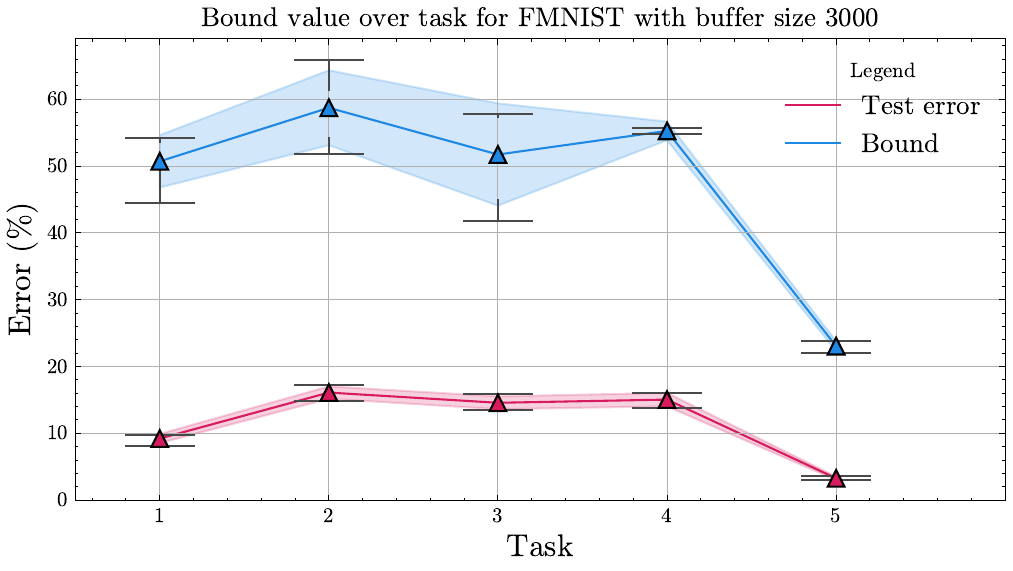}
        \includegraphics[width=.325\textwidth, trim=0cm 0cm 0cm 0cm,clip]{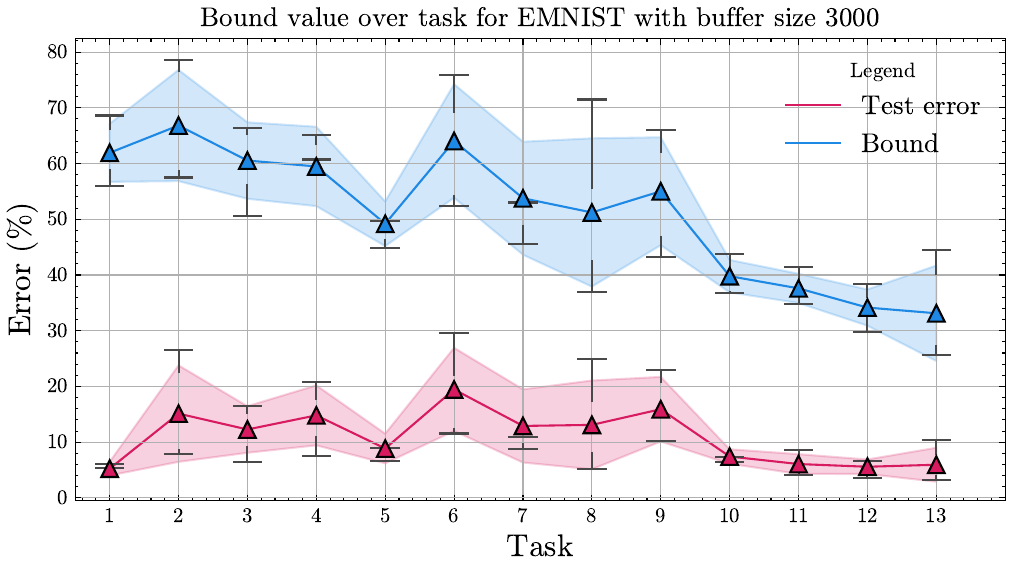}
        \includegraphics[width=.325\textwidth, trim=0cm 0cm 0cm 0cm,clip]{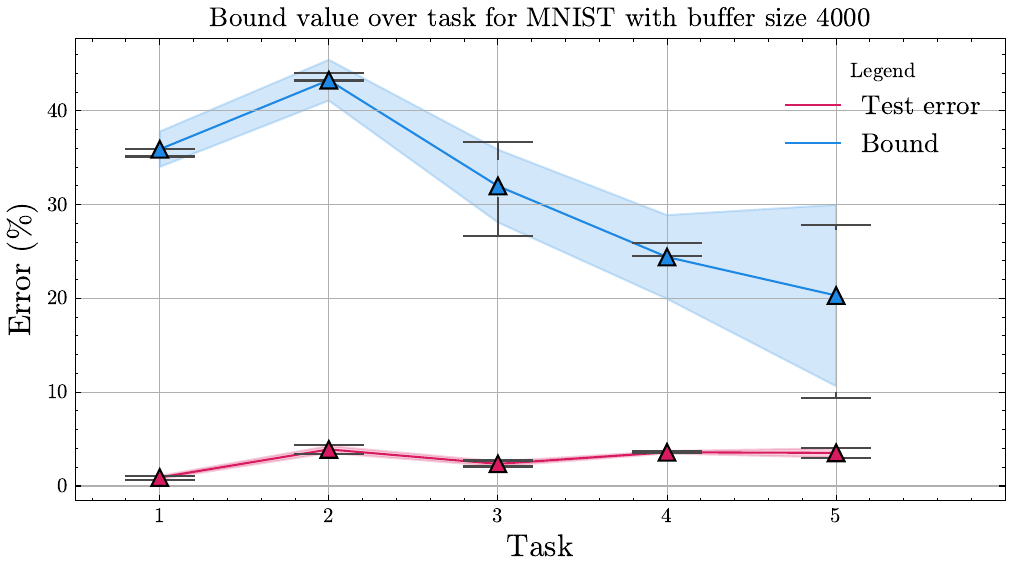}
        \includegraphics[width=.325\textwidth, trim=0cm 0cm 0cm 0cm,clip]{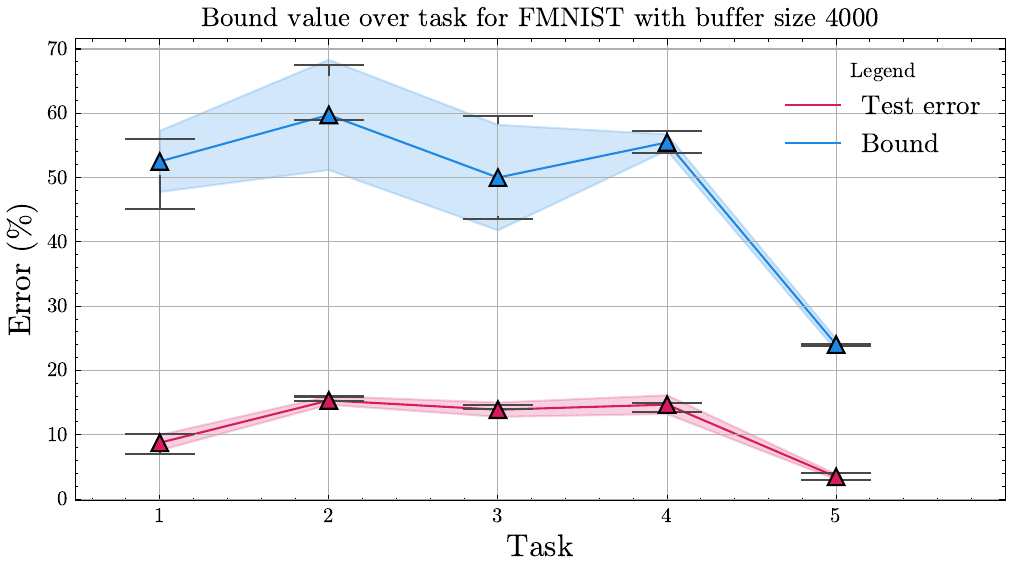}
        \includegraphics[width=.325\textwidth, trim=0cm 0cm 0cm 0cm,clip]{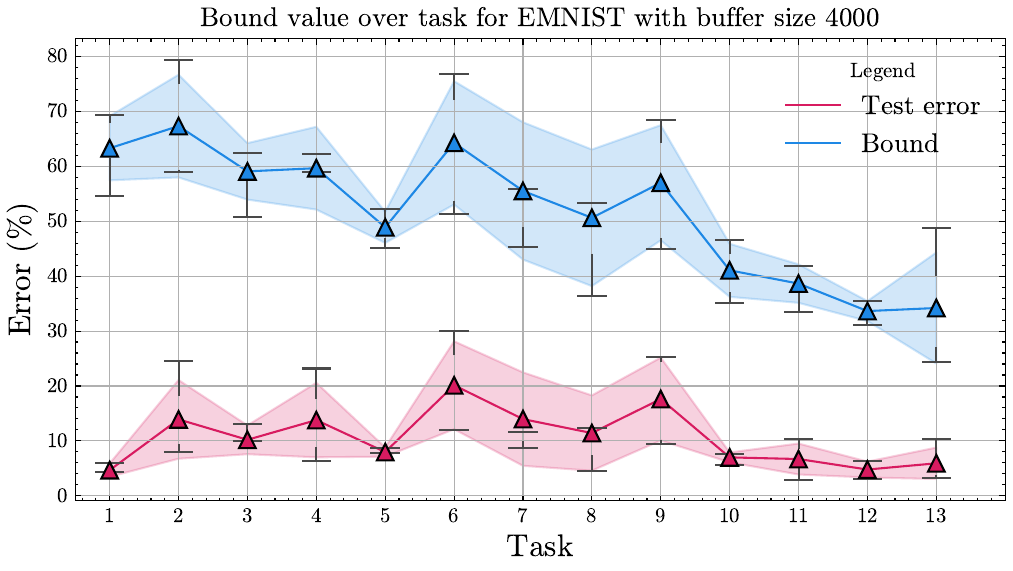}
        \includegraphics[width=.325\textwidth, trim=0cm 0cm 0cm 0cm,clip]{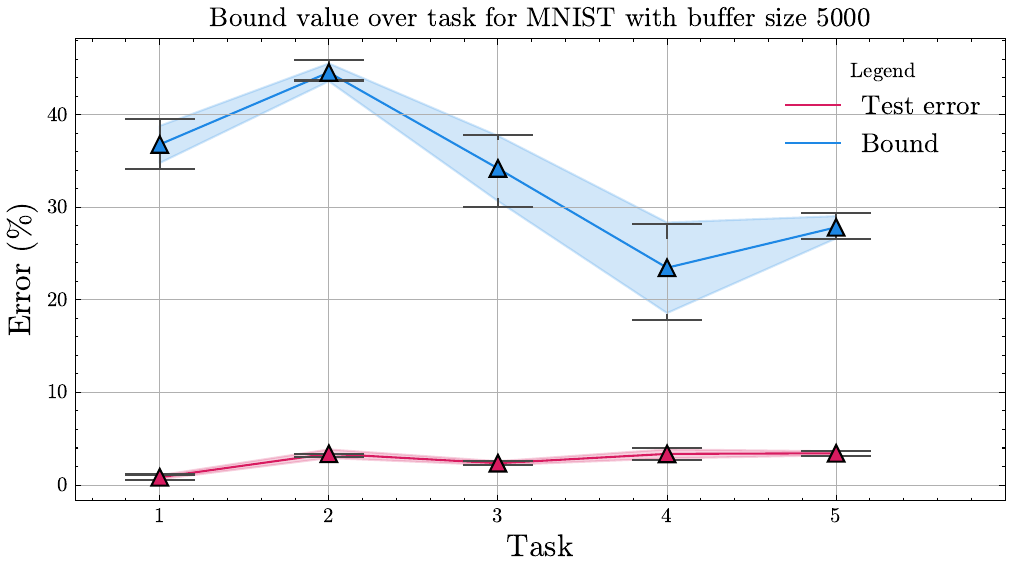}
        \includegraphics[width=.325\textwidth, trim=0cm 0cm 0cm 0cm,clip]{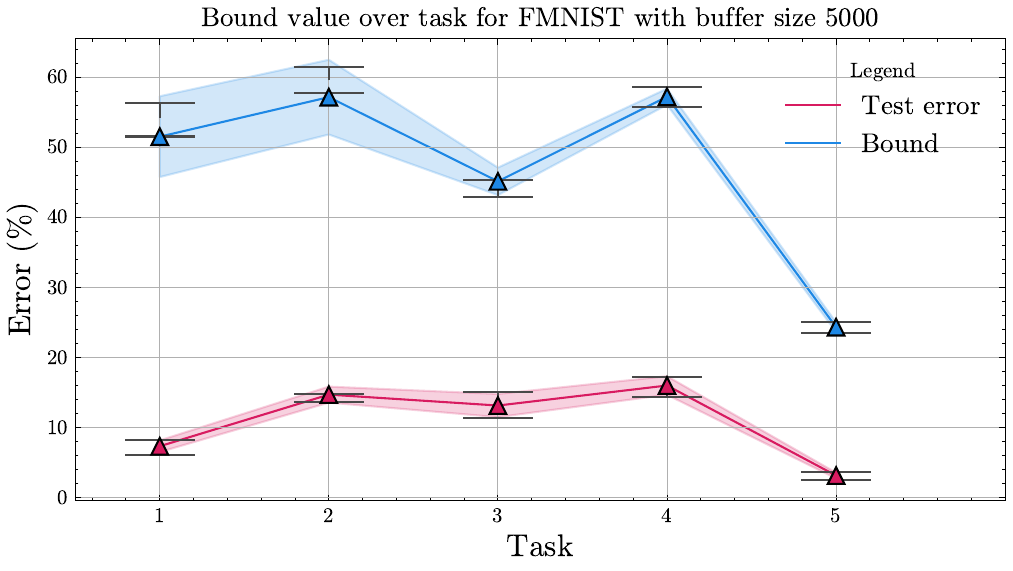}
        \includegraphics[width=.325\textwidth, trim=0cm 0cm 0cm 0cm,clip]{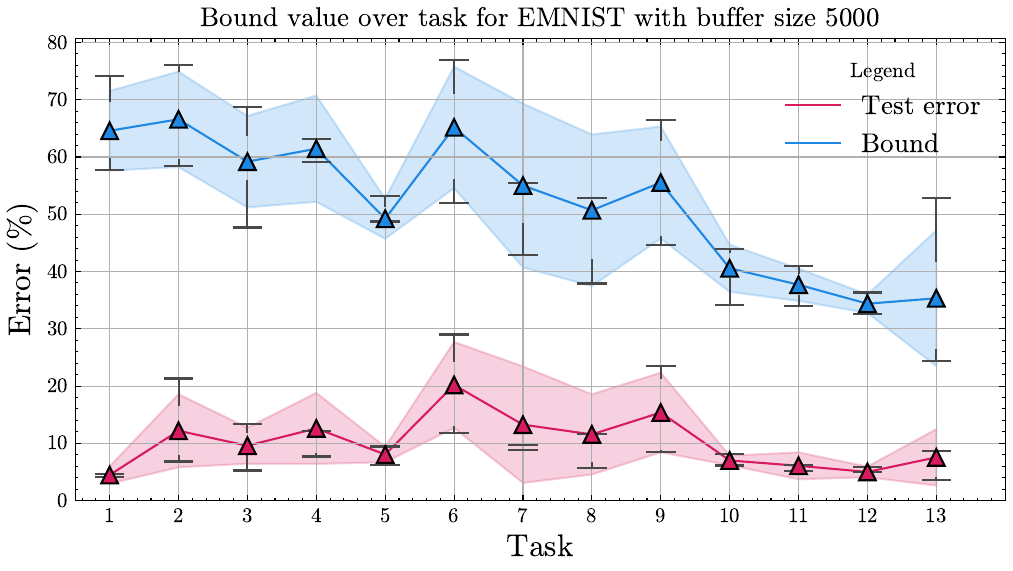}
    \end{center}
    \caption{\ouralgo{} Bounds with respect to tasks on the MNIST, Fashion-MNIST and EMNIST datasets using a CNN architecture}
    \label{fig:bound-mnist-cnn}
\end{figure}

{We observe from Figures \ref{fig:bound} and \ref{fig:bound-mnist-cnn} that the MLP architecture and the CNN architecture give non-trivial bounds that follow the trends exhibited by the test error curve.}

\clearpage

\subsubsection{Bound plots for CIFAR10, CIFAR100, TinyImageNet and ImageNet (subsets) datasets}

{In this section, we present the plots for the bound values obtained for the additional datasets, namely CIFAR10, CIFAR100, and ImageNet subsets. The ImageNet subsets include the Tinyimagenet \citep{tiny-imagenet}, ImageWoof \citep{Howard_Imagewoof_2019}, and ImageNette \citep{imagenette} datasets.} We provide the results on ImageWoof and ImageNette to showcase that the estimated bounds get tighter when more data is avaible per each task. 

We note that as the pretrained models used as backbones for our experiments were pretrained on the ImageNet, we provide the bounds on the ImageNet subsets simply to give insight to the reader. Indeed, the dependency of the backbones to the dataset over which we certify breaks the \emph{i.i.d.} assumption of the bound. 

\begin{figure}[H]
    \begin{center}
        \includegraphics[width=.48\textwidth, trim=0cm 0cm 0cm 0cm,clip]{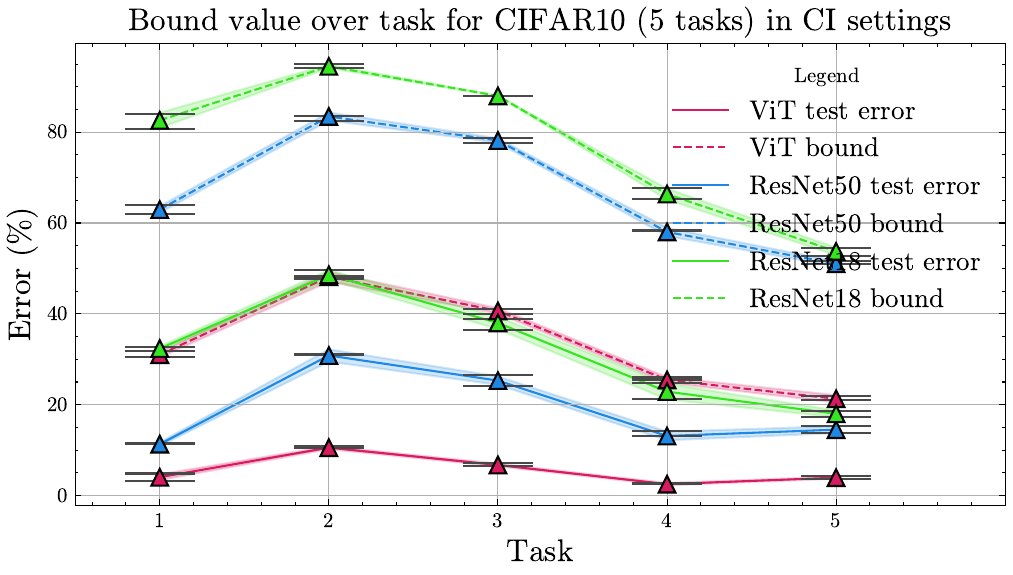}
        \includegraphics[width=.48\textwidth, trim=0cm 0cm 0cm 0cm,clip]{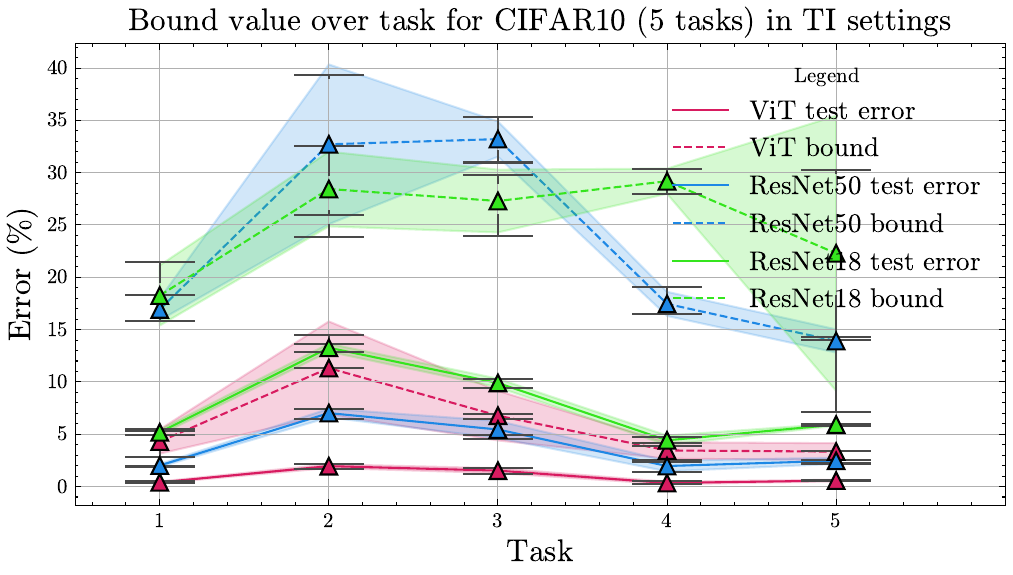}
    \end{center}
    \vspace{-.4cm}
    \caption{\ouralgo{} bounds with respect to tasks on the CIFAR10 (5 tasks) dataset in Class-Incremental (left) and Task-Incremental (right) settings}
\end{figure}

\vspace{-2mm}

\begin{figure}[h]
    \begin{center}
        \includegraphics[width=.49\textwidth, trim=0cm 0cm 0cm 0cm,clip]{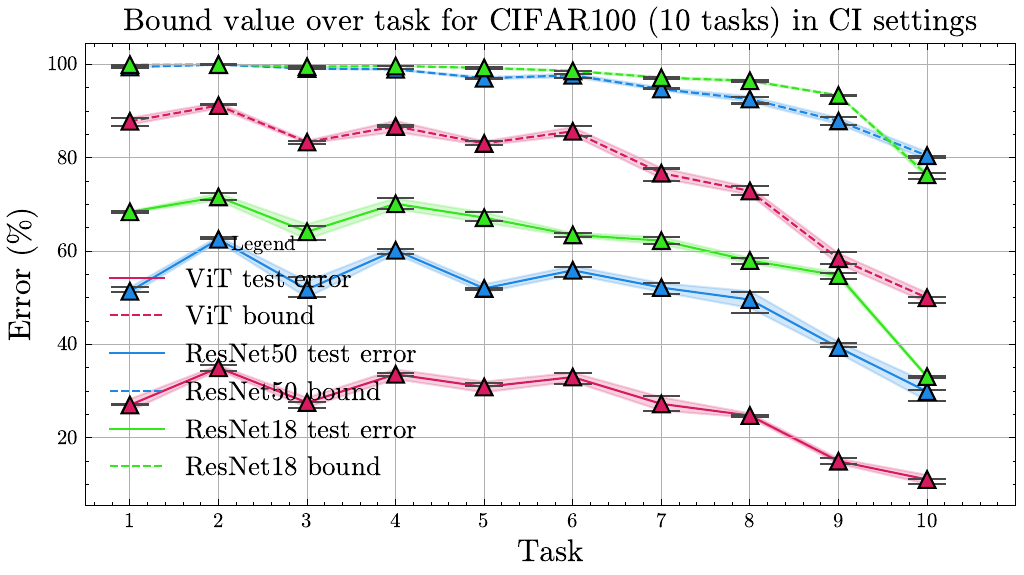}
        \includegraphics[width=.49\textwidth, trim=0cm 0cm 0cm 0cm,clip]{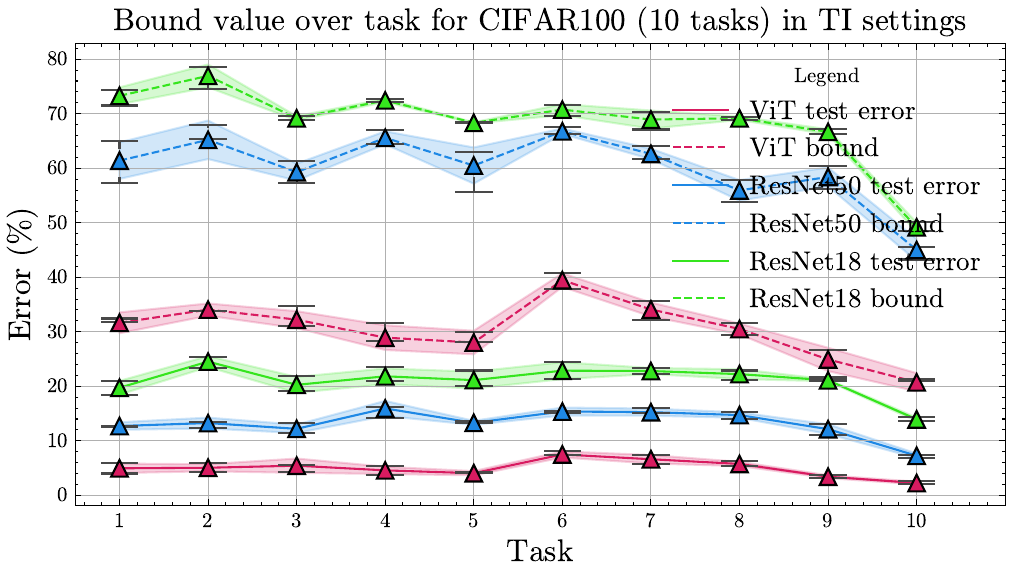}
    \end{center}
    \vspace{-.4cm}
    \caption{\ouralgo{} bounds with respect to tasks on the CIFAR100 (10 tasks) dataset in Class-Incremental (left) and Task-Incremental (right) settings}
\end{figure}
\vspace{-2mm}

\begin{figure}[h]
    \begin{center}
        \includegraphics[width=.49\textwidth, trim=0cm 0cm 0cm 0cm,clip]{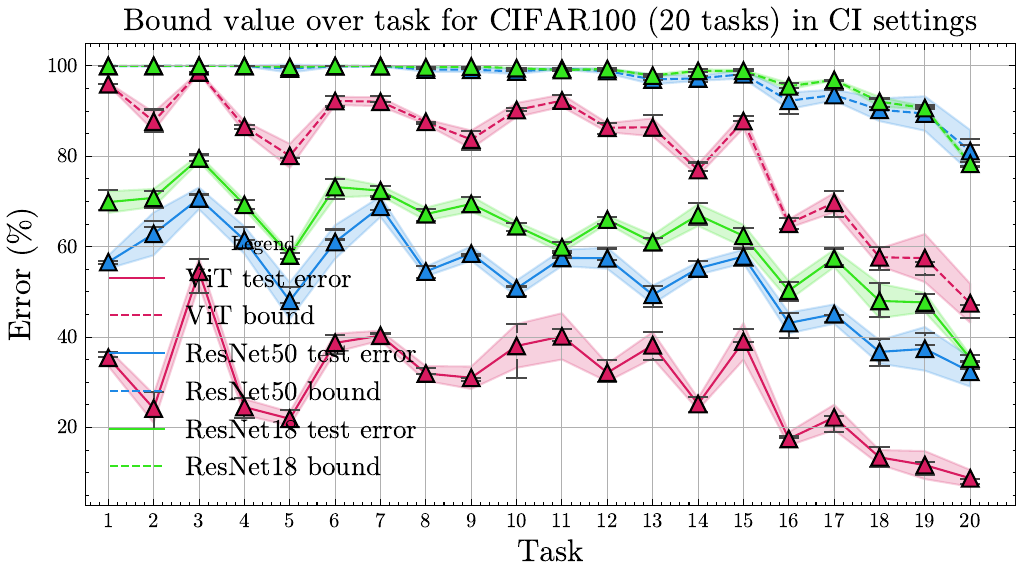}
        \includegraphics[width=.49\textwidth, trim=0cm 0cm 0cm 0cm,clip]{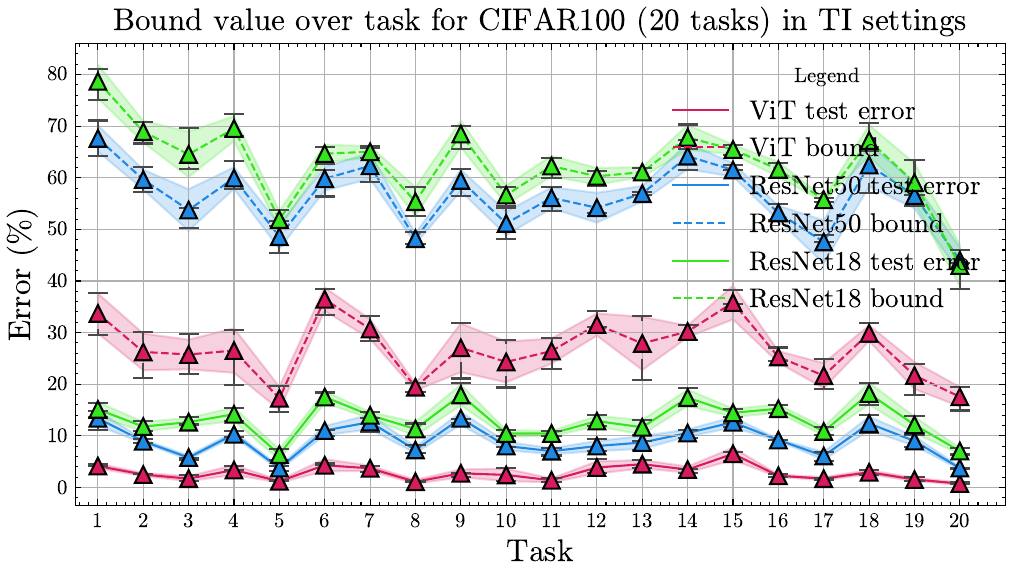}
    \end{center}
    \caption{\ouralgo{} bounds with respect to tasks on the CIFAR100 (20 tasks) dataset in Class-Incremental (left) and Task-Incremental (right) settings}
\end{figure}

\begin{figure}[h!]
    \begin{center}
        \includegraphics[width=.49\textwidth, trim=0cm 0cm 0cm 0cm,clip]{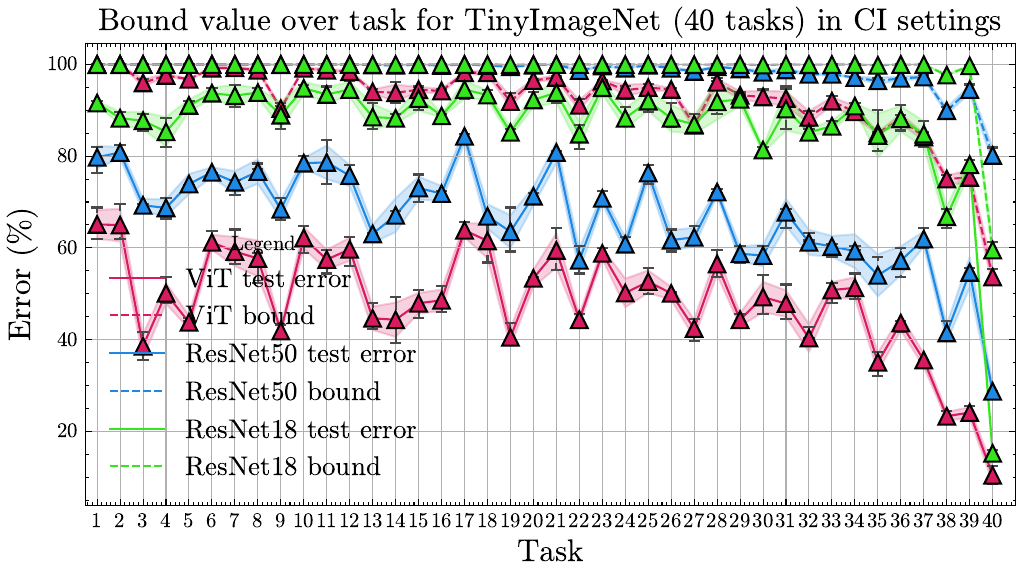}
        \includegraphics[width=.49\textwidth, trim=0cm 0cm 0cm 0cm,clip]{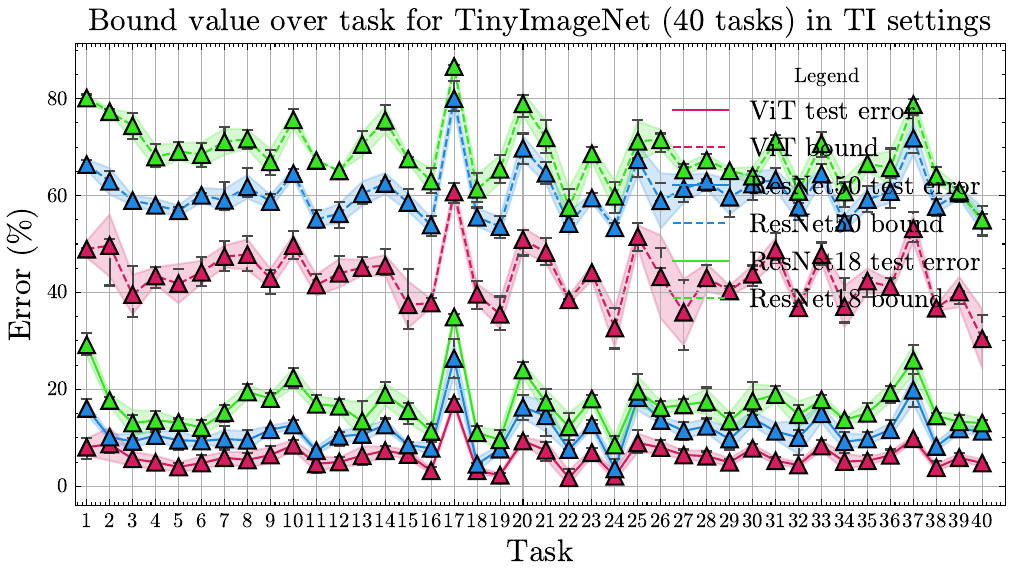}
    \end{center}
    \vspace{-.4cm}
    \caption{\ouralgo{} bounds with respect to tasks on the TinyImageNet (40 tasks) dataset in Class-Incremental (left) and Task-Incremental (right) settings}
\end{figure}

\vspace{-2mm}

\begin{figure}[h]
    \begin{center}
        \includegraphics[width=.49\textwidth, trim=0cm 0cm 0cm 0cm,clip]{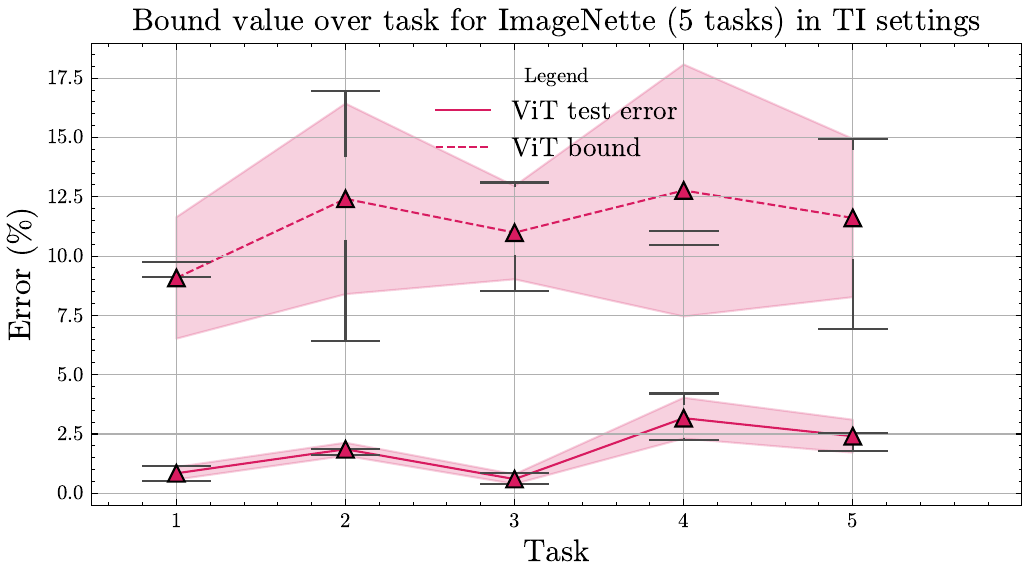}
        \includegraphics[width=.49\textwidth, trim=0cm 0cm 0cm 0cm,clip]{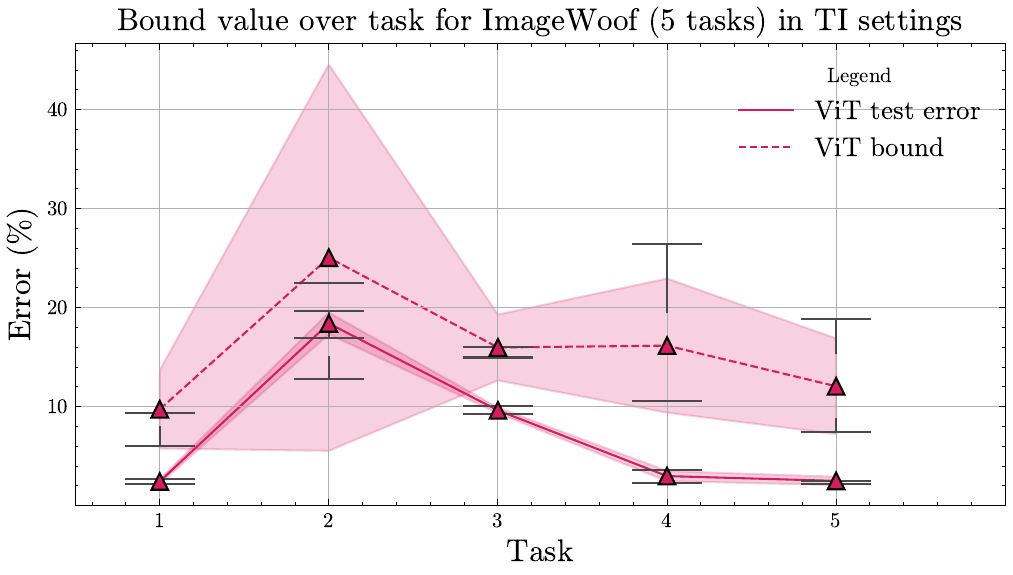}
    \end{center}
    \vspace{-.4cm}
    \caption{\ouralgo{} bounds with respect to tasks on the ImageNette (5 tasks) dataset (left) and ImageWoof dataset (right) in Task-Incremental setting}
\end{figure}

\begin{figure}[h]
    \begin{center}
        \includegraphics[width=.49\textwidth, trim=0cm 0cm 0cm 0cm,clip]{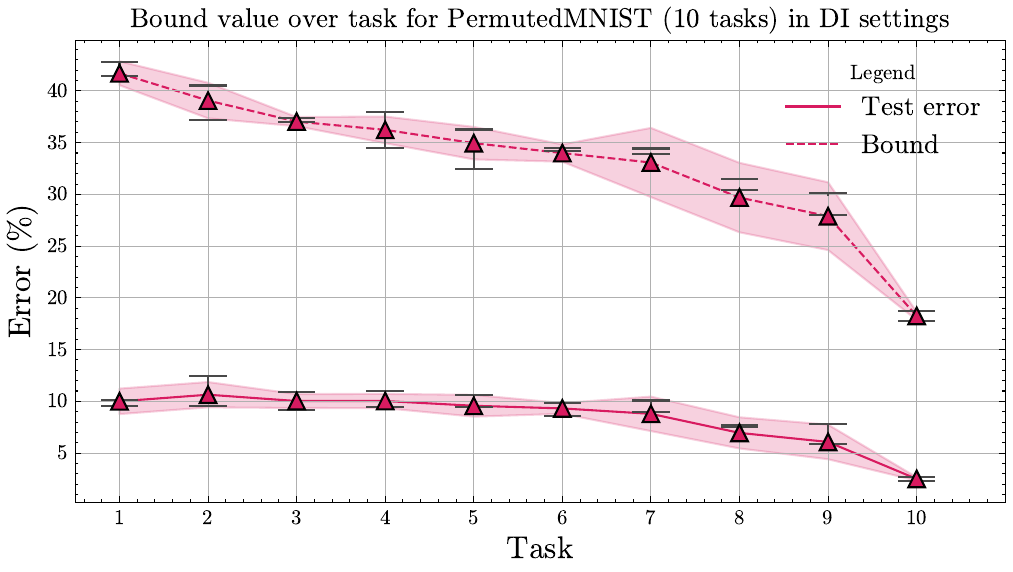}
        \includegraphics[width=.49\textwidth, trim=0cm 0cm 0cm 0cm,clip]{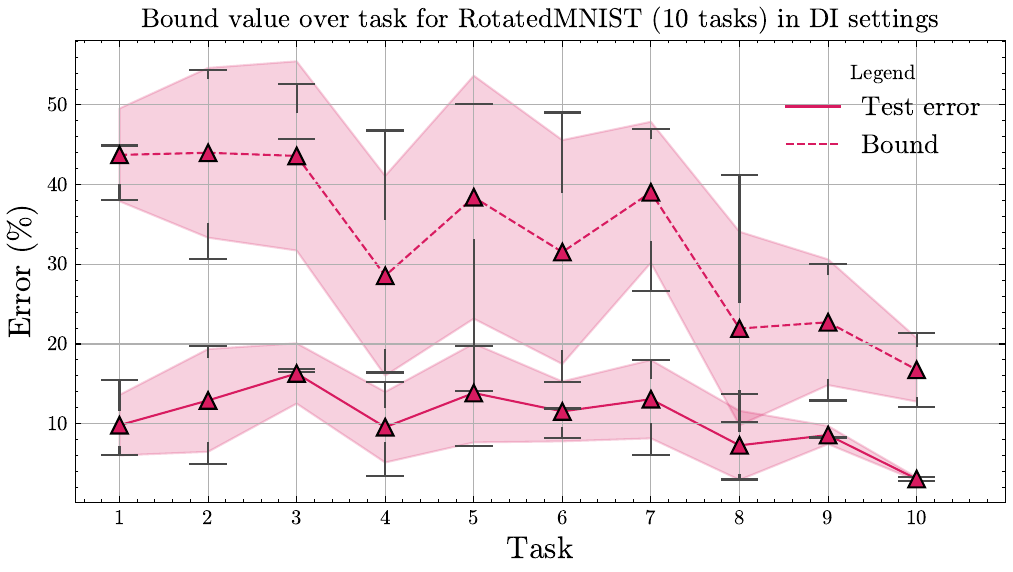}
    \end{center}
    \vspace{-.4cm}
    \caption{\ouralgo{} bounds with respect to tasks on the PermutedMNIST (10 tasks) dataset (left) and on the RotatedMNIST dataset in Domain-Incremental setting}
\end{figure}

\clearpage
\newpage

\section{Analysis of \ouralgo{}}
\label{app:ablation}

In this section we provide multiple analyses and ablation studies of \ouralgo{}.

\subsection{Ablation on \ouralgo{} loss terms.}
We have conducted an ablation study to understand the effect of additional weighting that counteracts class imbalance that arises in the continual learning setup when new tasks are introduced (Standard replay counteracts this by extending the batch size to have equal class representation). In the modified P2L (Algorithm \ref{alg:modified-p2l}), we incorporate an additional weighting term over classes to adapt P2L for class-balanced continual learning. In order to assess whether this idea is indeed helpful, we have run an experiment on the MNIST dataset with and without the additional weighting idea. In Figure \ref{fig:lossablation}, we show the results of this experiment. For both accuracy (left panel) and forgetting (right panel), we observe that the inclusion of the additional loss weighting term significantly improves the results. 

 \begin{figure}[h]
    \begin{center}
    \includegraphics[width=.49\columnwidth,trim=0cm 0cm 0cm 0cm,clip]{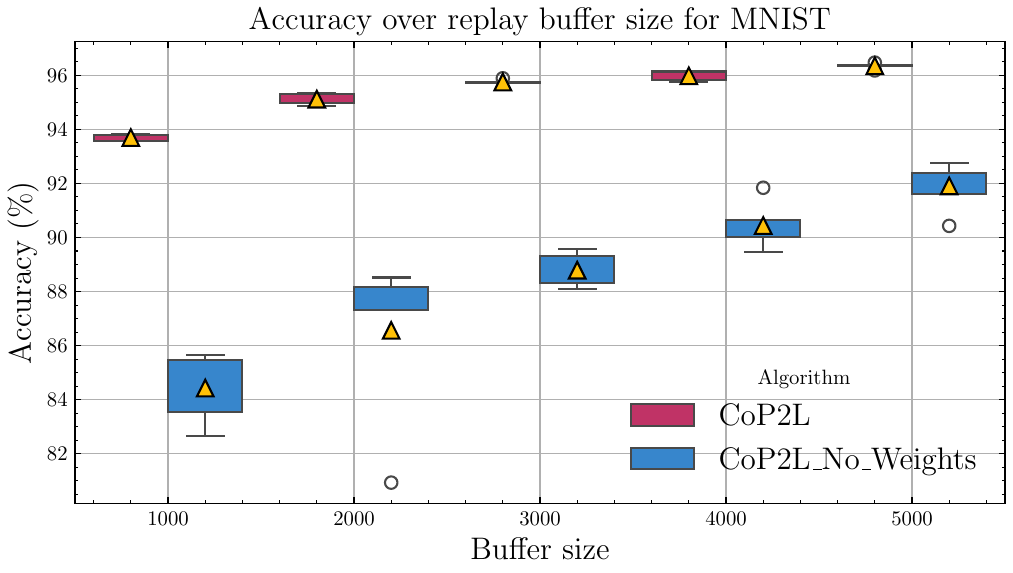}
    \includegraphics[width=.49\columnwidth,trim=0cm 0cm 0cm 0cm,clip]{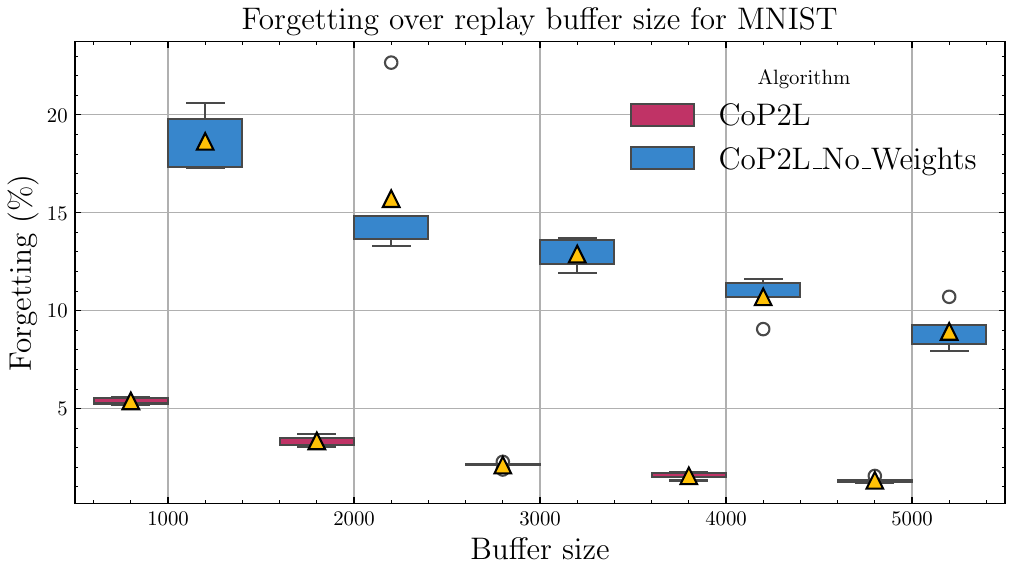}
    \vspace{-.4cm}
    \caption{Ablation study on  MNIST and Fashion-MNIST on the additional weighting to ensure class balance we propose in Algorithm \ref{alg:modified-p2l} on the \ouralgo{} loss function. Accuracy \textbf{(left)} and forgetting \textbf{(right)} over replay buffer size for the MNIST dataset}
    \vspace{-.3cm}
    \label{fig:lossablation}
    \end{center}
\end{figure}

\subsection{Sensitivity analysis with respect to the task ordering.}
In order to also assess the sensitivity of \ouralgo{} with respect to ordering of the tasks, on MNIST and Fashion-MNIST experiments we have conducted an experiment where we randomly shuffle the task order. For both datasets, we have tested our algorithm with 5 different random task ordering. 
For each task, we report the cumulative accuracy such that, $\text{Average Cumulative Acc.}_{T} = \frac{1}{O \cdot T} \sum_{o=1}^O \sum_{t=1}^T \text{Accuracy}_{T, t}^o$, 
where $\text{Accuracy}_{T, t}^o$ denotes the accuracy obtained at task $t$ after finishing training until task $T$, and the $o$ index denotes the task order. We have used $O=5$ different task orderings. We present the results of this experiment in Figure \ref{fig:permutation}. On both MNIST and Fashion-MNIST datasets \ouralgo{} is able to outperform with more tasks. We also observe that the variability of the results decrease more significantly for \ouralgo{} compared to replay. 

\begin{figure}[h]
    \begin{center}
    
    \includegraphics[width=.49\columnwidth,trim=0cm 0cm 0cm 0cm,clip]{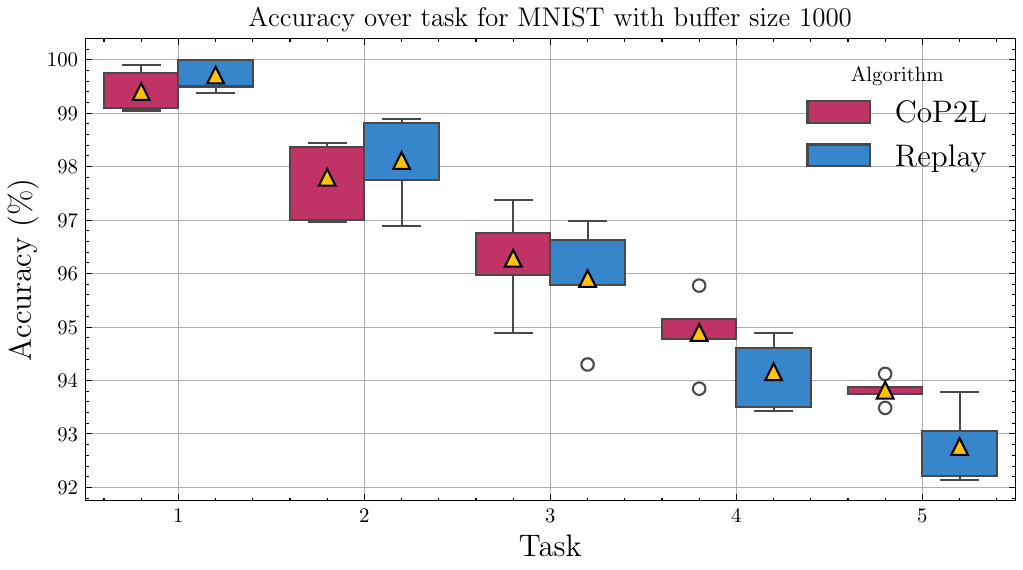}
    \includegraphics[width=.49\columnwidth,trim=0cm 0cm 0cm 0cm,clip]{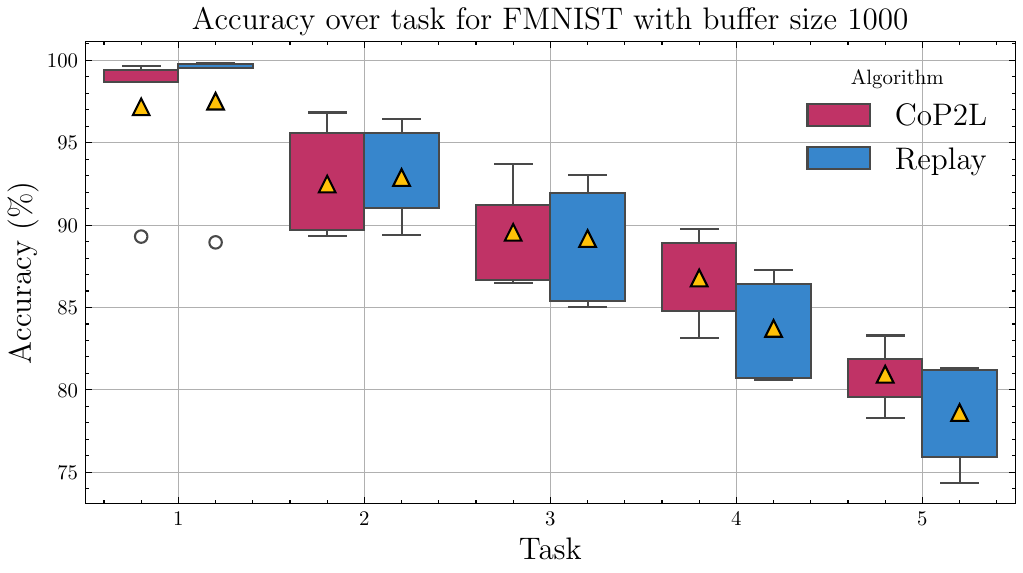}
    \vspace{-.4cm}
    \caption{Studying the variation of cumulative accuracy with respect to task ordering for \ouralgo{} and Replay on MNIST dataset \textbf{(left)} and Fashion-MNIST dataset \textbf{(right)}. We use 5 different task order, and show the variability with respect to each task.}
    \vspace{-0.6cm}
    \label{fig:permutation}
    \end{center}
\end{figure}

\subsection{Analysis with respect to buffer size}
\vspace{-2mm}

In Fig.~\ref{fig:experiments}, we compare the accuracy and forgetting obtained with replay buffer sizes of 1000--5000 samples for \ouralgo{} and replay. This corresponds to roughly 8--40\% of the per-task dataset size for MNIST and Fashion-MNIST, and 10--52\% for EMNIST. We report the final average accuracy (after $T{=}5$ tasks for MNIST and Fashion-MNIST, and $T{=}13$ for EMNIST) across seeds as well as per-task test accuracy.

For MNIST, \ouralgo{} clearly outperforms replay on average.  
For Fashion-MNIST and EMNIST, \ouralgo{} shows a clear advantage for small buffers and remains comparable or slightly better for larger buffers.  
Across all datasets, \ouralgo{} substantially reduces forgetting.

\begin{figure}[H]
\centering
\setlength{\abovecaptionskip}{2pt}
\setlength{\belowcaptionskip}{-6pt}

\includegraphics[width=.49\textwidth]{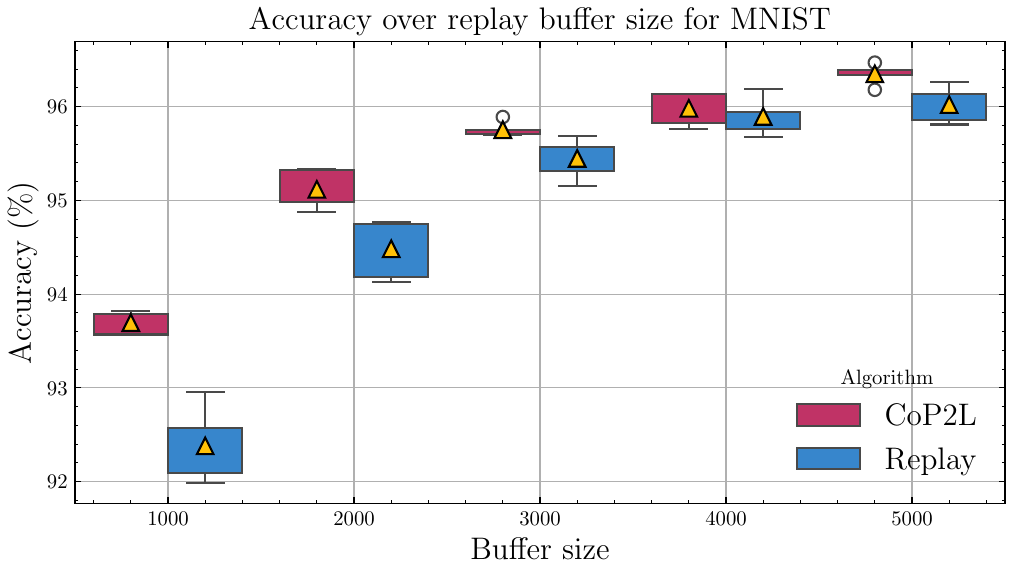}
\includegraphics[width=.49\textwidth]{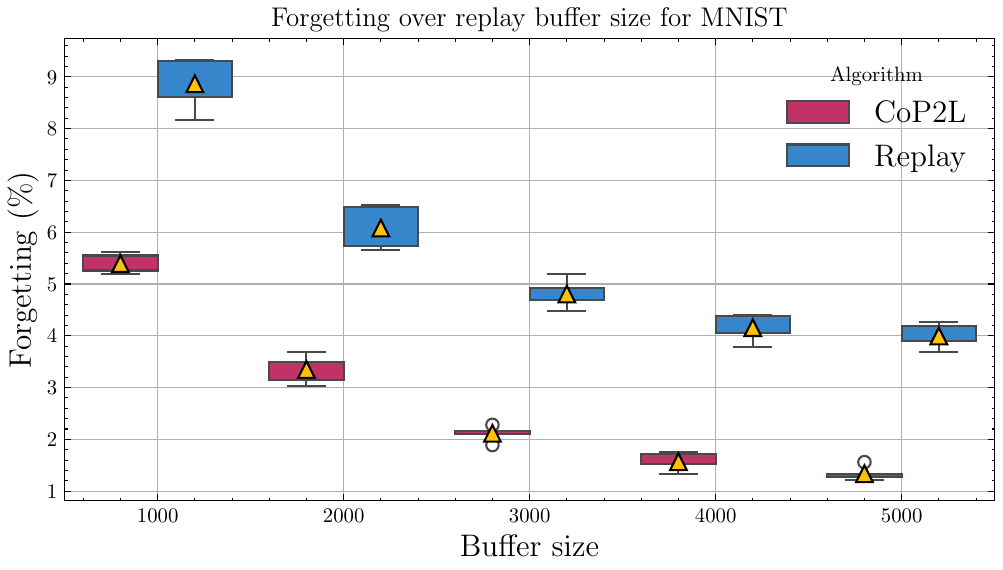}

\includegraphics[width=.49\textwidth]{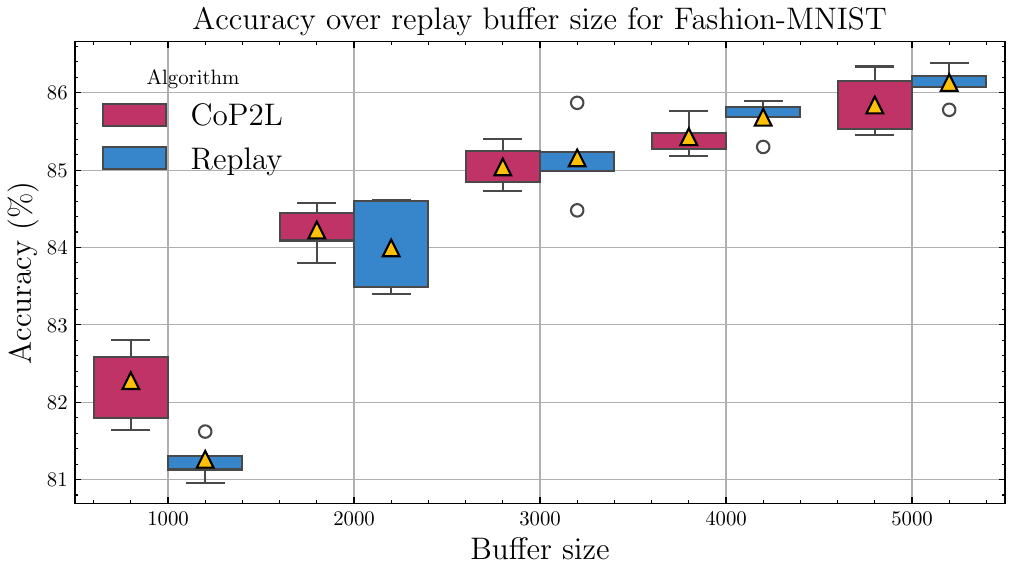}
\includegraphics[width=.49\textwidth]{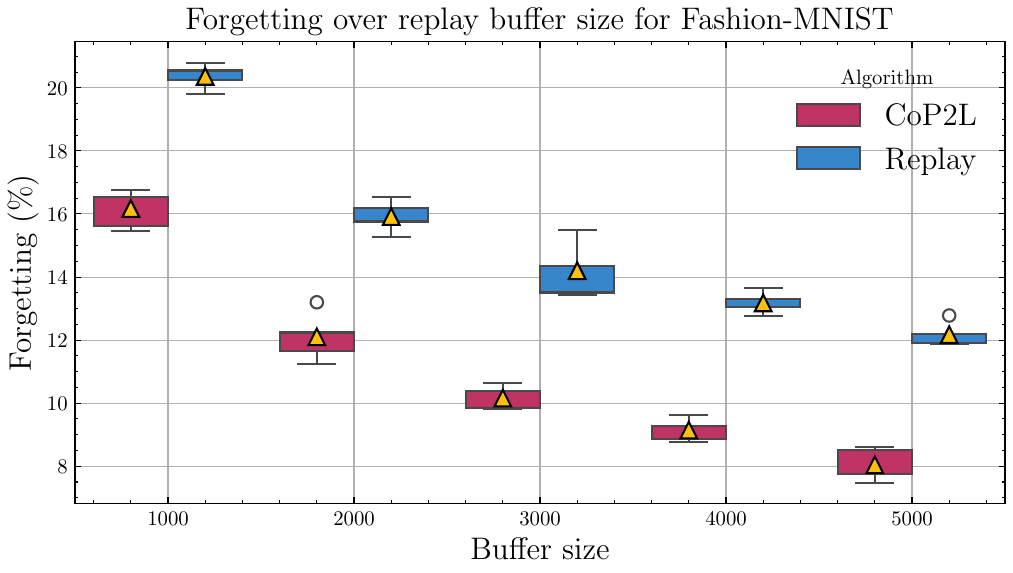}

\includegraphics[width=.49\textwidth]{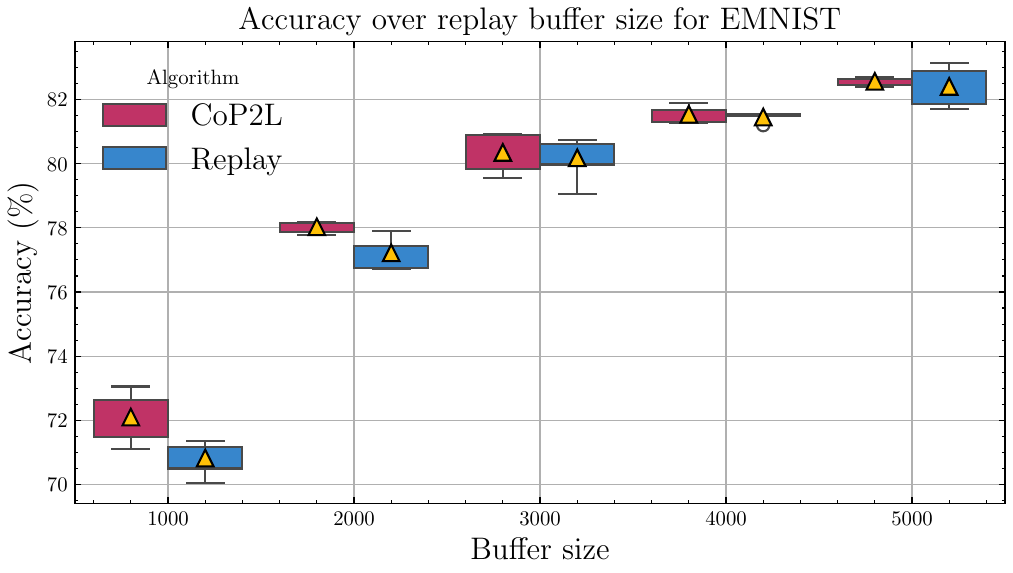}
\includegraphics[width=.49\textwidth]{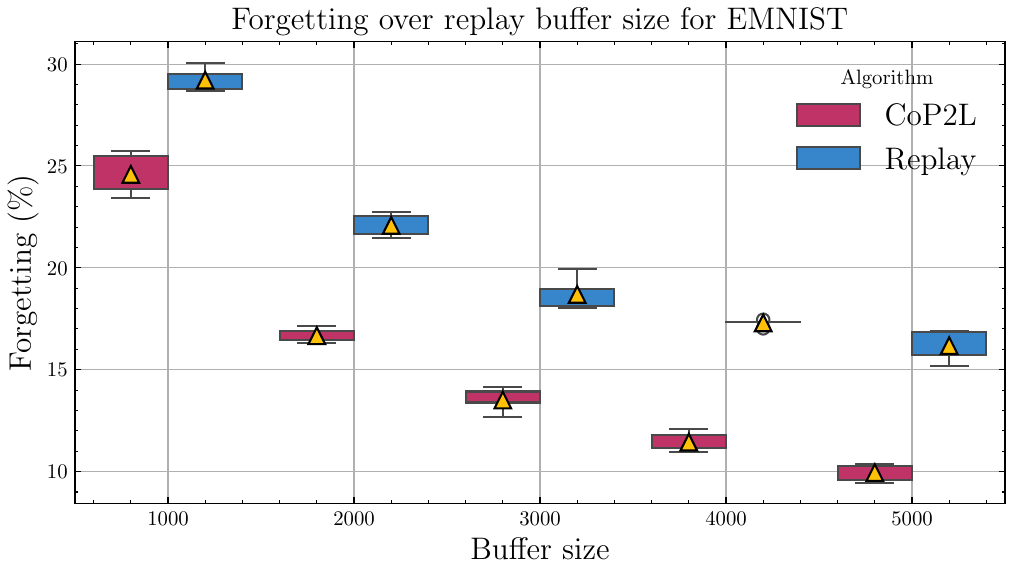}

\caption{Accuracy (left) and forgetting (right) vs.\ buffer size for MNIST (top), Fashion-MNIST (middle), and EMNIST (bottom). Orange triangles indicate mean performance.}
\label{fig:experiments}
\vspace{-4mm}
\end{figure}

\subsection{Analysis of memory cost}
\label{app:memory_comparison}

We have conducted a memory comparison on Split-CIFAR10 (5 tasks) in Class Incremental Learning (CIL) and Task Incremental Learning (TIL) settings, as shown in \cref{tab:memory_usage}.

We report Maximum allocated memory (actual usage), maximum reserved memory (allocated memory PyTorch’s caching allocator), and also the average cumulative memory over all tasks.

We would like to note that, in the current implementation, we execute the sample compression in one go, by processing all data items that correspond to the current task in parallel. It is totally possible to implement this with minibatching without loss of performance. The table reports, in parentheses, the approximate memory that would be used by CoP2L using the same minibatch size as the reported Replay.

\begin{table}[ht]
  \centering\small
  \caption{Memory usage comparison for CoP2L and Replay in CIL and TIL settings.}
  \label{tab:memory_usage}
  \begin{tabular}{@{}lccc@{}}
    \toprule
    \textbf{Method} & \textbf{Max Mem allocated (MB)} & \textbf{Max Mem reserved (MB)} & \textbf{Cumulate Mem usage (TB)} \\
    \midrule
    \multicolumn{4}{@{}l}{\textbf{CIL}} \\
    \addlinespace[4pt]
    \hline
    CoP2L  & 15535.57 $\pm$ 87.53 (342.85 $\pm$ 6.34) & 22310.0 $\pm$ 104.67 (492.36 $\pm$ 7.98) & 1261.38 $\pm$ 31.72 \\
    \hline
    Replay & 458.76 $\pm$ 9.84     & 648.0 $\pm$ 12.35    & 3383.43 $\pm$ 74.51 \\
    \midrule
    \multicolumn{4}{@{}l}{\textbf{TIL}} \\
    \addlinespace[4pt]
    \hline
    CoP2L  & 15535.66 $\pm$ 98.34 (342.86 $\pm$ 6.78) & 22310.0 $\pm$ 102.65  (492.36 $\pm$ 7.81) & 197.18 $\pm$ 6.42 \\
    \hline
    Replay & 458.85 $\pm$ 9.91     & 648.0 $\pm$ 12.41    & 1783.11 $\pm$ 52.63 \\
    \bottomrule
  \end{tabular}
\end{table}

\subsection{Ablation on the use of early stopping}

In this section, we study the use of early stopping in our CoP2L algorithm. We present the result in \cref{tab:ablation_ES}. The early stopping improves the accuracy, the forgetting and the bound on each task except for the bound on Task 1 with the ViT. This is to be expected as the early stopping helps to alleviate the overfitting of Pick-To-Learn by selecting the checkpoint that achieves the best bound once the algorithm has effectively converged. In contrast, without early stopping, training proceeds until the training error reaches zero, which can lead to overfitting, bigger compression set and consequently worse generalization metrics.This can also explain the tighter bounds, as it was demonstrated by \citet{bazinet2024} that, on binary subsets of the MNIST dataset, the sample-compression bound was actually minimized about halfway before Pick-To-Learn converges.
\begin{table}[h]
    \centering
    \caption{Illustration of the behavior of the bound on CIFAR10 after 5 tasks, with and without early stopping using the bound.}
    \begin{adjustbox}{width=\textwidth}
    \begin{tabular}{cccccccccc}
    \toprule
    &Backbone & Accuracy $(\uparrow)$ & Forgetting $(\downarrow)$ & Bound on Task 1 $(\downarrow)$ & Bound on Task 2 $(\downarrow)$& Bound on Task 3$(\downarrow)$ & Bound on Task 4$(\downarrow)$ & Bound on Task 5$(\downarrow)$ \\\midrule
    \multirow{3}{9em}{With early stopping} 
        & RN18 & 68.12 $\pm$ 0.42 & 17.62 $\pm$ 0.31 & 82.34 $\pm$ 0.55 & 94.62 $\pm$ 0.47 & 88.17 $\pm$ 0.36 & 66.15 $\pm$ 0.28 & 53.52 $\pm$ 0.44 \\
        & RN50 & 80.98 $\pm$ 0.25 & 5.84 $\pm$ 0.18 & 62.97 $\pm$ 0.63 & 83.46 $\pm$ 0.39 & 78.17 $\pm$ 0.41 & 58.03 $\pm$ 0.22 & 51.02 $\pm$ 0.35 \\
        & ViT  & 94.45 $\pm$ 0.12 & 2.10 $\pm$ 0.09 & 31.02 $\pm$ 0.48 & 48.17 $\pm$ 0.33 & 40.66 $\pm$ 0.27 & 25.45 $\pm$ 0.19 & 21.30 $\pm$ 0.26 \\ 
        \midrule
    \multirow{3}{9em}{Without early stopping} 
        & RN18 & 65.30 $\pm$ 0.51 & 22.71 $\pm$ 0.37 & 90.02 $\pm$ 0.46 & 96.67 $\pm$ 0.40 & 91.20 $\pm$ 0.35 & 71.06 $\pm$ 0.29 & 59.15 $\pm$ 0.43 \\
        & RN50 & 77.36 $\pm$ 0.34 & 12.11 $\pm$ 0.26 & 78.05 $\pm$ 0.52 & 91.65 $\pm$ 0.38 & 83.52 $\pm$ 0.31 & 60.04 $\pm$ 0.24 & 53.28 $\pm$ 0.36 \\
        & ViT  & 94.34 $\pm$ 0.11 & 2.63 $\pm$ 0.07 & 29.46 $\pm$ 0.45 & 49.46 $\pm$ 0.30 & 42.16 $\pm$ 0.28 & 25.99 $\pm$ 0.21 & 23.72 $\pm$ 0.25 \\
    \bottomrule
    \end{tabular}
    \end{adjustbox}
    \label{tab:ablation_ES}
\end{table}
\subsection{Class-wise accuracy}
\label{app:classwise}

We report the accuracies obtained over earlier tasks with \ouralgo{} and with replay. We observe that \ouralgo{} is able to obtain a more balanced performance over earlier tasks compared to standard replay, as can be seen from the last rows of Table \ref{tab:ablation_earlier_cop2l}, and Table \ref{tab:ablation_earlier_replay}.  

\begin{table}[!h]
    \centering
    \caption{{The accuracies obtained on earlier tasks with \ouralgo{} on CIFAR10.}}
    \begin{adjustbox}{width=\textwidth}
    \begin{tabular}{ccccccccccc}
    \toprule
    & Class 1 & Class 2 & Class 3 & Class 4 & Class 5 & Class 6 & Class 7 & Class 8 & Class 9 & Class 10 \\\midrule
Task 1 & 98.10 $\pm$ 0.22 & 97.30 $\pm$ 0.31 \\
Task 2 & 98.60 $\pm$ 0.18 & 96.80 $\pm$ 0.27 & 83.70 $\pm$ 0.45 & 85.50 $\pm$ 0.36 \\
Task 3 & 95.70 $\pm$ 0.29 & 96.70 $\pm$ 0.21 & 79.40 $\pm$ 0.52 & 77.40 $\pm$ 0.47 & 82.90 $\pm$ 0.33 & 71.10 $\pm$ 0.58\\
Task 4 & 93.50 $\pm$ 0.34 & 97.00 $\pm$ 0.19 & 72.60 $\pm$ 0.49 & 68.70 $\pm$ 0.55 & 84.10 $\pm$ 0.28 & 71.20 $\pm$ 0.46 & 84.40 $\pm$ 0.31 & 80.10 $\pm$ 0.37 \\
Task 5 & 85.30 $\pm$ 0.41 & 91.50 $\pm$ 0.26 & 71.40 $\pm$ 0.53 & 66.50 $\pm$ 0.60 & 80.20 $\pm$ 0.35 & 69.70 $\pm$ 0.48 & 90.30 $\pm$ 0.24 & 83.60 $\pm$ 0.39 & 89.50 $\pm$ 0.30  & 83.10 $\pm$ 0.42 \\\bottomrule
    \end{tabular}
    \label{tab:ablation_earlier_cop2l}
\end{adjustbox}
\end{table}

\begin{table}[!h]
    \centering
    \caption{The accuracies obtained on earlier tasks with standard replay algorithm on CIFAR10.}
    \begin{adjustbox}{width=\textwidth}
    \begin{tabular}{ccccccccccc}
    \toprule
    & Class 1 & Class 2 & Class 3 & Class 4 & Class 5 & Class 6 & Class 7 & Class 8 & Class 9 & Class 10 \\\midrule
Task 1 & 98.10 $\pm$ 0.24 & 98.40 $\pm$ 0.19 \\
Task 2 & 93.10 $\pm$ 0.38 & 97.20 $\pm$ 0.27 & 94.70 $\pm$ 0.33 & 92.40 $\pm$ 0.41 \\
Task 3 & 93.80 $\pm$ 0.36 & 96.80 $\pm$ 0.22& 74.90 $\pm$ 0.55 & 53.80 $\pm$ 0.61 & 94.70 $\pm$ 0.29 & 90.50 $\pm$ 0.34 \\
Task 4 & 91.70 $\pm$ 0.40 & 96.50 $\pm$ 0.25 & 70.90 $\pm$ 0.57 & 57.20 $\pm$ 0.63 & 69.80 $\pm$ 0.48 & 76.20 $\pm$ 0.45 & 96.30 $\pm$ 0.21 & 95.60 $\pm$ 0.26 \\
Task 5 & 77.30 $\pm$ 0.52 & 77.90 $\pm$ 0.49 & 73.00 $\pm$ 0.54 & 59.90 $\pm$ 0.60 & 75.10 $\pm$ 0.47 & 80.90 $\pm$ 0.39 & 92.00 $\pm$ 0.28 & 87.70 $\pm$ 0.36 & 98.20 $\pm$ 0.18 & 96.10 $\pm$ 0.23 \\
\bottomrule
    \end{tabular}
    \label{tab:ablation_earlier_replay}
\end{adjustbox}
\end{table}

\subsection{Plasticity forgetting tradeoff}
\label{app:plasticity}
Denoting $A(t, \theta_t)$ the accuracy obtained on a task $t$ of a model~$f_{\theta_t}$ after finishing training on task $t$, the plasticity is given by
$$\overline{P}(\theta_T) \,{=}\, \frac{1}{T} \sum_{t=1}^T A(t, \theta_t)\,,$$
We observe that \ouralgo{} is able to retain a low level of forgetting while having a relatively good level of plasticity. As shown in Table \ref{tab:accuracyvsplasticity}, \ouralgo{} consistently attains substantially lower forgetting than Finetuning, Replay, and Dark Experience Replay, and also outperforms GDumb on this metric. Although iCaRL achieves slightly lower forgetting, this comes at the cost of reduced plasticity. Overall, this trade-off highlights the favorable balance achieved by \ouralgo{}. We observe the same behavior across both RN50 and ViT backbones.

\begin{table}[!h]
    \centering \small
    \caption{Plasticity vs Forgetting Tradeoff on CIFAR10 dataset}
    \begin{tabular}{cccc}
    \toprule
         Backbone & Method & Plasticity ($\uparrow$) & Forgetting ($\downarrow$) \\\midrule
RN50 & CoP2L & 85.57 $\pm$ 0.34 & \phantom05.55 $\pm$ 0.28 \\ 
RN50 & Finetuning & 97.21 $\pm$ 0.22 & 97.22 $\pm$ 0.41 \\
RN50 & Replay & 95.50 $\pm$ 0.37 & 17.09 $\pm$ 0.33 \\
RN50 & iCaRL & 75.79 $\pm$ 0.45 & \phantom01.54 $\pm$ 0.19 \\
RN50 & Gdumb & 87.86 $\pm$ 0.31 & \phantom08.95 $\pm$ 0.36 \\
RN50 & DER & 94.78 $\pm$ 0.27 & 15.60 $\pm$ 0.42 \\\midrule
ViT & CoP2L & 96.33 $\pm$ 0.18 & \phantom03.36 $\pm$ 0.21 \\
ViT & Finetuning & 99.36 $\pm$ 0.09 & 89.74 $\pm$ 0.38 \\
ViT & Replay & 98.82 $\pm$ 0.14 & \phantom06.01 $\pm$ 0.26 \\
ViT & iCaRL & 94.48 $\pm$ 0.23 & \phantom01.55 $\pm$ 0.17 \\
ViT & Gdumb & 96.86 $\pm$ 0.20 & \phantom03.39 $\pm$ 0.24 \\
ViT & DER & 98.09 $\pm$ 0.16 & \phantom03.86 $\pm$ 0.29 \\ 
\bottomrule
    \end{tabular}
    \label{tab:accuracyvsplasticity}
\end{table}

\end{document}